\documentclass[conference]{IEEEtran}
\IEEEoverridecommandlockouts
\usepackage{cite}
\usepackage{amsmath,amssymb,amsfonts}
\usepackage{algorithmic}
\usepackage{graphicx}
\usepackage{textcomp}
\usepackage{xcolor}
\usepackage{comment}

\DeclareMathOperator*{\argmin}{\arg\!\min}

\usepackage[ruled,vlined,shortend,linesnumbered]{algorithm2e} 

\SetCommentSty{mycommfont}

\usepackage{amsthm}
\theoremstyle{definition}
\newtheorem{definition}{Definition}
\newtheorem{theorem}{Theorem}

\newtheorem{lemma}{Lemma}

\usepackage{xspace}
\usepackage{booktabs} 
\usepackage{anyfontsize}
\usepackage{mathtools}
\usepackage{amsmath}
\usepackage[hidelinks]{hyperref}
\usepackage{pifont}
\usepackage{stmaryrd}
\usepackage{bm}
\usepackage{multirow}
\usepackage{subfigure}
\usepackage{kotex}
\usepackage{makecell}
\usepackage{enumitem}
\usepackage{nccmath}
\usepackage{hyperref}
\usepackage{microtype}

\newcommand{\cp}{CANDECOMP/PARAFAC\xspace}
\newcommand{\method}{\textsc{SOFIA}\xspace}
\newcommand{\methodALS}{\textsc{SOFIA\textsubscript{ALS}}\xspace}

\newcommand{\cpwopt}{\textsc{CP-WOPT}\xspace}

\newcommand{\smf}{\textsc{SMF}\xspace}
\newcommand{\cphw}{\textsc{CPHW}\xspace}
\newcommand{\brst}{\textsc{BRST}\xspace}
\newcommand{\olstec}{\textsc{OLSTEC}\xspace}
\newcommand{\onlineSGD}{OnlineSGD\xspace}

\newcommand{\ormstc}{\textsc{OR-MSTC}\xspace}
\newcommand{\mast}{\textsc{MAST}\xspace}
\newcommand{\cmark}{\ding{51}}

\newcommand{\change}{\textcolor{black}}

\newcommand{\red}{\textcolor{red}}

\newcommand{\TODO}{\red{[TODO]}\xspace}

\newcommand{\address}{\url{https://github.com/wooner49/sofia}\xspace}

\DeclarePairedDelimiter{\norm}{\lVert}{\rVert}
\DeclareMathAlphabet{\tensor}{OMS}{cmsy}{b}{n}
\newcommand{\TY}{\tensor{Y}}
\newcommand{\TYt}{\tensor{Y}_{t}}
\newcommand{\TYtstar}{\TYt^{*}}
\newcommand{\TYthat}{\hat{\TY}_{t|t-1}}
\newcommand{\TX}{\tensor{X}}
\newcommand{\MX}{\mathbf{X}}
\newcommand{\TXt}{\tensor{X}_{t}}
\newcommand{\TXthat}{\hat{\tensor{X}}_{t}}
\newcommand{\TXhat}{\tensor{\hat{X}}}
\newcommand{\TO}{\tensor{O}}
\newcommand{\TOt}{\tensor{O}_{t}}

\newcommand{\TRt}{\tensor{R}_{t}}
\newcommand{\TOmega}{\boldsymbol{\Omega}}
\newcommand{\TOmegat}{\boldsymbol{\Omega}_{t}}

\newcommand{\ft}{f_{t}}

\newcommand{\FM}{\mathbf{U}}
\newcommand{\fv}{\mathbf{u}}

\newcommand{\FMn}{\FM^{(n)}}

\newcommand{\allFM}{\{\FMn\}_{n=1}^{N}}

\newcommand{\IonetoINminus}{I_{1}\times\dots\times I_{N-1}}
\newcommand{\IonetoIN}{I_{1}\times\dots\times I_{N}}

\newcommand{\hadamard}{\circledast}
\newcommand{\bigHadamard}{\mathop{\scalebox{1.4}{\raisebox{-0.2ex}{$\circledast$}}}}
\newcommand{\khatrirao}{\odot}
\newcommand{\bigKhatrirao}{\mathop{\scalebox{1.4}{\raisebox{-0.2ex}{$\odot$}}}}

\newcommand{\uvec}{\mathbf{u}}
\newcommand{\wvec}{\mathbf{w}}
\newcommand{\Umat}{\mathbf{U}}
\newcommand{\Wmat}{\mathbf{W}}
\newcommand{\eyeR}{\mathbf{I}_{R}}

\newcommand{\Uone}{\Umat^{(1)}}
\newcommand{\Un}{\Umat^{(n)}}

\newcommand{\UN}{\Umat^{(N)}}
\newcommand{\UonetoUN}{\Uone,\dots,\UN}
\newcommand{\allU}{\{\Un\}_{n=1}^{N}}
\newcommand{\ntU}{\{\Un\}_{n=1}^{N-1}}
\newcommand{\ntUnt}{\{\Un_{t}\}_{n=1}^{N-1}}
\newcommand{\ntUnti}{\{\Un_{\ti}\}_{n=1}^{N-1}}
\newcommand{\ntUntminus}{\{\Un_{t-1}\}_{n=1}^{N-1}}

\newcommand{\tildew}{\tilde{\wvec}}
\newcommand{\tildeu}{\tilde{\uvec}}
\newcommand{\tildeur}{\tilde{\uvec}_{r}}
\newcommand{\tildeurn}{\tilde{\uvec}_{r}^{(n)}}
\newcommand{\tildeurN}{\tilde{\uvec}_{r}^{(N)}}
\newcommand{\tildeuN}{\tilde{\uvec}^{(N)}}

\newcommand{\ktensor}{\llbracket\Uone,\dots,\UN\rrbracket}
\newcommand{\ktensoruN}{\llbracket\ntU,\uN\rrbracket}
\newcommand{\ktensorTau}{\llbracket\ntU,\uvec_{\tau}^{(N)}\rrbracket}
\newcommand{\ktensort}{\llbracket\ntUnt,\uN_{t}\rrbracket}

\newcommand{\uinn}{\uvec_{i_{n}}^{(n)}}
\newcommand{\uiNN}{\uvec_{i_{N}}^{(N)}}
\newcommand{\Binn}{\mathbf{B}^{(n)}_{i_{n}}}
\newcommand{\BiNN}{\mathbf{B}^{(N)}_{i_{N}}}
\newcommand{\cinn}{\mathbf{c}^{(n)}_{i_{n}}}
\newcommand{\ciNN}{\mathbf{c}^{(N)}_{i_{N}}}

\newcommand{\uN}{\uvec^{(N)}}
\newcommand{\uNt}{\uvec^{(N)}_{t}}
\newcommand{\uNhat}{\hat{\uvec}^{(N)}}

\newcommand{\ti}{t_{i}}
\newcommand{\Omegainn}{\Omega^{(n)}_{i_{n}}}
\newcommand{\OmegaiNN}{\Omega^{(N)}_{i_{N}}}
\newcommand{\inOmega}{(i_{1},\dots,i_{N})\in\Omegainn}
\newcommand{\inOmegasplit}{\substack{(i_{1},\dots,i_{N})\\\in\Omegainn}}
\newcommand{\inOmegaN}{(i_{1},\dots,i_{N})\in\OmegaiNN}

\newcommand{\bigHadamardBinn}{\bigHadamard_{l\neq n} \uvec^{(l)}_{i_{l}}}

\newcommand{\uNiNj}{u^{(N)}_{i_{N}j}}
\newcommand{\uNiNjplus}{u^{(N)}_{(i_{N}+1)j}}
\newcommand{\uNiNjminus}{u^{(N)}_{(i_{N}-1)j}}
\newcommand{\uNiNjplusm}{u^{(N)}_{(i_{N}+m)j}}
\newcommand{\uNiNjminusm}{u^{(N)}_{(i_{N}-m)j}}

\newcommand{\ystar}{y^{*}_{i_{1},\dots,i_{N}}}
\newcommand{\level}{\textbf{\textit{l}}}
\newcommand{\trend}{\textbf{\textit{b}}}
\newcommand{\seasonal}{\textbf{\textit{s}}}
\newcommand{\sAlpha}{\boldsymbol{\alpha}}
\newcommand{\sBeta}{\boldsymbol{\beta}}
\newcommand{\sGamma}{\boldsymbol{\gamma}}

\newcommand{\ML}{\mathbf{L}}
\newcommand{\Lone}{\ML_{1}}
\newcommand{\Lm}{\ML_{m}}

\newcommand{\pvec}{\mathbf{p}}
\newcommand{\qvec}{\mathbf{q}}
\newcommand{\ptau}{\pvec_{\tau}}
\newcommand{\qtau}{\qvec_{\tau}}

\newcommand{\hsigma}{\hat{\sigma}}
\newcommand{\bSigma}{\hat{{\bm\Sigma}}}

\def\BibTeX{{\rm B\kern-.05em{\sc i\kern-.025em b}\kern-.08em
    T\kern-.1667em\lower.7ex\hbox{E}\kern-.125emX}}
\begin{document}

\title{Robust Factorization of Real-world Tensor Streams with Patterns, Missing Values, and Outliers}

\author{
	\IEEEauthorblockN{Dongjin Lee}
	\IEEEauthorblockA{School of Electrical Engineering, KAIST\\dongjin.lee@kaist.ac.kr}
	\and
	\IEEEauthorblockN{Kijung Shin}
	\IEEEauthorblockA{Graduate School of AI and School of Electrical Engineering, KAIST\\kijungs@kaist.ac.kr}
}

\maketitle

\begin{figure}
	\vspace{-8mm}
\end{figure}

\begin{abstract}
	Consider multiple seasonal time series being collected in real-time, in the form of a tensor stream. 
Real-world tensor streams often include missing entries (e.g., due to network disconnection) and at the same time unexpected outliers (e.g., due to system errors).
Given such a real-world tensor stream, how can we estimate missing entries and predict future evolution accurately in real-time?

In this work, we answer this question by introducing \method, a robust factorization method for real-world tensor streams.
In a nutshell, \method smoothly and tightly integrates tensor factorization, outlier removal, and temporal-pattern detection, which naturally reinforce each other.
Moreover, \method integrates them in linear time, in an online manner, despite the presence of missing entries.
We experimentally show that \method is (a) \textit{robust and accurate}: yielding up to $76\%$ lower imputation error and $71\%$ lower forecasting error; (b) \textit{fast}: up to $935\times$ faster than the second-most accurate competitor; and (c) \textit{scalable}: scaling linearly with the number of new entries per time step.

\end{abstract}

\begin{IEEEkeywords}
	Tensor Factorization; Streaming Algorithm; Outlier Robustness; Holt-Winters Forecasting.
\end{IEEEkeywords}

\section{Introduction}
\label{sec:intro}
Tensors are high-dimensional arrays that are used to represent multi-way data.
Data modeled as tensors are collected and utilized in various fields, including machine learning~\cite{sidiropoulos2017tensor}, urban computing~\cite{tan2013tensor,gandy2011tensor}, chemometrics~\cite{tomasi2005parafac}, image processing~\cite{li2010optimum,liu2012tensor}, and recommender systems~\cite{karatzoglou2010multiverse,shin2016fully}.

As the dimension and size of data increase, extensive research has been conducted on tensor factorization, using which tensors can be analyzed efficiently and effectively.
Given a tensor, \textit{tensor factorization} extracts its underlying latent structure, which is meaningful and at the same time useful for various purposes.
Among many tensor factorization methods, \cp (CP) factorization~\cite{tomasi2005parafac,shin2016fully,acar2011scalable} is most widely used due to its simplicity and effectivenss.

Real-world tensors are often incomplete due to unintended problems such as network disconnection and system errors.
The problem of imputing the missing entries based on observed data is called \textit{tensor completion}.
It is one of the most important and actively studied problems in tensor related research, and numerous solutions based on tensor factorization has been developed  \cite{acar2011scalable,liu2012tensor,song2019tensor}.

\setlength{\tabcolsep}{5pt}
\begin{table}[!t]
	\vspace{-5mm}
	\centering
	\caption{
		\label{tab:comparison}
		Comparison of tensor factorization and completion algorithms. Notice that only our proposed algorithm \method satisfies all the criteria.
	}
	\scalebox{0.9}{
		\begin{tabular}{c|r|ccccccccc|c}
			\toprule
			\multicolumn{2}{c|}{}                                                       & \rotatebox{90}{\cpwopt \cite{acar2011scalable}} & \rotatebox{90}{\onlineSGD \cite{mardani2015subspace}} & \rotatebox{90}{\olstec \cite{kasai2016online}} & \rotatebox{90}{\mast \cite{song2017multi}} & \rotatebox{90}{\brst \cite{zhang2018variational}} & \rotatebox{90}{\ormstc \cite{najafi2019outlier}} & \rotatebox{90}{\smf \cite{hooi2019smf}} & \rotatebox{90}{\cphw \cite{dunlavy2011temporal}} &
			\rotatebox{90}{Others \cite{takahashi2017autocyclone,araujo2019tensorcast}} &
			\rotatebox{90}{{\bf \method} (Proposed)}                                                                                                                                                                                                                                                                                                                                                                                                                                                                           \\
			\midrule
			\multirow{2}{*}{Functions}                                                  & Imputation                                      & \cmark                                                & \cmark                                         & \cmark                                     & \cmark                                            & \cmark                                           & \cmark                                  &                                                  &        &        & \cmark \\
			                                                                            & Forecasting                                     &                                                       &                                                &                                            &                                                   &                                                  &                                         & \cmark                                           & \cmark & \cmark & \cmark \\
			\midrule
			\multirow{5}{*}{Properties}                                                 & Robust to missings                              & \cmark                                                & \cmark                                         & \cmark                                     & \cmark                                            & \cmark                                           & \cmark                                  &                                                  &        &        & \cmark \\
			                                                                            & Robust to outliers                              &                                                       &                                                &                                            &                                                   & \cmark                                           & \cmark                                  &                                                  &        &        & \cmark \\
			                                                                            & Online algorithm                                &                                                       & \cmark                                         & \cmark                                     & \cmark                                            & \cmark                                           & \cmark                                  & \cmark                                           &        &        & \cmark \\
			                                                                            & Seasonality-aware                               &                                                       &                                                &                                            &                                                   &                                                  &                                         & \cmark                                           & \cmark & \cmark & \cmark \\
			                                                                            & Trend-aware                                     &                                                       &                                                &                                            &                                                   &                                                  &                                         & \cmark                                           & \cmark & \cmark & \cmark \\
			\bottomrule
		\end{tabular}
	}
\end{table}

At the same time, real-world tensors are easily corrupted by outliers due to unpredicted events during data collection, such as sensor malfunctions and malicious tampering.
Recovering incomplete and at the same time contaminated tensors is a challenging and unwieldy task since tensor factorization, which most tensor completion techniques are based on, is vulnerable to outliers.
To find latent structure behind such a noisy tensor accurately, many efforts have been made to design a `outlier-robust' tensor factorization algorithm~\cite{goldfarb2014robust,li2010optimum,zhao2015bayesian}.

\begin{figure*}[!t]
	\vspace{-4mm}
	\centering
	\subfigure[Outlier-robust Imputation]{\label{fig:crown:imputation}
		\includegraphics[width=0.25\linewidth]{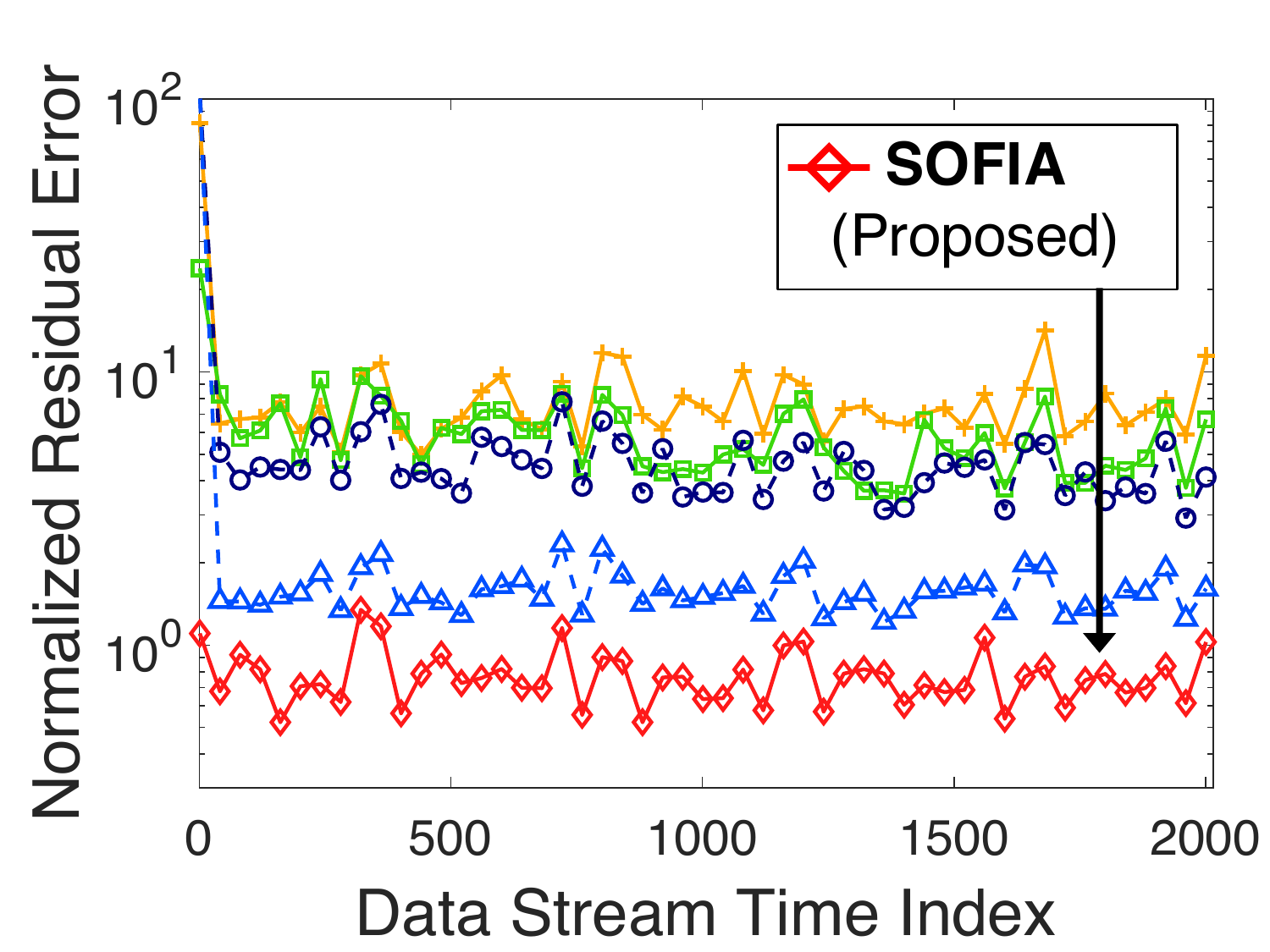}
	}
	\subfigure[Fast and Accurate Imputation]{\label{fig:crown:accuracy_speed}
		\includegraphics[width=0.20\linewidth]{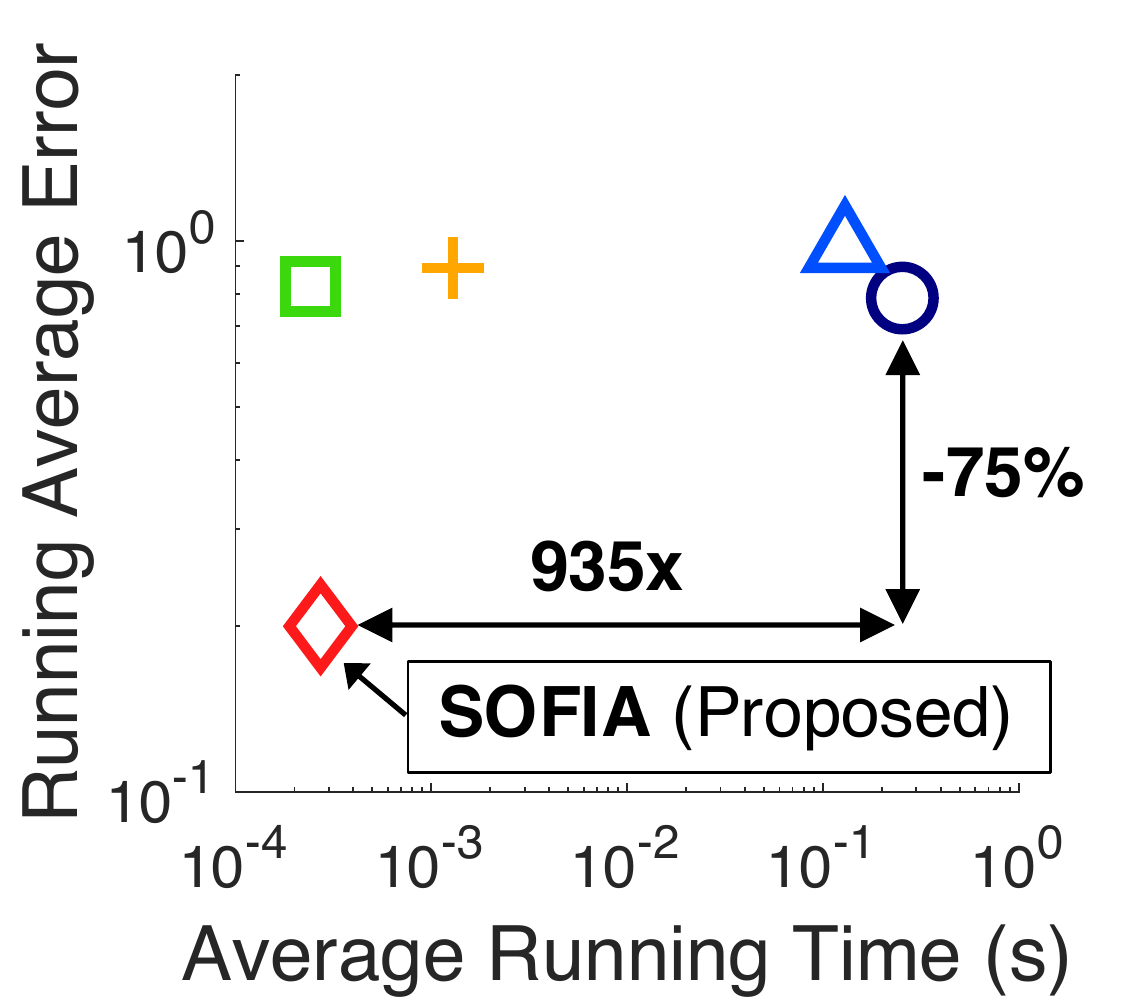}
	}
	\subfigure[Accurate Forecasting]{\label{fig:crown:forecasting}
		\includegraphics[width=0.20\linewidth]{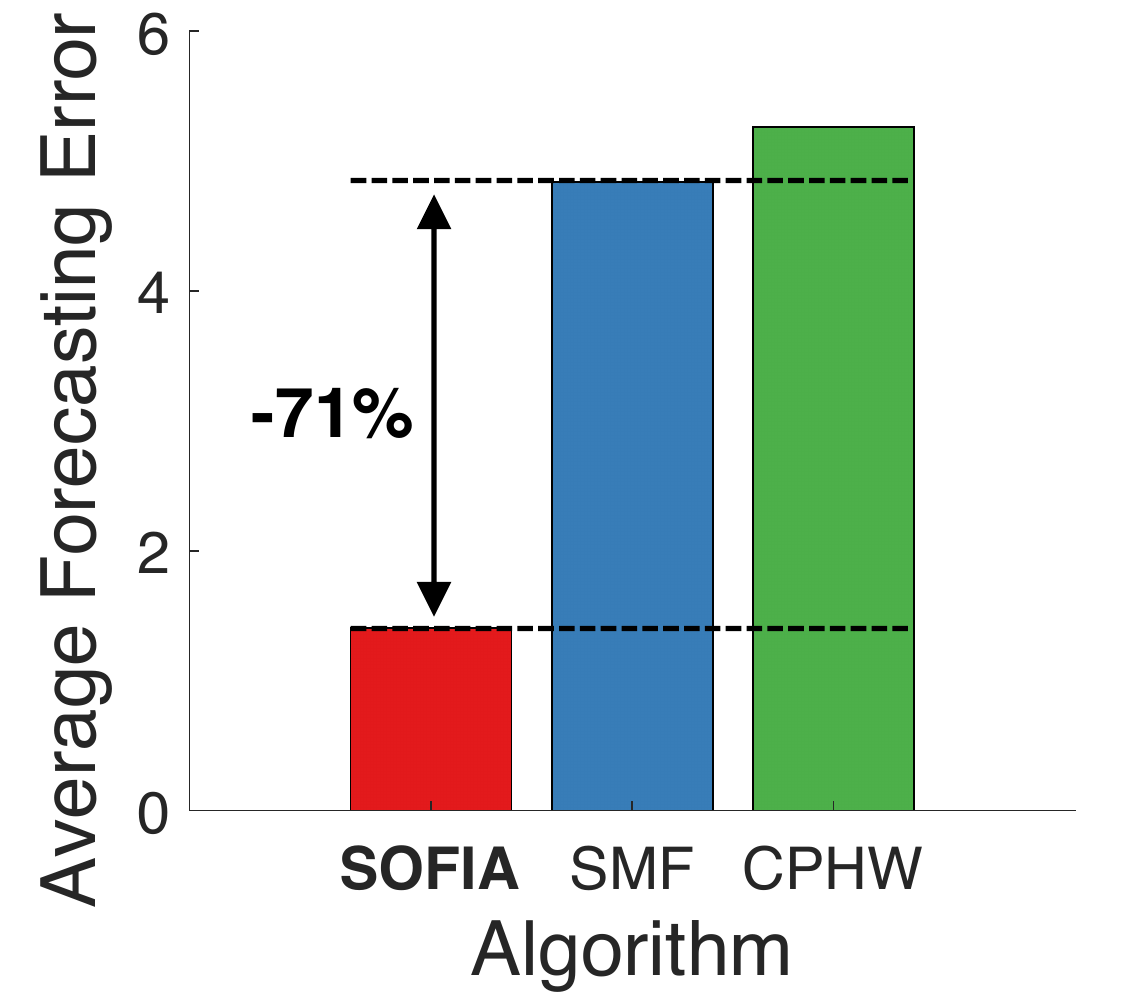}
	}
	\subfigure[Linear Scalability]{\label{fig:crown:scalability}
		\includegraphics[width=0.20\linewidth]{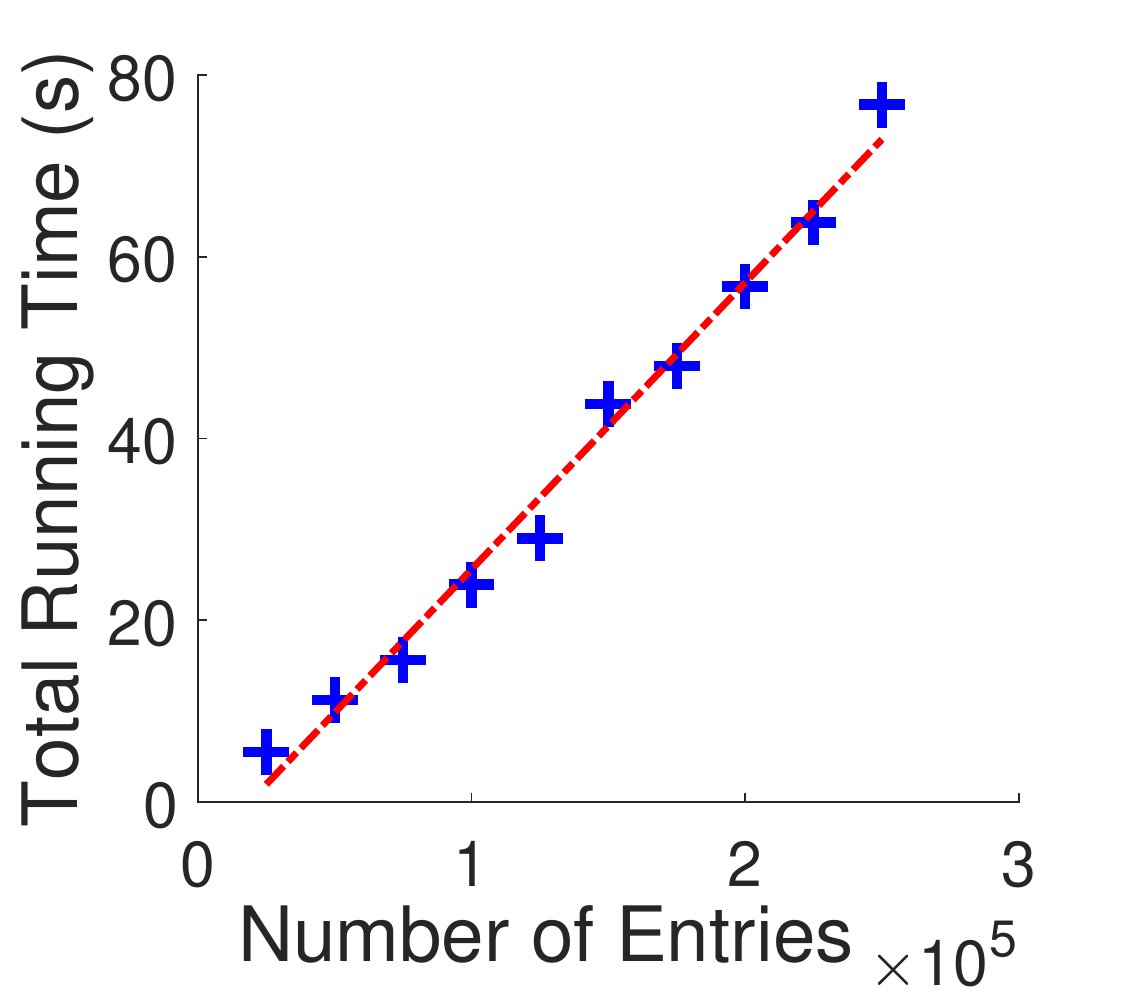}
	}
	\caption[]{\label{fig:crown}
		\textbf{\method is accurate, fast, and scale linearly.}
		(a) \method recovers the missing values more accurately compared to state-of-the-art streaming tensor factorization and completion algorithms over the entire stream.
		(b) \method is $75\%$ more accurate and $935\times$ faster than the second-most accurate competitor.
		(c) \method gives up to $71\%$ lower forecasting error than its best competitor.
		(d) \method scales linearly with the number of entries per time step.
	}
\end{figure*}

Moreover, it is common to incorporate temporal information into tensors by adding temporal modes that evolve over time.
Tensor factorization also has been applied to such time-evolving tensors since latent temporal components provide considerable insight into temporal dynamics. Successful applications include anomaly detection~\cite{fanaee2016tensor}, discussion tracking~\cite{bader2008discussion}, recommender system \cite{xiong2010temporal} and link prediction~\cite{dunlavy2011temporal}.

In many applications, tensor data are usually collected incrementally over time in the form of a tensor stream.
Since tensor streams can potentially be of infinite length, it is infeasible, in terms of computational and storage costs, to use batch-based tensor factorization approaches to process them.
Instead, tensor factorization should be performed in an online manner.
That is, new tensor entries should be processed incrementally, as they arrive, without requiring too much space.


Given that real-world tensors are evolving constantly over time with missing entries and unexpected outliers, how can we estimate the missing entries?
Can we also predict future entries?
Can both imputation and prediction be performed accurately in an online manner?

In this work, we propose \method, a \textbf{S}easonality-aware \textbf{O}utlier-robust \textbf{F}actorization of \textbf{I}ncomplete stre\textbf{A}ming tensors.
\method extends CP factorization to leverage two temporal characteristics inherent in real-world time series: graduality and seasonality.\footnote{That is, real-world time series tend to change gradually with patterns recurring at regular intervals.}
Note that even identifying these patterns itself is challenging in real-world tensors due to the presence of missing entries and unexpected outliers.
\method addresses this challenge by smoothly and tightly integrating our extended CP factorization, outlier removal, and temporal-pattern detection so that they enhance each other.
Notably, \method conducts this integration in an online manner, scaling linearly with the number of new entries per time step.
After that, using discovered latent structures and temporal patterns, \method accurately imputes missing entries and also predicts future entries.
To this end, \method employs the Holt-Winters algorithm, a well-known time-series forecasting method.

Through extensive experiments on four real-world datasets, we demonstrate the advantages of \method over six state-of-the-art competitors, which only partially address the above questions (see Table~\ref{tab:comparison}).
Specifically, we show that \method has the following desirable properties:
\begin{itemize}[leftmargin=3mm]
	\item \textbf{Robust and accurate}: \method performs imputation and forecasting with up to $76\%$ and $71\%$ lower estimation error than its best competitors (Figures~\ref{fig:crown:imputation}-\ref{fig:crown:forecasting}).
	\item \textbf{Fast}: \method is  up to $935\times$ faster than the second-most accurate imputation method (Figure~\ref{fig:crown:accuracy_speed}).
	\item \textbf{Scalable}: \method is an online algorithm, and its time complexity is linear in the number of new tensor entries per time step (Figure~\ref{fig:crown:scalability} and Lemma~\ref{lemma:complexity:dynamic}).
\end{itemize}

\noindent{\textbf{Reproducibility}}:
The code and datasets used in the paper are available at \address.

The rest of this paper is organized as follows.
In Section~\ref{sec:related}, we review related work.
In Section~\ref{sec:prelim}, we introduce notations and give preliminaries.
In Section~\ref{sec:model} and~\ref{sec:method}, we propose our factorization model and optimization algorithm, respectively.
After presenting experimental results in Section~\ref{sec:exp}, we conclude in Section~\ref{sec:conclusion}.

\section{Related Work}
\label{sec:related}
\begin{figure*}[!t]
    \vspace{-4mm}
\end{figure*}

In this section, we review previous studies that are closely related to our work.

\textbf{Streaming Tensor Factorization}:
Extensive research on a streaming tensor factorization and completion based on CP factorization have been widely studied in the past few years. 
Several online algorithms \cite{kwon2021slice,smith2018streaming, zhou2016accelerating} have been proposed for tracking the CP decomposition of an incremental tensor.
However, since their approaches assume all tensor entries are observable, they are not directly applicable to tensor completion tasks.
To recover missing entries in tensor streams, Mardani et al.~\cite{mardani2015subspace} proposed \onlineSGD, an online tensor factorization and completion algorithm under the presence of missing values.
They used a stochastic gradient descent (SGD) to optimize their factorization model.
Kasai~\cite{kasai2016online} solved the same problem with recursive least squares (RLS). It gives smaller imputation errors than \onlineSGD when subspaces change dramatically.
Considering a more general situation where tensors increase in multiple modes, Song et al.~\cite{song2017multi} proposed \mast, which handles such dynamics while imputing the missing entries.
Robust factorization approaches for incomplete and contaminated tensor streams have also been studied.
Zhang et al.~\cite{zhang2018variational} proposed \brst, which aims to distinguish sparsely corrupted outliers from low-rank tensor streams via variational Bayesian inference.
Najafi et al.~\cite{najafi2019outlier} proposed \ormstc, an outlier-robust completion algorithm for multi-aspect incomplete tensor streams.
Although many studies have dealt with the problem of streaming tensor factorization for recovering missing entries, no studies have been conducted on utilizing temporal patterns inherent in tensor streams for robustness.

\textbf{Temporal Patterns in Tensor Factorization}:
Temporal information obtained by tensor factorization has proved effective in various studies.
Dunlavy et al.~\cite{dunlavy2011temporal} developed a temporal link prediction algorithm that applies $3$-way CP decomposition and extends the temporal factor matrix using the Holt-Winters method.
Araujo et al.~\cite{araujo2019tensorcast} proposed a coupled tensor factorization approach to forecast future values via the Holt-Winters method.
Takahashi et al.~\cite{takahashi2017autocyclone} proposed a factorization model to separate cyclic patterns from outliers in a static tensor.
Esp\'\i n-Noboa et al.~\cite{espin2016discovering} used non-negative tensor factorization to analyze temporal dynamics on a Taxi dataset.
For a streaming algorithm, Hooi et al.~\cite{hooi2019smf} proposed \smf, which utilizes seasonal patterns for updating its factorization model and forecasting future values.
While the above-mentioned approaches utilize temporal patterns in static or streaming tensors, they are vulnerable to unexpected outliers and not applicable when tensor entries are missing.

\textbf{Traditional Time-series Models}:
Traditional time-series forecasting or statistical regression models, such as autoregressive based models~\cite{akaike1969fitting}, Kalman filters~\cite{welch1995introduction}, and exponential smoothing~\cite{gardner1985exponential}, have been successful as simple and effective approaches for predicting future data.
Especially, seasonal ARIMA~\cite{hyndman2018forecasting} and triple exponential smoothing (also known as the Holt-Winters model~\cite{holt2004forecasting,winters1960forecasting}) have proved effective for predicting seasonal time series, whose values are periodic, by capturing repeated patterns in data.
However, these traditional approaches cannot be used if time series have missing values, and thus they are not suitable for our problems (i.e., imputing missing values and forecasting future data).

\textbf{Contrast with Competitors}:
Our proposed algorithm, \method, takes into account all the functionalities and properties that are partially handled by existing approaches (see Table~\ref{tab:comparison}) for accuracy and efficiency.
Specifically, \method captures temporal patterns (i.e., seasonality and trends) in incomplete and corrupted tensor streams accurately, and then it utilizes them to recover missing entries and forecast future entries in an online manner.

\section{Preliminaries and Notations}
\label{sec:prelim}

\begin{figure*}[!t]
  \vspace{-4mm}
\end{figure*}

\begin{table}[!t]
  \centering
  \caption{Table of symbols.}
  \label{tab:symbols}
  \scalebox{1}{
    \begin{tabular}{l|l}
      \toprule
      \textbf{Symbol}                                 & \textbf{Definition}                                                                    \\
      \midrule
      \multicolumn{2}{l}{Notations for Tensors, Matrices, and Vectors (Sections~\ref{sec:prelim:notation} and~\ref{sec:prelim:factorization})} \\
      \midrule
      $u$, $\uvec$, $\Umat$, $\TX$                    & scalar, vector, matrix, tensor                                                         \\
      $u_{i_{1}}$, $u_{i_{1}i_{2}}$                   & $i_{1}$-th entry of $\uvec$ and $(i_{1},i_{2})$-th enrty of $\Umat$                    \\
      $x_{i_{1}\dots i_{N}}$                          & $(i_{1}\dots i_{N})$-th entry of $\TX$                                                 \\
      $\uvec_{i}$, $\tildeu_{i}$                      & $i$-th row and column vectors of $\Umat$                                               \\
      $\Umat^{\top}$, $\Umat^{-1}$, $\Umat^{\dagger}$ & transpose, inverse, and pseudoinverse of $\Umat$                                       \\
      $\norm{\Umat}_{F}$                              & Frobenius norm of $\Umat$                                                              \\
      $\MX_{(n)}$                                     & mode-$n$ unfolding matrix of $\TX$                                                     \\
      $\khatrirao$                                    & Khatri-Rao product                                                                     \\
      $\hadamard$                                     & Hadamard product                                                                       \\
      $\llbracket\cdot\rrbracket$                     & Kruskal operator                                                                       \\
      \midrule
      \multicolumn{2}{l}{Notations for the Holt-Winters Method (Sections~\ref{sec:prelim:hw} and~\ref{sec:prelim:robust})}                     \\
      \midrule
      $l_{t},b_{t},s_{t}$                             & level, trend, and seasonal components at time $t$                                      \\
      $\alpha,\beta,\gamma$                           & \makecell[l]{smoothing parameters corresponding to the level,                          \\trend, and seasonal components}\\
      $\hsigma_{t}$                                   & scale of one-step-ahead forecast error at time $t$                                     \\
      $\Psi(\cdot)$, $\rho(\cdot)$                    & Huber $\Psi$-function and biweight $\rho$-function                                     \\
      \midrule
      \multicolumn{2}{l}{Notations for \method (Sections~\ref{sec:model} and~\ref{sec:method})}                                                \\
      \midrule
      $\TY$, $\TYt$                                   & input tensor and input subtensor at time $t$                                           \\
      $\TOmega$, $\TOmegat$                           & indicator tensor and indicator subtensor at time $t$                                   \\
      $\TO$, $\TOt$                                   & outlier tensor and outlier subtensor at time $t$                                       \\
      $\Lone$, $\Lm$                                  & temporal and seasonal smoothness constraint matrices                                   \\
      $\Un$                                           & mode-$n$ factor matrix                                                                 \\
      $\bSigma_{t}$                                   & one-step-ahead forecast error scale tensor at time $t$                                 \\
      $R$                                             & rank of CP factorization                                                               \\
      $m$                                             & seasonal period                                                                        \\
      $\mu$                                           & gradient step size                                                                     \\
      $\phi$                                          & smoothing parameter for $\bSigma_{t}$                                                  \\
      $\lambda_{1}$                                   & temporal smoothness control parameter                                                  \\
      $\lambda_{2}$                                   & seasonal smoothness control parameter                                                  \\
      $\lambda_{3}$                                   & sparsity control parameter for $\TO$ and $\TOt$                                        \\

      \bottomrule
    \end{tabular}
  }
\end{table}

In this section, we first give some preliminaries on tensors and introduce some notations (Section~\ref{sec:prelim:notation}). Then, we briefly describe CP factorization (Section~\ref{sec:prelim:factorization}) and the standard and robust versions of Holt-Winters method (Section~\ref{sec:prelim:hw} and~\ref{sec:prelim:robust}), which \method is based on.

\subsection{Tensors and Notations}\label{sec:prelim:notation}
Table~\ref{tab:symbols} lists the symbols frequently used in this paper.

\noindent{\textbf{Symbols and indexing}}: We denote scalars by lowercase letters, e.g., $u$, vectors by boldface lowercase letters, e.g., $\uvec$, matrices by boldface capital letters, e.g., $\Umat$, and tensors by boldface calligraphic letters, e.g., $\tensor{X}$.
The order of a tensor is the number of modes, also known as ways or dimensions.
Consider an $N$-way tensor $\TX\in\mathbb{R}^{\IonetoIN}$, where $I_{n}$ is the length of mode $n$.
Given the integer $1\leq i_{n}\leq I_{n}$ for each mode $n=1,\dots,N$, the $(i_{1},\dots,i_{N})$-th entry of the tensor $\TX$ is denoted by $x_{i_{1}\dots i_{N}}$.
For each matrix $\Umat$, we denote its $(i,j)$-th element by $u_{ij}$, its $i$-th row vector by $\uvec_{i}$, and its $j$-th column vector by $\tildeu_{j}$.
For each vector $\uvec$, we denote its $i$-th element by $u_{i}$.
The inverse, Moore-Penrose pseudoinverse, transpose, and Frobenius norm of $\Umat$ are denoted by $\Umat^{-1}$, $\Umat^{\dagger}$, $\Umat^{\top}$, and $\norm{\Umat}_{F}$, respectively.

\noindent{\textbf{Matricization}}: Matricization, also called unfolding, is the process of reordering the elements of a given $N$-way tensor $\TX\in\mathbb{R}^{\IonetoIN}$ into a matrix.
The mode-$n$ unfolding matrix of $\TX$, which is denoted by $\MX_{(n)}\in\mathbb{R}^{I_{n}\times(\prod_{i\neq n}^{N}I_{i})}$, is obtained by considering the $n$-th mode as the rows of the matrix and collapsing the other modes into the columns of the matrix.

\noindent{\textbf{Hadamard product}}: The Hadamard product of two matrices $\Umat$ and $\Wmat$ of the same size, which is denoted by $\Umat\hadamard\Wmat$, is their element-wise product.
The sequence of the Hadamard products $\Umat^{(N)}\hadamard\cdots\hadamard\Umat^{(1)}$ is denoted by $\bigHadamard_{n=1}^{N}\Umat^{(n)}$.
The Hadamard product is naturally extended to tensors.

\noindent{\textbf{Khatri-Rao product}}: Given matrices $\Umat\in\mathbb{R}^{I\times R}$ and $\Wmat\in\mathbb{R}^{J\times R}$, their Khatri-Rao product is denoted by $\Umat\khatrirao\Wmat$.
The result matrix, which is of size $IJ\times R$, is defined as:

{\small
\begin{equation}
  \Umat\khatrirao\Wmat=\begin{bmatrix}
    u_{11}\tildew_{1} & u_{12}\tildew_{2} & \cdots & u_{1R}\tildew_{R} \\
    u_{21}\tildew_{1} & u_{22}\tildew_{2} & \cdots & u_{2R}\tildew_{R} \\
    \vdots            & \vdots            & \ddots & \vdots            \\
    u_{I1}\tildew_{1} & u_{I2}\tildew_{2} & \cdots & u_{IR}\tildew_{R}
  \end{bmatrix}.
\end{equation}
}\normalsize
The sequence of the Khatri-Rao products $\Umat^{(N)}\khatrirao\cdots\khatrirao\Umat^{(1)}$ is denoted by $\bigKhatrirao_{n=1}^{N}\Umat^{(n)}$.

\subsection{\cp (CP) Factorization}\label{sec:prelim:factorization}

Among tensor factorization models, CP factorization, which is based on CP decomposition \cite{kolda2009tensor}, is most widely used due to its simplicity and effectiveness.
\begin{definition}[Rank-$1$ Tensor]
  A tensor $\TX\in\mathbb{R}^{\IonetoIN}$ is a rank-$1$ tensor if it can be expressed as the outer product of $N$ vectors (i.e., $\TX=\uvec^{(1)}\circ\uvec^{(2)}\circ\cdots\circ\uvec^{(N)}$, where $\uvec^{(n)}\in\mathbb{R}^{I_{n}}$).
  The outer product is defined as $(\uvec^{(1)}\circ\uvec^{(2)}\circ\cdots\circ\uvec^{(N)})_{i_{1}\dots i_{N}}=u_{i_{1}}^{(1)}\cdots u_{i_{N}}^{(N)}$ for all $1\leq n \leq N$ and $1\leq i_{n}\leq I_{n}$.
\end{definition}

\begin{definition}[CP Decomposition~\cite{carroll1970analysis,harshman1970foundations}]
  Given an $N$-way tensor $\TX\in\mathbb{R}^{\IonetoIN}$ and a positive integer $R$, CP decomposition of rank $R$ approximates $\TX$ as the sum of $R$ rank-$1$ tensors, as formulated in Eq.~\eqref{eq:cpd}.

  {\small
  \begin{equation}
    \TX\approx\sum_{r=1}^{R} \tildeur^{(1)}\circ\tildeur^{(2)}\circ\cdots\circ\tildeur^{(N)}=\ktensor, \label{eq:cpd}
  \end{equation}
  }\normalsize
  where $\Un=[\tildeu_{1}^{(n)},\dots,\tildeu_{R}^{(n)}]\in\mathbb{R}^{I_{n}\times R}$ is called the mode-$n$ factor matrix, which can also be represented by row vectors as $\Un=[\uvec_{1}^{(n)},\dots,\uvec_{I_{n}}^{(n)}]^{\top}$.
\end{definition}

\begin{definition}[CP Factorization of Incomplete Tensors~\cite{acar2011scalable,tomasi2005parafac,shin2016fully}]
  Consider an $N$-way incomplete tensor $\TX\in\mathbb{R}^{\IonetoIN}$ and a rank $R$.
  Let $\TOmega$ be the binary tensor of the same size as $\TX$ each of whose entry indicates whether the corresponding entry in $\TX$ is observed, i.e., for all $1\leq n \leq N$ and $1\leq i_{n}\leq I_{n}$,

  {\small
      \begin{equation}
        \omega_{i_{1}\dots i_{N}}=\begin{cases}
          1 & \textnormal{if}\ x_{i_{1}\dots i_{N}} \ \textnormal{is known},   \\
          0 & \textnormal{if}\ x_{i_{1}\dots i_{N}} \ \textnormal{is missing},
        \end{cases}
      \end{equation}
    }\normalsize
  CP factorization is to find the factor matrices $\Uone,\dots,\UN$ that minimize (\ref{eq:prelim:cp}), where only the observed entries are taken into consideration.

    {\small
      \begin{equation} \label{eq:prelim:cp}
        \min_{\UonetoUN} \frac{1}{2}\norm{\TOmega\hadamard(\TX-\ktensor)}_{F}^{2}.
      \end{equation}
    }\normalsize
\end{definition}

For more details, \cite{kolda2009tensor,sidiropoulos2017tensor} provide comprehensive reviews of tensor factorization.

\subsection{Holt-Winters Method} \label{sec:prelim:hw}
The Holt-Winters (HW) method~\cite{holt2004forecasting,winters1960forecasting} is an effective forecasting method for time series with trend and seasonality.
There are two variations of it~\cite{hyndman2018forecasting}: the additive method, which is suitable when the seasonal variations are roughly constant, and the multiplicative method, which is preferred when the seasonal variations change proportional to the level of time series.
We focus on the additive model in this paper.

The additive HW method consists of one forecast equation and three smoothing equations that are for the \textit{level} (e.g., $l_{t}$), the \textit{trend} (e.g., $b_{t}$), and the \textit{seasonal component} (e.g., $s_{t}$), respectively, with corresponding smoothing parameters $0\leq\alpha\leq 1$, $0\leq\beta\leq 1$, and $0\leq\gamma\leq 1$.
The smoothing equations are defined as:

{\small
\begin{subequations} \label{eq:hw_smoothing}
  \begin{align}
    l_{t} & = \alpha(y_{t}-s_{t-m}) + (1-\alpha)(l_{t-1}+b_{t-1}), \label{eq:hw_level} \\
    b_{t} & = \beta(l_{t}-l_{t-1}) + (1-\beta)b_{t-1}, \label{eq:hw_trend}             \\
    s_{t} & = \gamma(y_{t}-l_{t-1}-b_{t-1}) + (1-\gamma)s_{t-m}. \label{eq:hw_season}
  \end{align}
\end{subequations}
}\normalsize
The forecast equation is defined as:

{\small
\begin{equation} \label{eq:hw_forecast}
  \hat{y}_{t+h|t} = l_{t} + hb_{t} + s_{t+h-m(\lfloor \frac{h-1}{m}\rfloor+1)},
\end{equation}
}\normalsize
where $\hat{y}_{t+h|t}$ is the $h$-step-ahead forecast of time series $\mathbf{y}$ at time $t$, and $m$ is the seasonal period.
Note that $\lfloor \frac{h-1}{m}\rfloor+1$ ensures that the estimates of the seasonal components used for forecasts are obtained in the last season of the time series.


Forecasts produced by the HW method are weighted averages of past observations, with the weights decreasing exponentially as the observations get older.
To use the HW method, the smoothing parameters and initial values of level, trend, and seasonal component need to be estimated.
We define the residuals of the one-step-ahead forecasts as $e_{t}=y_{t}-\hat{y}_{t|t-1}$ for $t=1,\dots,T$, where $T$ is the last time of the time-series.
Then, we can find the unknown parameters by minimizing the sum of squared errors (SSE) defined as $\sum_{t=1}^{T} e_{t}^{2}$~\cite{hyndman2018forecasting}.

\subsection{Robust Holt-Winters Forecasting} \label{sec:prelim:robust}
The HW model is vulnerable to unusual events or outliers since the smoothing equations (\ref{eq:hw_smoothing}) involve current and past values of the time series including the outliers.
Thus, Gelper el al.~\cite{gelper2010robust} proposed a robust HW method based on a pre-cleaning mechanism that identifies and downweights outliers before updating the model.
The currupted observation $y_{t}$ is replaced with a `cleaned' version $y_{t}^{*}$ by the follwing rule:

{\small
\begin{equation}\label{eq:prelim:pre_cleaning}
  y_{t}^{*}=\Psi\Big(\frac{y_{t}-\hat{y}_{t|t-1}}{\hsigma_{t}}\Big)\hsigma_{t} + \hat{y}_{t|t-1},
\end{equation}
}\normalsize
where \small$\Psi(x)=\begin{cases}
    x        & \textnormal{if}\ |x|<k, \\
    sign(x)k & \textnormal{otherwise},
  \end{cases}$\normalsize \ and $\hat{\sigma}_{t}$ is an estimated scale of one-step-ahead forecast error at time $t$.
Note that, in the Huber $\Psi$-function~\cite{maronna2019robust}, the magnitude of $x$ is upper bounded by $k$.
The equation (\ref{eq:prelim:pre_cleaning}) can be interpreted as identifying unexpected high or low observations as outliers and replacing them by more likely values.

The time varying error scale $\hsigma_{t}$ in (\ref{eq:prelim:pre_cleaning}) is updated by the following equation:

{\small
\begin{equation}
  \hsigma^{2}_{t}=\phi\rho\Big(\frac{y_{t}-\hat{y}_{t|t-1}}{\hsigma_{t-1}}\Big)\hsigma^{2}_{t-1} + (1-\phi)\hsigma^{2}_{t-1},
\end{equation}
}\normalsize
where $\phi$ is the smoothing parameter and $\rho(x)$ is the biweight $\rho$-function~\cite{maronna2019robust}, which is defined as:

{\small
\begin{equation} \label{eq:biweight_rho}
  \rho(x)=\begin{cases}
    c_{k}(1-(1-(x/k)^{2})^{3}) & \textnormal{if}\ |x|\leq k, \\
    c_{k}                      & \textnormal{otherwise}.
  \end{cases}
\end{equation}
}\normalsize
It is common to set $k$ to $2$ and $c_{k}$ to $2.52$ in the Huber $\Psi$-function and the biweight $\rho$-function \cite{gelper2010robust}.


\section{Proposed Factorization Model}
\label{sec:model}

\begin{figure*}[!t]
  \vspace{-4mm}
\end{figure*}

In this section, we formally define our model of robust factorization of real-world tensor streams.
We leverage two types of temporal properties in real-world tensor streams: \textbf{temporal smoothness}, the property that the successive values tend to be close, and \textbf{seasonal smoothness}, the property that the values between consecutive seasons tend to be close.
\change{For example, the current indoor temperature is likely to be similar to that of $10$ minutes ago and close to that of yesterday's from the same time.
  As another example, the number of taxi trips from Near North Side to Lake View in Chicago this Friday is probably similar to that of last week.
  If we can find such temporal patterns in tensor streams, then we can utilize them to estimate the missing values and detect outliers.}
In order to take into account these characteristics in tensor factorization, we impose smoothness constraints on the temporal factor matrix.

\subsection{Proposed Factorization Model for Static Tensors}

First, we consider a case where the input tensor is static.
Without loss of generality, we assume an $N$-way partially observed tensor $\TY\in\mathbb{R}^{\IonetoIN}$ where the $N$-th mode is temporal, and the others are non-temporal.
Since the entries of $\TY$ might be contaminated by erroneous values, we assume that $\TY$ is a mixture of $\TX$, a clean low-rank tensor, and $\TO$, a sparse outlier tensor.
This is formulated as $\TY=\TX+\TO$, where both $\TX$ and $\TO$ are the same size as $\TY$.
In addition to $\TO$, we seek factor matrices $\allU$ that minimize (\ref{eq:batch:smoothObj}).

  {\small
    \begin{multline} \label{eq:batch:smoothObj}
      C(\allU,\TO) = \norm{\TOmega\hadamard(\TY-\TO-\TX)}_{F}^{2} \\
      + \lambda_{1}\norm{\Lone\UN}_{F}^{2} + \lambda_{2}\norm{\Lm\UN}_{F}^{2} + \lambda_{3}\norm{\TO}_{1},\\
      \textnormal{subject to} \ \TX=\ktensor, \norm{\tildeurn}_{2}=1, \\
      \forall{r}\in\{1,\cdots,R\}, \forall{n}\in\{1,\cdots,N-1\},
    \end{multline}
  }\normalsize
where $\lambda_{1}$ and $\lambda_{2}$ are the smoothness control parameters, $\lambda_{3}$ is the sparsity control parameter, $m$ is the seasonal period, and each matrix $\ML_{i}\in\mathbb{R}^{(I_{N}-i)\times I_{N}}$ is a smoothness constraint matrix.
As in \cite{yokota2016smooth}, in $\ML_{i}$, $l_{nn}=1$, and $l_{n(n+i)}=-1 \ \textnormal{for all}\ 0\leq n\leq I_{N}-i$, and the other entries are $0$.
Note that the columns of the non-temporal factor matrices $\ntU$ are normalized to one, while those of the temporal factor matrix $\UN$ remain unnormalized.


The first and second penalty terms in (\ref{eq:batch:smoothObj}) encourage the temporal and seasonal smoothness in $\UN$, respectively.
Specifically, minimizing $\norm{\Lone\UN}_{F}^{2}=\sum_{i=1}^{I_{N}-1}\norm{\uvec_{i}^{(N)}-\uvec_{i+1}^{(N)}}^{2}$ enforces that the values of consecutive temporal vectors do not change dramatically. 
Similarly, 
minimizing $\norm{\Lm\UN}_{F}^{2}=\sum_{i=1}^{I_{N}-m}\norm{\uvec_{i}^{(N)}-\uvec_{i+m}^{(N)}}^{2}$ imposes that the values of temporal vectors of successive seasons do not change abruptly.
The last penalty term in (\ref{eq:batch:smoothObj}) enforces sparsity of $\TO$.

\change{After solving the optimization problem, the missing entries can be recovered with the values of the low-rank approximation $\TXhat=\ktensor$ in the same position.}

\subsection{Proposed Factorization Model for Dynamic Tensors}

Now, we consider a case where the input tensor is dynamic. Specifically, we assume that $(N-1)$-way incomplete subtensor $\TY_{t}\in\mathbb{R}^{\IonetoINminus}$, where $t=1,2,\cdots$, are concatenated sequentially over time.
Then, at each time $t$, we aim to minimize (\ref{eq:streaming:smoothObj}). Note that if $t=I_{N}$, then (\ref{eq:streaming:smoothObj}) is equivalent to (\ref{eq:batch:smoothObj}), which is the batch optimization problem.

  {\small
    \begin{multline} \label{eq:streaming:smoothObj}
      \hspace{-3mm} C_{t}(\ntU,\{\uN_{\tau},\TO_{\tau}\}_{\tau=1}^{t}) = \\
      \sum_{\tau=1}^{t}\bigg[\norm{\TOmega_{\tau}\hadamard(\TY_{\tau}-\TO_{\tau}-\TX_{\tau})}_{F}^{2}
        + \lambda_{1}\norm{\ptau}_{F}^{2} + \lambda_{2}\norm{\qtau}_{F}^{2} + \lambda_{3}\norm{\TO_{\tau}}_{1}\bigg],\\
      \textnormal{subject to} \ \TX_{\tau}=\ktensorTau, \norm{\tildeurn}_{2}=1, \\
      \forall{r}\in\{1,\cdots,R\}, \forall{n}\in\{1,\cdots,N-1\},
    \end{multline}
  }\normalsize
where {\small$\ptau=\fv_{\tau-1}^{(N)}-\fv_{\tau}^{(N)}$} if $\tau>1$ and $\mathbf{0}$ otherwise, and {\small$\qtau=\fv_{\tau-m}^{(N)}-\fv_{\tau}^{(N)}$} if $\tau>m$ and $\mathbf{0}$ otherwise.
The temporal vector $\uvec_{\tau}^{(N)}$ is the $\tau$-th row vector of the temporal factor matrix and the temporal component of the subtensor $\TY_{\tau}$.

\section{Proposed Method}
\label{sec:method}

\begin{figure*}[!t]
    \vspace{-4mm}
\end{figure*}

In this section, we introduce \method, a robust streaming tensor factorization and completion algorithm for seasonal tensor streams.
\method aims to find factor matrices and outlier subtensors that minimize (\ref{eq:streaming:smoothObj}) incrementally in an online manner.
\method consists of the following three steps.
\begin{enumerate}[leftmargin=*]
    \item {\bf Initialization}: initializes our model using the data streamed over a short time interval (e.g. 3 seasons),
    \item {\bf Fitting the Holt-Winters model}: decomposes the temporal factor into the level, trend, and seasonal components,
    \item {\bf Dynamic Update}: repeatedly observes a newly arrived subtensor and updates our model robustly to outliers using the trained seasonal patterns.
\end{enumerate}

\subsection{Initialization} \label{sec:method:init}

\begin{algorithm}[t]
    \small
    \caption{Initialization}\label{alg:init}
    \DontPrintSemicolon
    \SetNoFillComment
    \SetKwInOut{Input}{Input}
    \SetKwInOut{Output}{Output}
    \SetKwProg{Procedure}{Procedure}{}{}

    \Input{$\{\TY_{t},\TOmega_{t}\}_{t=1}^{\ti}$, $R$, $m$, $\lambda_{1}$, $\lambda_{2}$, $\lambda_{3}$}
    \Output{(1) completed tensor $\TXhat_{init}=\{\TXhat_{t}\}_{t=1}^{\ti}$, \newline (2) factor matrices $\allFM$} 
    $\TY_{init}\leftarrow[\TY_{1},\TY_{2},\cdots,\TY_{\ti}]$\\
    $\TOmega_{init}\leftarrow[\TOmega_{1},\TOmega_{2},\cdots,\TOmega_{\ti}]$\\
    $\TO_{init}\leftarrow\boldsymbol{\emptyset}_{\IonetoINminus\times\ti}$\\
    randomly initialize $\allFM$\\
    $\lambda_{3,init}\leftarrow\lambda_{3}$\\
    \Repeat{$\frac{\norm{\TXhat_{pre}-\TXhat_{init}}_{F}}{\norm{\TXhat_{pre}}_{F}}<tol$}{ \label{alg:init:start}
        $\TXhat_{init},\allFM\leftarrow\textnormal{\methodALS}(\TO_{init}, ..., \allFM)$ \\
        $\TO_{init}\leftarrow \textnormal{SoftThresholding}(\TOmega_{init}\hadamard(\TY_{init}-\TXhat_{init}), \lambda_{3})$\\
        $\lambda_{3}\leftarrow d\cdot\lambda_{3}$\\
        \If{$\lambda_{3}<\lambda_{3,init}/100$}{
            $\lambda_{3}\leftarrow\lambda_{3,init}/100$\\
        }
    } \label{alg:init:end}
    \normalsize
\end{algorithm}
\begin{algorithm}[t]
    \small
    \caption{\methodALS: Batch Update in \method}\label{alg:als}
    \DontPrintSemicolon
    \SetNoFillComment
    \SetKwInOut{Input}{Input}
    \SetKwInOut{Output}{Output}
    \SetKwProg{Procedure}{Procedure}{}{}

    \Input{(1) $\TO$, $\TY$, $\TOmega$, $R$, $m$, $\lambda_{1}$, $\lambda_{2}$, \newline (2) initial factor matrices $\allFM$}
    \Output{(1) completed tensor $\TXhat$, \newline (2) updated factor matrices $\allFM$}

    $\TY^{*}=\TY-\TO$\\
    \Repeat{$\Delta fitness<tol$}{
    \For{$n=1,\cdots,N-1$}{
    \For{$i_{n}=1,\cdots,I_{n}$}{
    Calculate $\Binn$ and $\cinn$ using (\ref{eq:als:Binn}) and (\ref{eq:als:cinn})\\
    Update $\uinn$ using (\ref{eq:als:nontemporal})
    }
    \For{$r=1,\cdots,R$}{
        $\tildeurN\leftarrow\tildeurN\cdot\norm{\tildeurn}_{2}$\\
        $\tildeurn\leftarrow\tildeurn/\norm{\tildeurn}_{2}$
    }
    }
    \For{$i_{N}=1,\cdots,I_{N}$}{
    Calculate $\BiNN$ and $\ciNN$ using (\ref{eq:als:Binn}) and (\ref{eq:als:cinn})\\
    Update $\uiNN$ using (\ref{eq:als:temporal})
    }
    $\TXhat\leftarrow\ktensor$\\
    $fitness\leftarrow 1-\frac{\norm{\TOmega\hadamard(\TY^{*}-\TXhat)}_{F}}{\norm{\TOmega\hadamard\TY^{*}}_{F}}$
    }
    \normalsize
\end{algorithm}

We first initialize all the factor matrices $\allFM$ by solving the batch optimization problem in (\ref{eq:batch:smoothObj}) using a subset of the corrupted tensor data over a short period of time.
Let $\ti$ denotes the start-up period.
We use the first $3$ seasons for initialization (i.e. $\ti=3m$), following the general convention for initializing the Holt-Winters method~\cite{hyndman2018forecasting}.

Algorithm~\ref{alg:init} describes the overall procedure of the initialization step.
First, we make a batch tensor $\TY_{init}\in\mathbb{R}^{\IonetoINminus\times\ti}$ by concatenating $\ti$ subtensors $\TY_{t}$.
Next, we factorize the outlier-removed tensor $\TY^{*}=\TY_{init}-\TO_{init}$ to get factor matrices, including a temporally and seasonally smooth temporal factor matrix, using \methodALS (Algorithm~\ref{alg:als}), which we describe below.
After that, we update the outlier tensor $\TO_{init}$ by applying the element-wise soft-thresholding with the threshold $\lambda_{3}$, defined as (\ref{eq:softthresholding}), to $\TOmega_{init}\hadamard(\TY_{init}-\TXhat_{init})$.
\begin{equation} \label{eq:softthresholding}
    \textnormal{SoftThresholding}(x,\lambda_{3})=\textnormal{sign}(x)\cdot\max(|x|-\lambda_{3},0).
\end{equation}
These two tasks, \methodALS and SoftThresholding, are repeated until the relative change of the recovered tensor $\TXhat_{init}$ in two successive iterations is less than the tolerance.
Note that, we update $\lambda_{3}$ to $d\cdot\lambda_{3}$ after each soft-thresholding.
This helps $\TXhat_{init}$ converge quickly.
Conceptually, it can be thought of as filtering out large outliers in the first few iterations and small outliers in the later iterations.
We set $d=0.85$.

In \methodALS, we use the alternating least squares (ALS) method to minimize the objective function in (\ref{eq:batch:smoothObj}) as its name implies.
The ALS approach updates the factor matrices alternately in such a way that one matrix is updated while fixing the others.
We update the non-temporal factor matrices one by one and row by row, as formulated in Theorem~\ref{theorem:als:nontemporal}.
\begin{theorem}[Update rule for $\uinn$] \label{theorem:als:nontemporal}
    For each row $\uinn$  of each non-temporal matrix $\Un$,  (\ref{eq:als:nontemporal}) holds.
        {\small
            \begin{equation} \label{eq:als:nontemporal}
                \argmin\nolimits_{\uinn}C(\allU,\TO)={\Binn}^{-1}\cinn,
            \end{equation}
        }\normalsize
    where
        {\small
            \begin{align} \label{eq:als:Binn}
                \Binn & = \ \sum_{\mathclap{\inOmega}}\ \ \ \ \ \ \bigHadamardBinn(\bigHadamardBinn)^{\top},
                \\ \label{eq:als:cinn}
                \cinn & = \ \sum_{\mathclap{\inOmega}}\ \ \ystar\bigHadamardBinn,
                \vspace{-3mm}
            \end{align}
        }\normalsize
    \hspace{-1mm}$\ystar=y_{i_{1},\dots,i_{N}}-o_{i_{1},\dots,i_{N}}$, and $\Omegainn$ is the set of indices of the observed entries whose $n$-th mode's index is $i_{n}$.
\end{theorem}

\begin{proof}
    For all $1\leq i_{n}\leq I_{n}$ and $1\leq j \leq R$,
    {\small
            \begin{equation*}
                \frac{\partial C}{\partial u^{(n)}_{i_{n}j}} = \ \ \ \sum_{\mathclap{\inOmega}}\ \ \ 2\Big(\big(\sum_{r=1}^{R}\prod_{l=1}^{N}u^{(l)}_{i_{l}r}-\ystar\big)\prod_{l\neq n}u^{(l)}_{i_{l}j}\Big)=0.
            \end{equation*}
        }\normalsize
    It is equivalent to
        {\footnotesize
            \begin{multline*}
                \sum_{\mathclap{\inOmegasplit}}\ \  \Big(\sum_{r=1}^{R}\big(u^{(n)}_{i_{n}r}\prod_{l\neq n}u^{(l)}_{i_{l}r}\big)\prod_{l\neq n}u^{(l)}_{i_{l}j}\Big)=
                \sum_{\mathclap{\inOmegasplit}}\ \  \Big(\ystar\prod_{l\neq n}u^{(l)}_{i_{l}j}\Big), \forall j.
            \end{multline*}
        }\normalsize
    Then, vectorize the equation as:
    {\small
    \begin{equation*}
        \Binn\uinn=\cinn \Leftrightarrow \uinn={\Binn}^{-1}\cinn.\qedhere
    \end{equation*}
    }\normalsize
\end{proof}

We next update the temporal factor matrix $\UN$ row by row as formulated in Theorem~\ref{theorem:als:temporal}.
\begin{theorem}[Update Rule for $\uiNN$] \label{theorem:als:temporal}
    For each row $\uiNN$  of the temporal matrix $\UN$,  (\ref{eq:als:temporal}) holds.
\end{theorem}
\begin{proof}
    For all $1\leq i_{N}\leq I_{N}$ and $1\leq j \leq R$,
    {\small
            \begin{multline} \label{eq:als:temporal:derivative}
                \frac{\partial C}{\partial u^{(N)}_{i_{N}j}} = \ \ \ \sum_{\mathclap{\inOmegaN}}\ \ \ 2\Big(\big(\sum_{r=1}^{R}\prod_{l=1}^{N}u^{(l)}_{i_{l}r}-\ystar\big)\prod_{l\neq N}u^{(l)}_{i_{l}j}\Big) \\
                + 2K_{i_{N}j} + 2H_{i_{N}j} = 0,
            \end{multline}
        }\normalsize
    where $K_{i_{N}j}$ and $H_{i_{N}j}$ are defined as (\ref{eq:kappa_eta}).
    We vectorize the solution of (\ref{eq:als:temporal:derivative}) as (\ref{eq:als:temporal}).
    See the supplementary document~\cite{sofia2020supple} for a full proof.
\end{proof}
The iterations are repeated until the fitness change in two consecutive iterations is less than the tolerance.


\begin{figure*}[!t]
    \vspace{-4mm}
    \hrulefill
    \small
    \begin{align}
         & \uiNN=
        \begin{cases}
            \big(\BiNN+(\lambda_{1}+\lambda_{2})\eyeR\big)^{-1}\big(\ciNN+\lambda_{1}\uvec_{i_{N}+1}^{(N)}+\lambda_{2}\uvec_{i_{N}+m}^{(N)}\big)                                                  & \textnormal{if}\ i_{N}=1,                        \\
            \big(\BiNN+(2\lambda_{1}+\lambda_{2})\eyeR\big)^{-1}\big(\ciNN+\lambda_{1}(\uvec_{i_{N}-1}^{(N)}+\uvec_{i_{N}+1}^{(N)})+\lambda_{2}\uvec_{i_{N}+m}^{(N)}\big)                         & \textnormal{else if}\ 1<i_{N}\leq m,             \\
            \big(\BiNN+2(\lambda_{1}+\lambda_{2})\eyeR\big)^{-1}\big(\ciNN+\lambda_{1}(\uvec_{i_{N}-1}^{(N)}+\uvec_{i_{N}+1}^{(N)})+\lambda_{2}(\uvec_{i_{N}-m}^{(N)}+\uvec_{i_{N}+m}^{(N)})\big) & \textnormal{else if}\ m<i_{N}\leq I_{N}-m,       \\
            \big(\BiNN+(2\lambda_{1}+\lambda_{2})\eyeR\big)^{-1}\big(\ciNN+\lambda_{1}(\uvec_{i_{N}-1}^{(N)}+\uvec_{i_{N}+1}^{(N)})+\lambda_{2}\uvec_{i_{N}-m}^{(N)}\big)                         & \textnormal{else if}\ I_{N}-m<i_{N}\leq I_{N}-1, \\
            \big(\BiNN+(\lambda_{1}+\lambda_{2})\eyeR\big)^{-1}\big(\ciNN+\lambda_{1}\uvec_{i_{N}-1}^{(N)}+\lambda_{2}\uvec_{i_{N}-m}^{(N)}\big)                                                  & \textnormal{otherwise}.                          \\
        \end{cases} \label{eq:als:temporal}                                                              \\
         & K_{i_{N}j}  =\begin{cases}
            \lambda_{1}(\uNiNj-\uNiNjplus)                         & \textnormal{if}\ i_{N}=1,      \\
            -\lambda_{1}(\uNiNjminus-\uNiNj)                       & \textnormal{if}\  i_{N}=I_{N}, \\
            2\lambda_{1}\uNiNj-\lambda_{1}(\uNiNjminus+\uNiNjplus) & \textnormal{otherwise},        \\
        \end{cases} \text{~}  H_{i_{N}j}=\begin{cases}
            \lambda_{2}(\uNiNj-\uNiNjplusm)                          & \textnormal{if}\ 1\leq i_{N}\leq m, \\
            -\lambda_{2}(\uNiNjminusm-\uNiNj)                        & \textnormal{if}\ i_{N} > I_{N}-m,   \\
            2\lambda_{2}\uNiNj-\lambda_{2}(\uNiNjminusm+\uNiNjplusm) & \textnormal{otherwise}.             \\
        \end{cases} \label{eq:kappa_eta}
    \end{align}
    \normalsize
    \hrulefill
\end{figure*}

\subsection{Fitting the Holt-Winters model} \label{sec:method:hw}
Through initialization, we can get a temporally and seasonally smooth temporal factor matrix $\UN$.
Each of the column vectors of $\UN$ (i.e., $\tildeuN_{1},\tildeuN_{2},\cdots,\tildeuN_{R}$) can be thought of as seasonal time series of length $\ti$ with seasonal period $m$.
In this step, we capture seasonal patterns and trend from each of $\tildeuN_{1},\tildeuN_{2},\cdots,\tildeuN_{R}$ by fitting the additive Holt-Winters model (see Section~\ref{sec:prelim:hw} for details).

We optimize the HW model by BFGS-B~\cite{byrd1995limited}, which belongs to quasi-Newton methods for solving non-linear optimization problem with box constraints on variables.
For each $\tildeuN_{r}$, we can get the level $\tilde{\level}_{r}$, trend $\tilde{\trend}_{r}$, and seasonal component $\tilde{\seasonal}_{r}$, and the corresponding smoothing parameters (i.e., $\alpha_{r}$, $\beta_{r}$, and $\gamma_{r}$, which are $r$-th entry of vectors $\sAlpha,\sBeta,\sGamma\in\mathbb{R}^{R}$, respectively).
As seen in (\ref{eq:hw_smoothing}) and  (\ref{eq:hw_forecast}), the HW model requires only the last values of the level and trend component, and the last one season of the seasonal component.
Thus, only the last time of level $\level_{\ti}$ and trend $\trend_{\ti}$, and the last $m$ values of the seasonal component (i.e., $\seasonal_{\ti-m+1},\dots,\seasonal_{\ti}$) are needed.





\subsection{Dynamic Update} \label{sec:method:update}

\DecMargin{0.5em}
\begin{algorithm}[t]
    \small
    \caption{Dynamic Updates in \method}\label{alg:dynamic}
    \DontPrintSemicolon
    \SetNoFillComment
    \SetKwInOut{Input}{Input}
    \SetKwInOut{Output}{Output}
    \SetKwProg{Function}{Function}{}{}
    \SetKwFunction{Preclean}{PRE_CLEANING}
    \SetKwFunction{Delete}{DELETE}

    \Input{(1) $\{\TY_{t},\TOmega_{t}\}_{t=\ti+1}^{\infty}$, $R$, $m$, $\lambda_{1}$, $\lambda_{2}$, $\lambda_{3}$, $\mu$, $\phi$, \newline (2) $\ntUnti$, $\{\uNt\}_{t=\ti-m+1}^{\ti}$, \newline (3) HW factors $\level_{\ti}$, $\trend_{\ti}$, $\{\seasonal_{t}\}_{t=\ti-m+1}^{\ti}$, $\sAlpha$, $\sBeta$, $\sGamma$}

    $\bSigma_{\ti}\leftarrow\lambda_{3}/100\times\boldsymbol{1}_{\IonetoINminus}$\\

    \For{$t=\ti+1,\ti+2,\cdots$}{
        $\uNhat_{t|t-1}\leftarrow\level_{t-1}+\trend_{t-1}+\seasonal_{t-m}$\\
        $\TYthat\leftarrow\llbracket\ntUntminus;\uNhat_{t|t-1}\rrbracket$\\
        Estimate $\TOt$ with (\ref{eq:dynamic:outliers})\\
        Update $\bSigma_{t}$ with (\ref{eq:update:bSigma})\\
        \For{$n=1,\cdots,N-1$}{
            Update $\Un_{t}$ using (\ref{eq:gradient:nontemp})
        }
        Update $\uN_{t}$ using (\ref{eq:gradient:temp})\\
        Update $\level_{t}$, $\trend_{t}$, $\seasonal_{t}$ using (\ref{eq:update:hwfactors})\\
        $\TXthat\leftarrow\ktensort$\\
    }
    \normalsize
\end{algorithm}

We let $\ntUnt$ be the non-temporal factor matrices after processing the $t$-th subtensor $\TYt$.
At time $t$, we receive $\TYt$ and have the previous estimates of the non-temporal factor matrices $\ntUntminus$ and the previous $m$ estimates of the temporal vectors $\uN_{t-m},\cdots,\uN_{t-1}$.
We also have the previous level and trend components $\level_{t-1}$ and $\trend_{t-1}$, and the previous $m$ seasonal components $\seasonal_{t-m},\cdots,\seasonal_{t-1}$.
Through the following steps, we can update our factorization model and impute the missing values incrementally. 
Algorithm~\ref{alg:dynamic} describes the procedure of the dynamic update step.

\subsubsection{Estimate $\TOt$}
We first estimate the outlier subtensor $\TOt$.
We predict the temporal vector by one-step-ahead Holt-Winters' forecast (see Section~\ref{sec:prelim:hw} for details) as follows:

{\small
\begin{equation} \label{eq:forecast:uNhat}
    \uNhat_{t|t-1}=\level_{t-1}+\trend_{t-1}+\seasonal_{t-m}.
\end{equation}
}\normalsize
We then predict the next subtensor $\TYthat$ as follows:
{\small
\begin{equation} \label{eq:forecast:Ythat}
    \TYthat=\llbracket\ntUntminus,\uNhat_{t|t-1}\rrbracket.
\end{equation}
}\normalsize
We regard the observations that deviate significantly from the prediction $\TYthat$ as outliers.
Specifically, we identify the outliers by checking whether the difference between observation and prediction is greater than twice of the scale of error.
By extending the Gelper's pre-cleaning approach (see Section~\ref{sec:prelim:robust}) to a tensor, we can estimate $\TOt$ as follows:

{\small
\begin{align}
    \TYtstar             & =\Psi\Big(\frac{\TYt-\TYthat}{\bSigma_{t-1}}\Big)\bSigma_{t-1}+\TYthat = \TYt-\TOt, \nonumber             \\
    \Leftrightarrow \TOt & = \TYt-\TYthat-\Psi\Big(\frac{\TYt-\TYthat}{\bSigma_{t-1}}\Big)\bSigma_{t-1}, \label{eq:dynamic:outliers}
\end{align}
}\normalsize
where $\Psi(\cdot)$ is the element-wise Huber $\Psi$-function (we set $k=2$), and $\bSigma_{t-1}\in\mathbb{R}^{\IonetoINminus}$ is an error scale tensor each of whose entries is the scale of one-step-ahead forecast error in the corresponding entry.

In order to enable our model to adapt, we update the error scale tensor as follows:

{\small
\begin{equation} \label{eq:update:bSigma}
    \bSigma^{2}_{t}=\phi\rho\Big(\frac{\TYt-\TYthat}{\bSigma_{t-1}}\Big)\bSigma^{2}_{t-1}+(1-\phi)\bSigma^{2}_{t-1},
\end{equation}
}\normalsize
where $0\leq\phi\leq 1$ is a smoothing parameter and $\rho(x)$ is the element-wise biweight $\rho$-function, defined as (\ref{eq:biweight_rho}).
We set $k=2$ and $c_{k}=2.52$ for the biweight $\rho$-function.
Note that, the main difference between Gelper's approach and our method is that our method rejects outliers first and updates the error scale tensor, while Gelper's approach updates the error scale first.
The reason is that the error scale tensor can be contaminated by extremely large outliers if the error scale updates first.
We calculate the error scale for each entry because the variation may differ in different entries.
The initial value of all entries in $\bSigma_{\ti}$ are set to $\lambda_{3}/100$.

\subsubsection{Update $\ntUnt$}
Ideally, we need to update the non-temporal factor matrices considering all the historical data as seen in (\ref{eq:streaming:smoothObj}).
However, this is not feasible since we are dealing with a tensor stream whose length could be infinite.
We therefore focus only on the current input subtensor $\TYt$ and update each of the factor matrices using gradient descent (GD).
To this end, we define a new cost function $\ft$ considering only the $t$-th summand of (\ref{eq:streaming:smoothObj}) as follows:

{\small
\begin{multline} \label{eq:ft}
    \ft(\ntU,\uN) = \norm{\TOmega_{t}\hadamard(\TYt-\TOt-\ktensoruN)}_{F}^{2} \\
    + \lambda_{1}\norm{\uN_{t-1}-\uN}_{F}^{2} + \lambda_{2}\norm{\uN_{t-m}-\uN}_{F}^{2} + \lambda_{3}\norm{\TOt}_{1}.
\end{multline}
}\normalsize
Let $\TRt$ be a residual subtensor defined as $\TOmegat\hadamard(\TYt-\TOt-\TYthat)$.
The non-temporal factor matrices are updated by taking a step of size $\mu$ in the direction of minimizing the cost function in (\ref{eq:ft}) as follows:

{\small
\begin{align} \label{eq:gradient:nontemp}
    \Un_{t} & = \Un_{t-1}-\mu\frac{\partial\ft(\ntUntminus,\uNhat_{t|t-1})}{\partial\Un} \nonumber \\
            & = \Un_{t-1}+2\mu\mathbf{R}_{(n)}
    \bigKhatrirao_{l=1,l\neq n}^{N-1}\Umat^{(l)}_{t-1}\cdot\textnormal{diag}(\uNhat_{t|t-1}),
\end{align}
}\normalsize
where $\mathbf{R}_{(n)}$ is the mode-$n$ matricization of $\TRt$. 

\subsubsection{Update $\uN_{t}$}
Next, we update the temporal vector $\uN_{t}$ by a gradient descent step of size $\mu$ as follows:

{\small
\begin{align} \label{eq:gradient:temp}
    \uN_{t} & = \uNhat_{t|t-1}-\mu\frac{\partial\ft(\ntUntminus,\uNhat_{t|t-1})}{\partial\uN} \nonumber \\
            &
    \begin{aligned}
        {}=\uNhat_{t|t-1}+2\mu\Big[(\bigKhatrirao_{n=1}^{N-1}\Umat^{(n)}_{t-1})^{\top}\cdot\textnormal{vec}(\TRt)+ \lambda_{1}\uN_{t-1} \\
            +\lambda_{2}\uN_{t-m}-(\lambda_{1}+\lambda_{2})\uNhat_{t|t-1})\Big],
    \end{aligned}
\end{align}
}\normalsize
where $\textnormal{vec}(\cdot)$ is the vectorization operator.


\subsubsection{Update $\level_{t},\trend_{t},\seasonal_{t}$}
We update the level, trend, and seasonal components of the Holt-Winters model with the updated temporal vector by (\ref{eq:update:hwfactors}), which extends (\ref{eq:hw_smoothing}) to vectors:

{\small
\begin{subequations} \label{eq:update:hwfactors}
    \begin{equation} \label{eq:update:level}
        \level_{t} = \textnormal{diag}(\sAlpha)(\uN_{t}-\seasonal_{t-m})+(\eyeR-\textnormal{diag}(\sAlpha))(\level_{t-1}+\trend_{t-1}),
    \end{equation}
    \begin{equation} \label{eq:update:trend}
        \trend_{t} = \textnormal{diag}(\sBeta)(\level_{t}-\level_{t-1})+(\eyeR-\textnormal{diag}(\sBeta))\trend_{t-1},
    \end{equation}
    \begin{equation} \label{eq:update:seasonal}
        \seasonal_{t} = \textnormal{diag}(\sGamma)(\uN_{t}-\level_{t-1}-\trend_{t-1})+(\eyeR-\textnormal{diag}(\sGamma))\seasonal_{t-m},
    \end{equation}
\end{subequations}
}\normalsize
where diag($\cdot$) is an operator that creates a matrix with the elements of input vector on the main diagonal and $\eyeR$ is an $R$-by-$R$ identity matrix.

\subsubsection{Compute $\TXthat$}
Lastly, we can get $\TXthat$ by:
{\small
\begin{equation}
    \TXthat=\ktensort.
\end{equation}
}\normalsize
Using the reconstructed subtensor $\TXthat$, we can estimate the missing values on $\TYt$.


\subsection{Forecast} \label{sec:method:forecasting}
Let $t_{end}$ be the last timestamp of the stream.
Given any $t=t_{end}+h$, where $h$ is a positive integer, we can forecast a future temporal vector $\uNhat_{t|t_{end}}$ using the level, trend, and seasonal components by applying (\ref{eq:hw_forecast}) to each of its elements.
We also can forecast a future subtensor $\hat{\TY}_{t|t_{end}}$ using the most recent non-temporal factor matrices $\{\Un_{t_{end}}\}_{n=1}^{N-1}$ and the predicted temporal vector $\uNhat_{t|t_{end}}$ by:

{\small
\begin{equation}
    \hat{\TY}_{t|t_{end}}=\llbracket\{\Un_{t_{end}}\}_{n=1}^{N-1},\uNhat_{t|t_{end}}\rrbracket.
\end{equation}
}\normalsize


\begin{figure*}[!t]
    \vspace{-4mm}
    \centering
    \subfigure[Ground Truth]{\label{fig:als:ground}
        \includegraphics[width=0.18\linewidth]{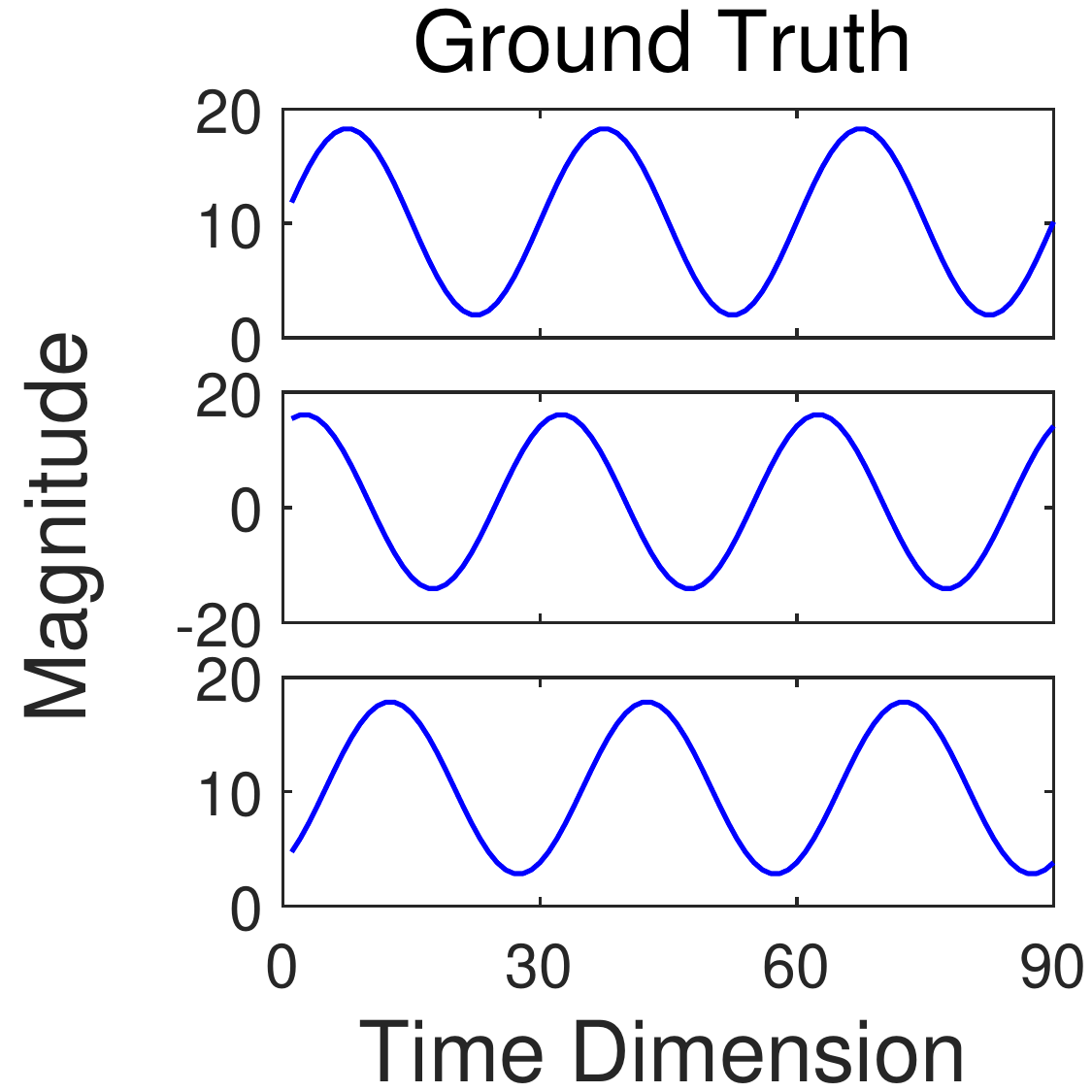}
    }
    \subfigure[Initialization with the vanilla ALS~\cite{zhou2008large}]{\label{fig:als:vanilla}
        \includegraphics[width=0.668\linewidth]{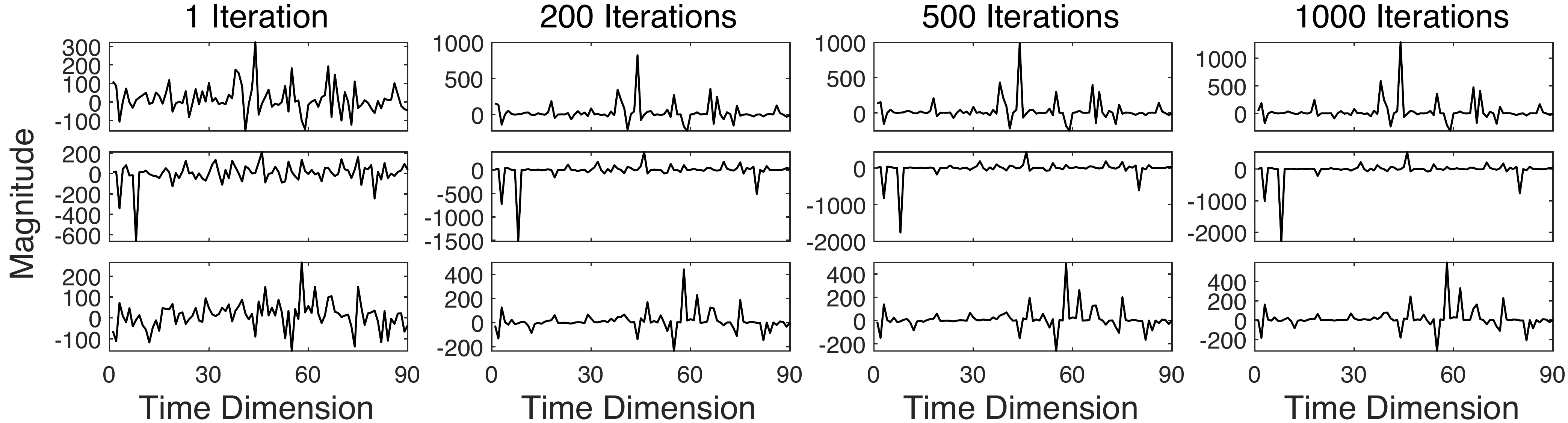}
    }\\
    \subfigure[Initialization with \methodALS (Proposed)]{\label{fig:als:sofia}
        \includegraphics[width=0.668\linewidth]{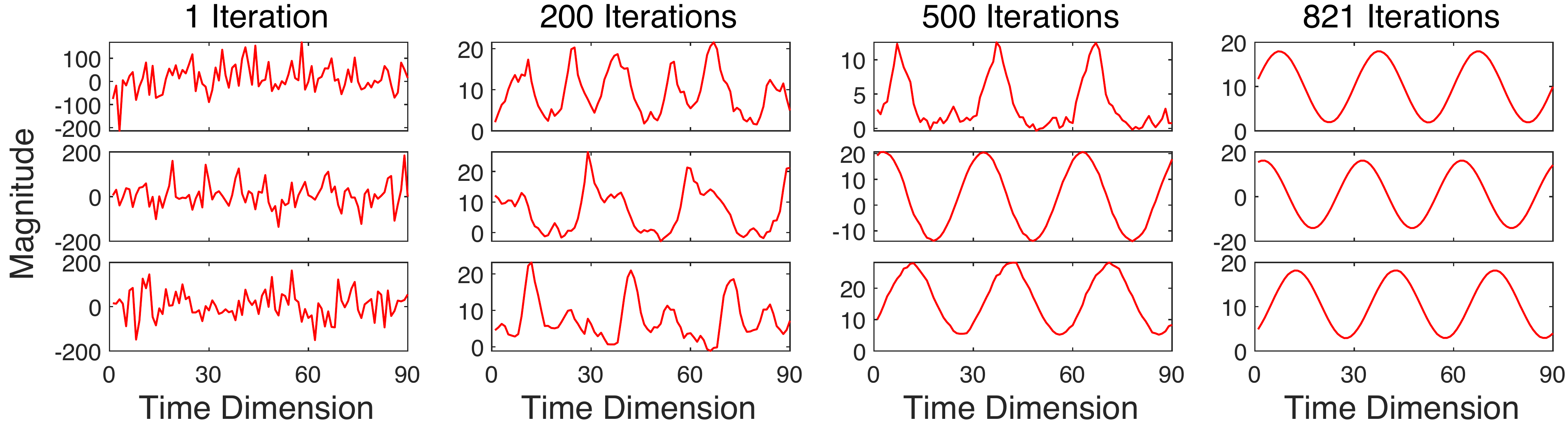}
    }
    \subfigure[Normalized Residual Error]{\label{fig:als:nre}
        \includegraphics[width=0.18\linewidth]{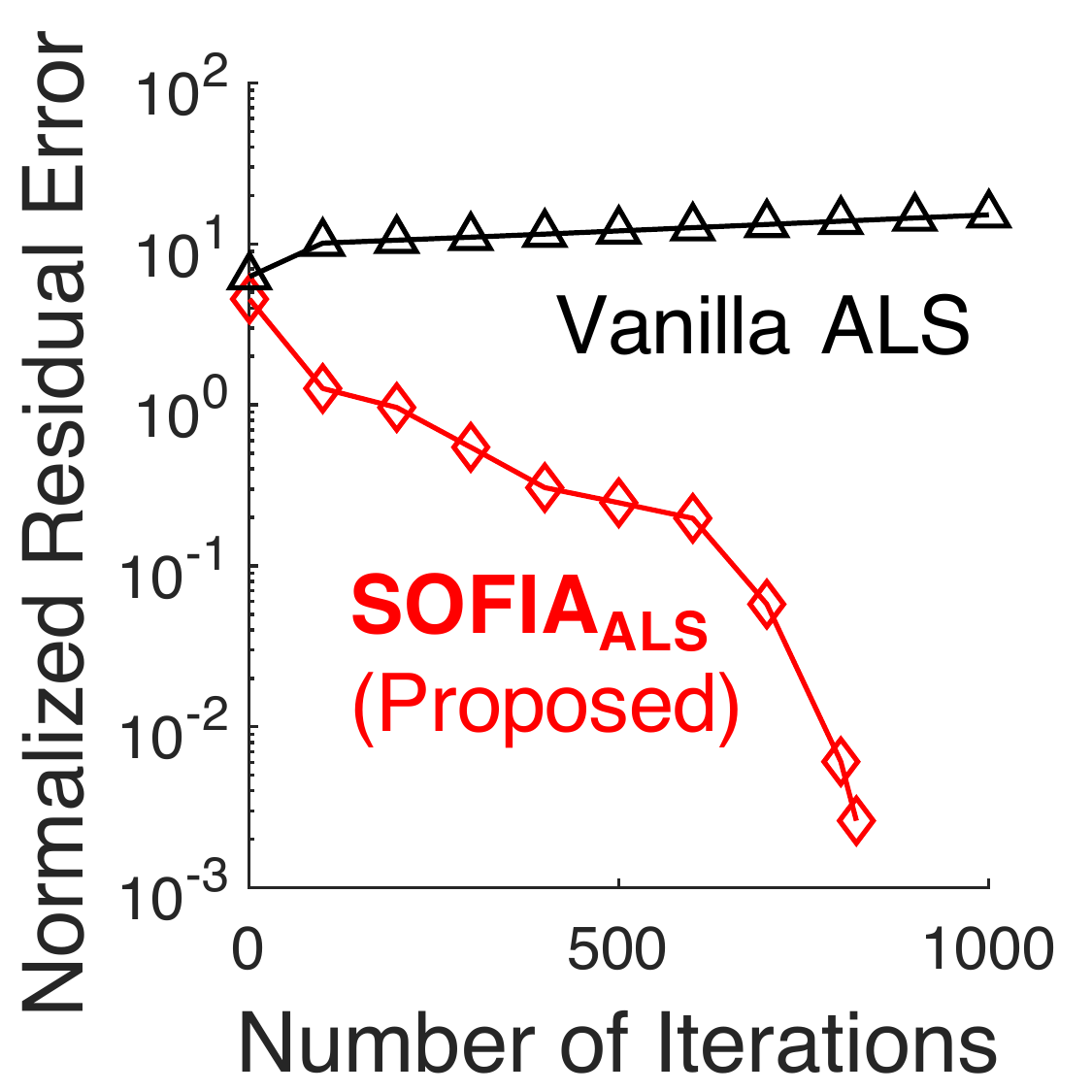}
    }
    \caption[]{\label{fig:exp:sofia_als}
        \textbf{\methodALS accurately captured temporal patterns} from an incomplete and noisy tensor in the initialization step.
        We used a synthetic tensor with the temporal factor matrix shown in (a) and performed experiments under the $(90,20,7)$ environment.
        (b) and (c) show the evolution of the temporal factor matrix as the outer iteration (lines~\ref{alg:init:start}-\ref{alg:init:end} in Algorithm~\ref{alg:init}) proceeded, when the vanilla ALS and \methodALS were used in the initialization step, respectively.
        (d) shows the normalized residual error between the ground truth and the temporal factor matrix obtained by each of the algorithms as the outer iteration proceeded.
        \textbf{Using temporal and seasonal smoothness was greatly helpful} for \methodALS to find the underlying temporal patterns.
    }
\end{figure*}

\subsection{Time Complexity} \label{sec:method:complexity}
The time complexity of \method is the sum of the complexity of each step (i.e., initialization, HW fitting, and dynamic updates).
Since the time cost of fitting the HW model  depends only on the length of the series (i.e., $O(\ti)$), it is not a dominant part of the overall complexity.
The time complexities of Algorithms~\ref{alg:init} and~\ref{alg:dynamic} are formulated in Lemmas~\ref{lemma:complexity:init} and \ref{lemma:complexity:dynamic}, respectively.
\begin{lemma}[Time Complexity of Initialization in \method] \label{lemma:complexity:init}
    The time complexity of Algorithm~\ref{alg:init} is $O\big(|\TOmega_{init}|NR(N+R)+R^{3}(\sum_{n=1}^{N-1}I_{n}+\ti)\big)$ per iteration.
\end{lemma}
\begin{proof}
    In Algorithm~\ref{alg:als},
    updating each row $\uinn$ of the factor matrices by (\ref{eq:als:nontemporal}) and (\ref{eq:als:temporal}) takes $O(|\Omegainn|R(N+R)+R^{3})$ time.
    It is composed of $O(|\Omegainn|RN)$ time to compute $\bigHadamardBinn$ for all the entries in $\Omegainn$, $O(|\Omegainn|R^{2})$ time to compute $\Binn$, $O(|\Omegainn|R)$ time to compute $\cinn$, and $O(R^{3})$ time to invert $\Binn$.
    We update all factor matrices one by one and row by row, and thus the overall complexity of Algorithm~\ref{alg:als}, which is the dominant part in each iteration of Algorithm~\ref{alg:init},
    is $O\big(|\TOmega_{init}|NR(N+R)+R^{3}(\sum_{n=1}^{N-1}I_{n}+\ti)\big)$.
\end{proof}

\begin{lemma}[Time Complexity of Dynamic Updates in \method] \label{lemma:complexity:dynamic}
    The time complexity of the iteration of
    Algorithm~\ref{alg:dynamic} at time $t$ is $O\big(|\TOmegat|NR\big)$.
\end{lemma}
\begin{proof}
    For the iteration at time $t$, predicting $\TYthat$ and updating $\bSigma_{t}$ and $\TOt$ only for the observed entries at time $t$ takes $O(|\TOmegat|NR)$ time.
    Updating $\ntUnt$ takes $O(|\TOmegat|NR)$ time if we compute only the entries of  $\bigKhatrirao_{l=1,l\neq n}^{N-1}\Umat^{(l)}_{t-1}$ multiplied with the non-zeros in $\mathbf{R}_{(n)}$, whose number is $|\TOmegat|$, when computing  $\mathbf{R}_{(n)}\bigKhatrirao_{l=1,l\neq n}^{N-1}\Umat^{(l)}_{t-1}$.
    Updating $\uNt$ also takes $O\big(|\TOmegat|NR\big)$ time if we compute only the entries of  $\bigKhatrirao_{n=1}^{N-1}\Umat^{(n)}_{t-1}$ multiplied with the non-zeros in $\textnormal{vec}(\TRt)$, whose number is $|\TOmegat|$, when computing  $(\bigKhatrirao_{n=1}^{N-1}\Umat^{(n)}_{t-1})^{\top}\cdot \textnormal{vec}(\TRt)$.
    Thus, the overall time complexity is $O\big(|\TOmegat|NR\big)$.
\end{proof}
Since the initialization step is executed only once at the start, after initialization, \method takes time proportional to the number of observed entries in the received subtensor (i.e., $|\TOmegat|$), as shown in Lemma~\ref{lemma:complexity:dynamic}.

\section{Experiments}
\label{sec:exp}

In this section, we review our experiments to answer the following questions:
\begin{itemize}
	\setlength{\itemindent}{-.1in}
	\item \textbf{Q1. Initialization Accuracy}: How accurately does \methodALS capture seasonal patterns in time-series? 
	\item \textbf{Q2. Imputation Accuracy}: How accurately does \method estimate missing entries compared to its best competitors?
	\item \textbf{Q3. Speed}: How fast is \method?
	\item \textbf{Q4. Forecasting Accuracy}: How precisely does \method predict future entries?
	\item \textbf{Q5. Scalability}: How does \method scale with regard to the size of the input tensor?
\end{itemize}

\begin{figure*}[!t]
	\centering
	\vspace{-4mm}
	\includegraphics[width=0.60\linewidth]{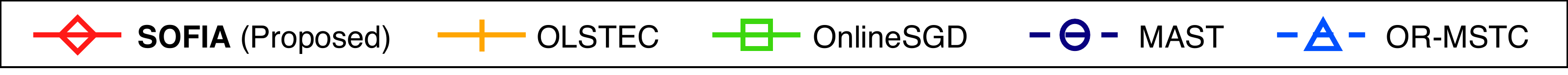} \\ \vspace{1mm}
	\includegraphics[width=0.23\linewidth]{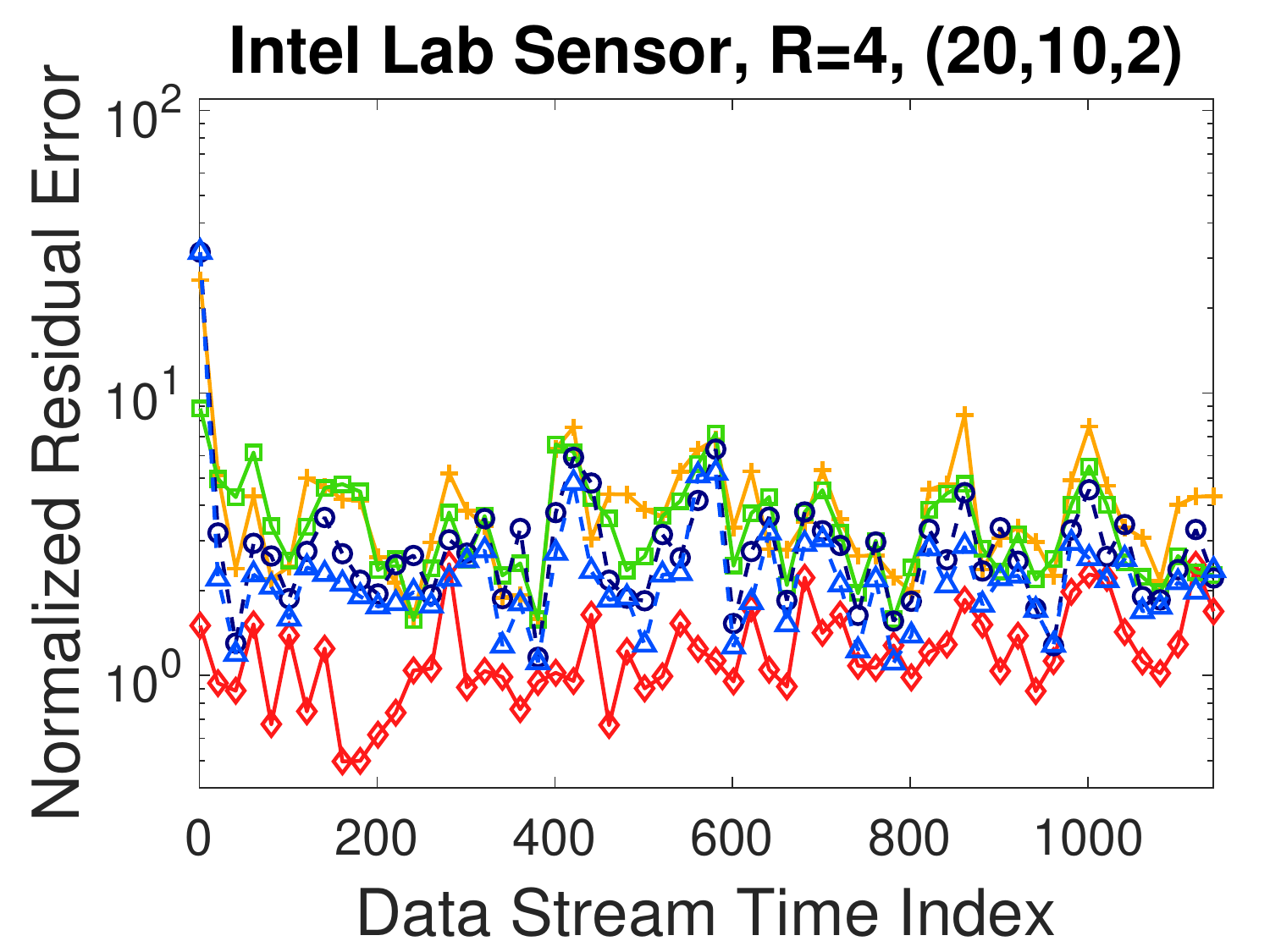}
	\includegraphics[width=0.23\linewidth]{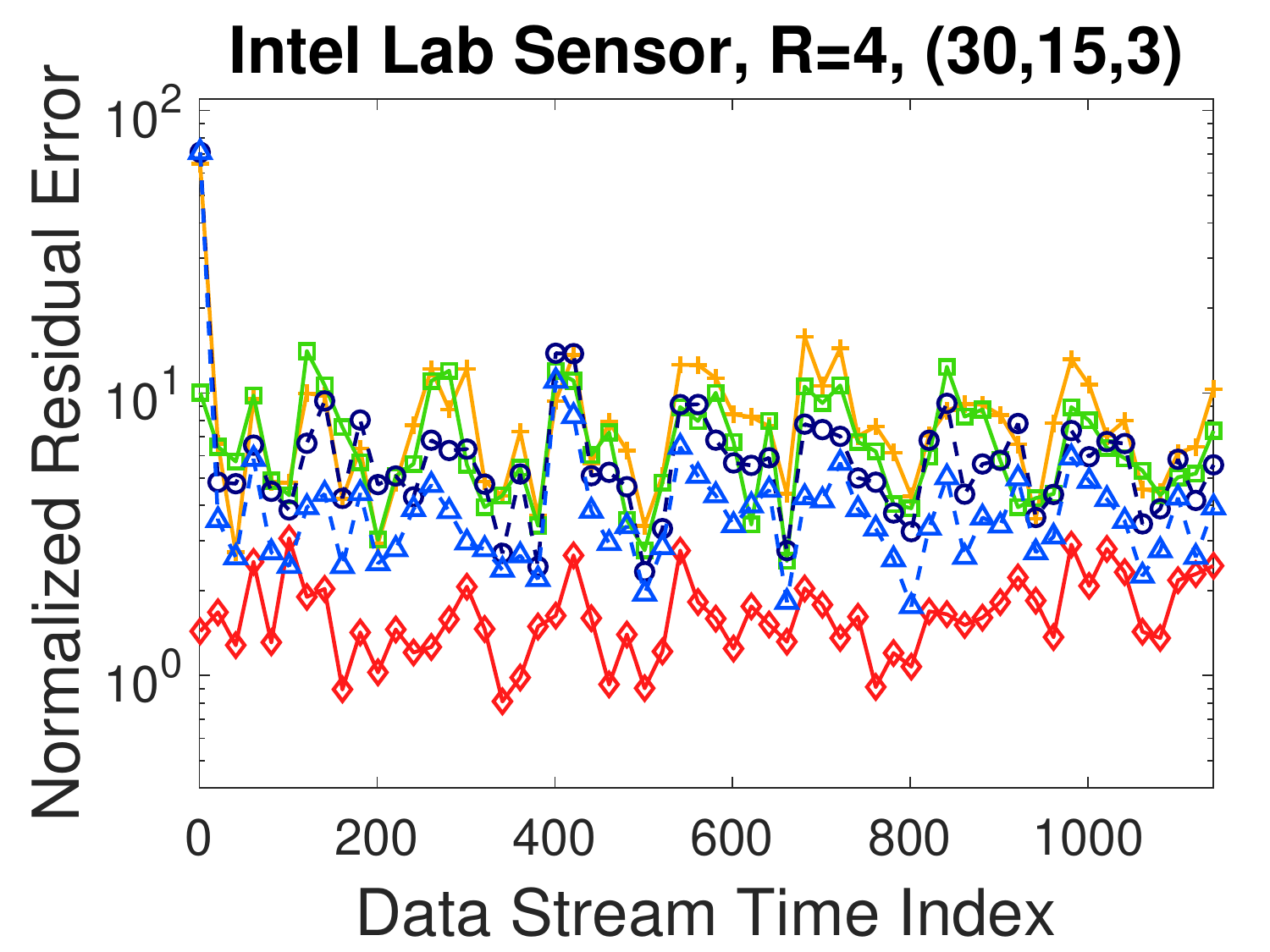}
	\includegraphics[width=0.23\linewidth]{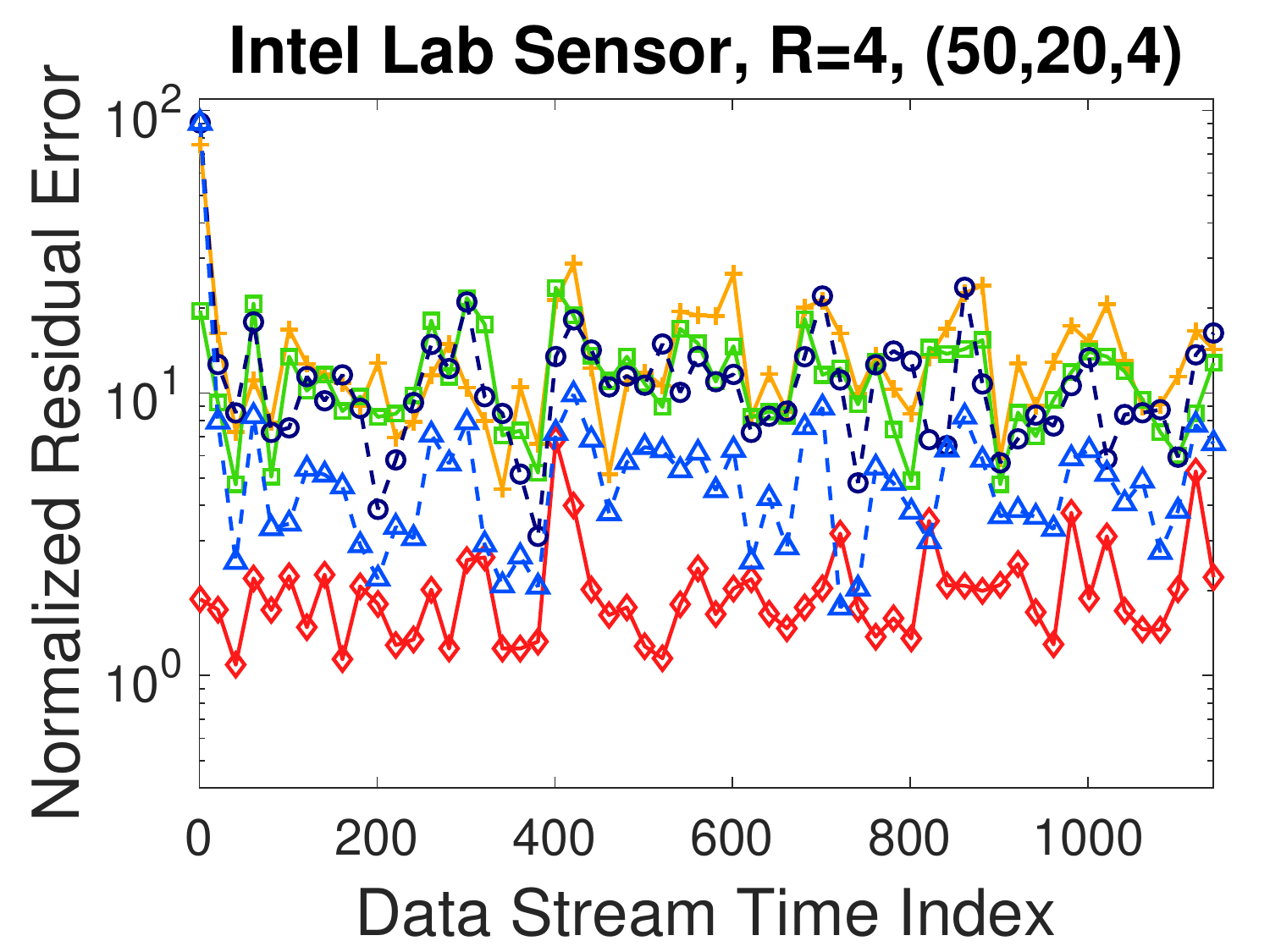}
	\includegraphics[width=0.23\linewidth]{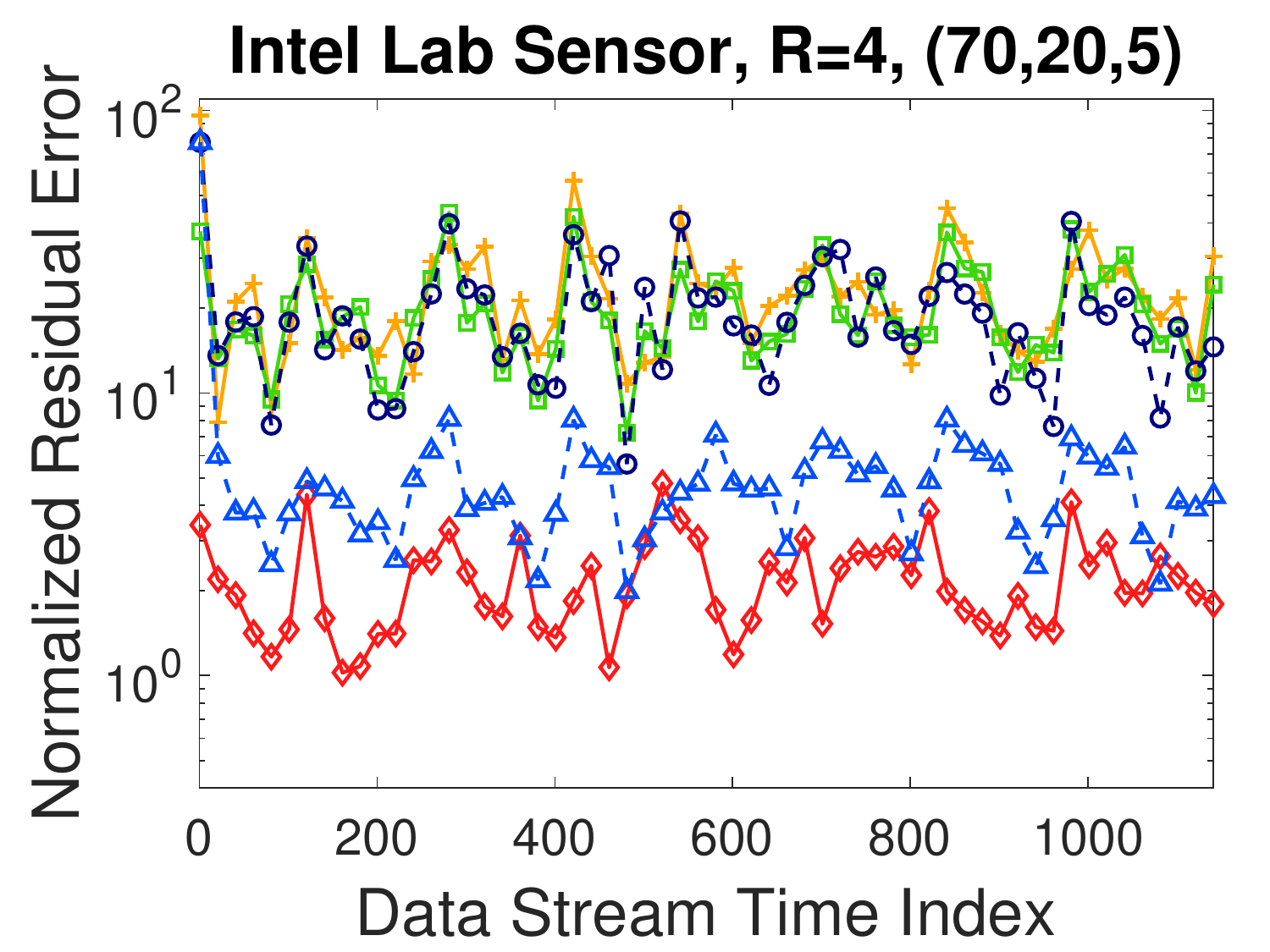} \\ \vspace{1mm}
	\includegraphics[width=0.23\linewidth]{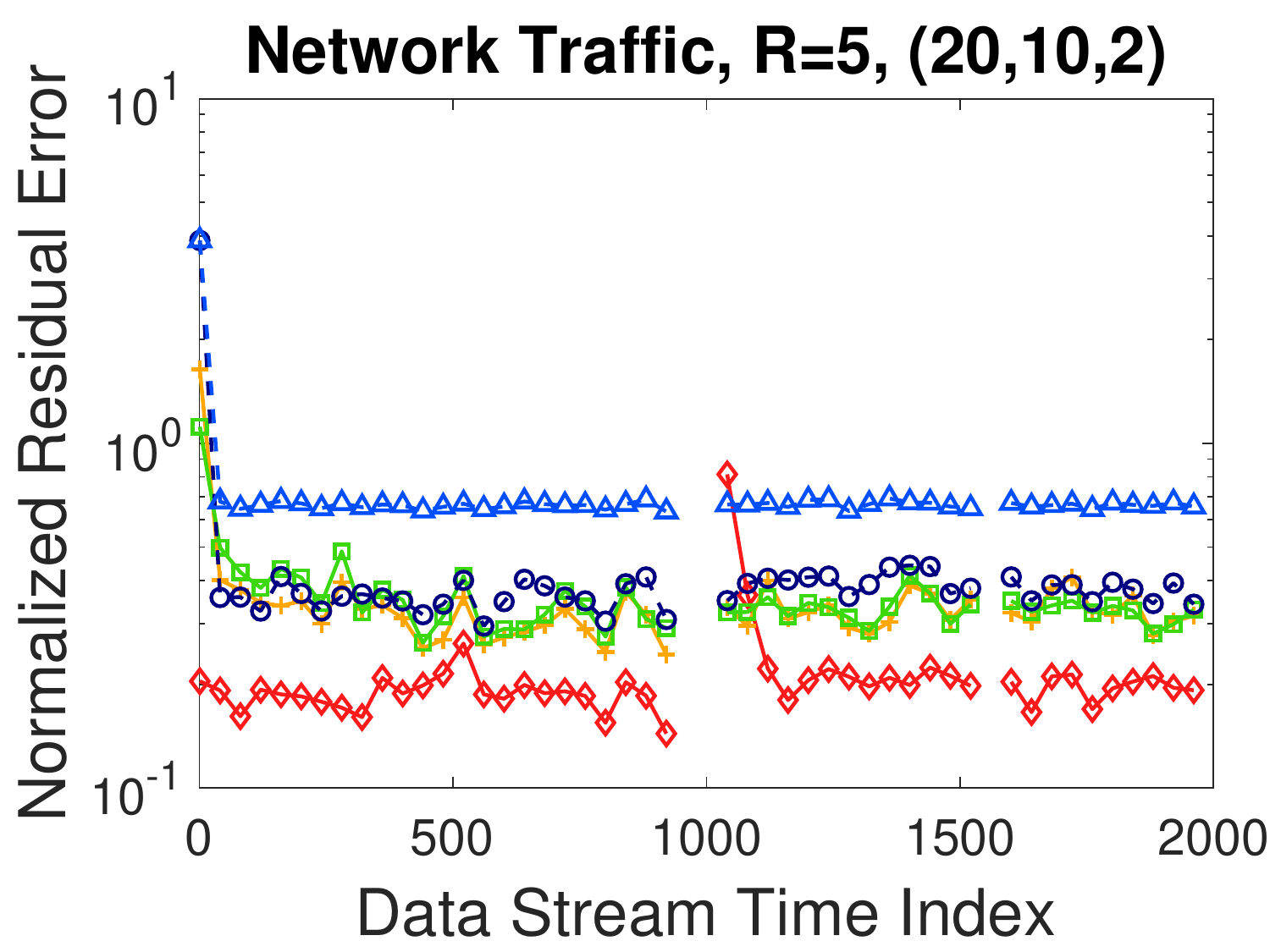}
	\includegraphics[width=0.23\linewidth]{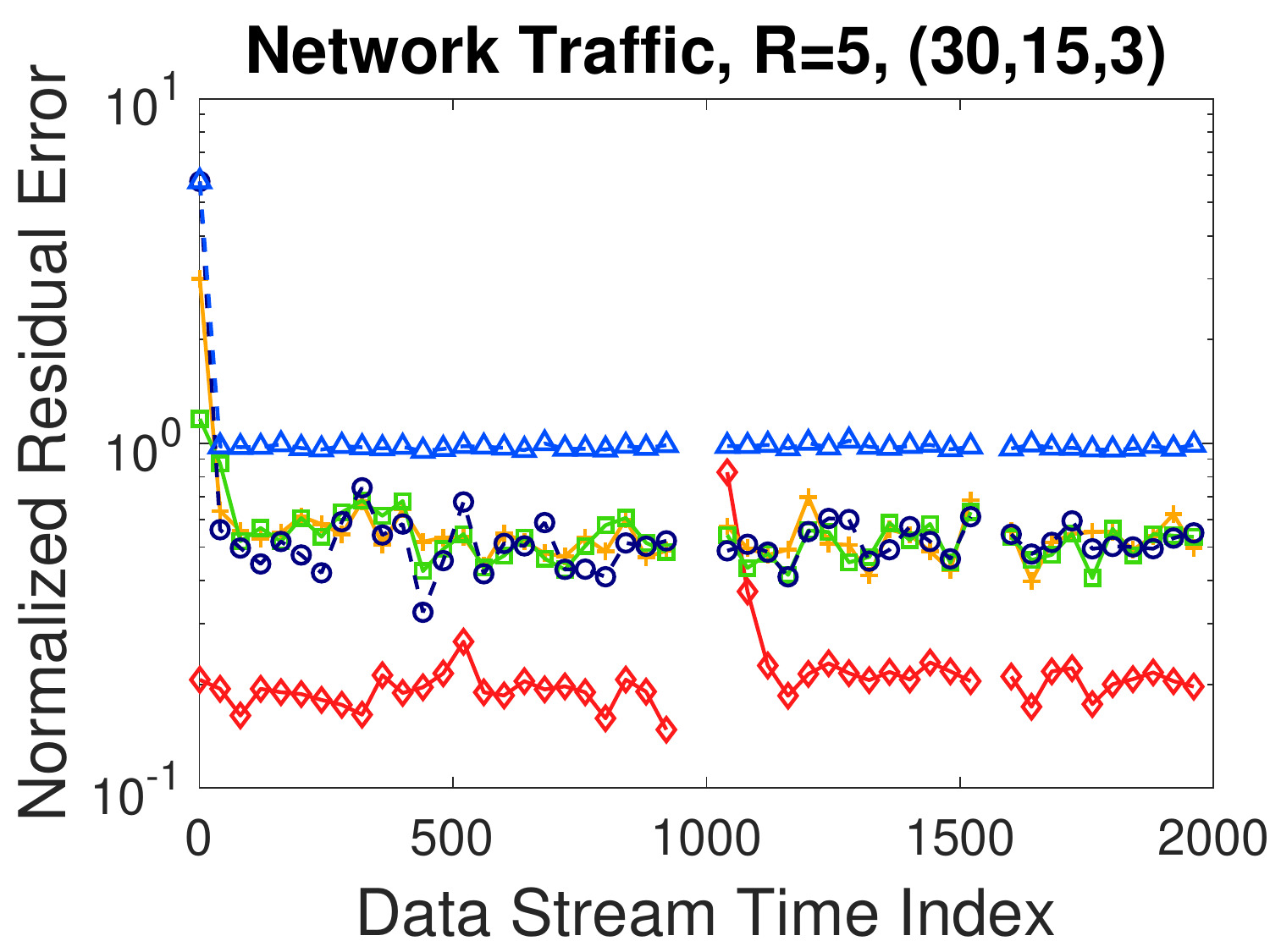}
	\includegraphics[width=0.23\linewidth]{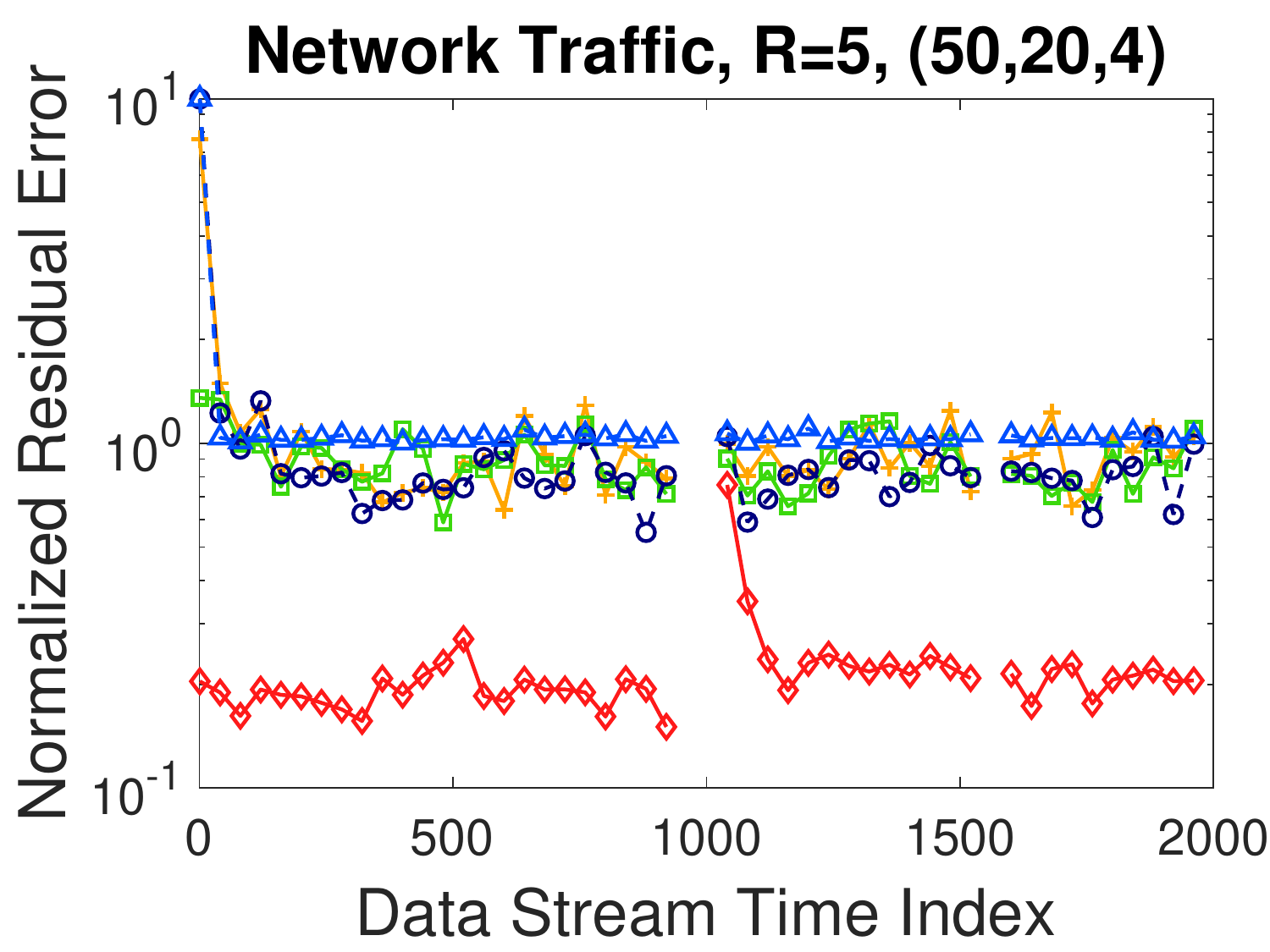}
	\includegraphics[width=0.23\linewidth]{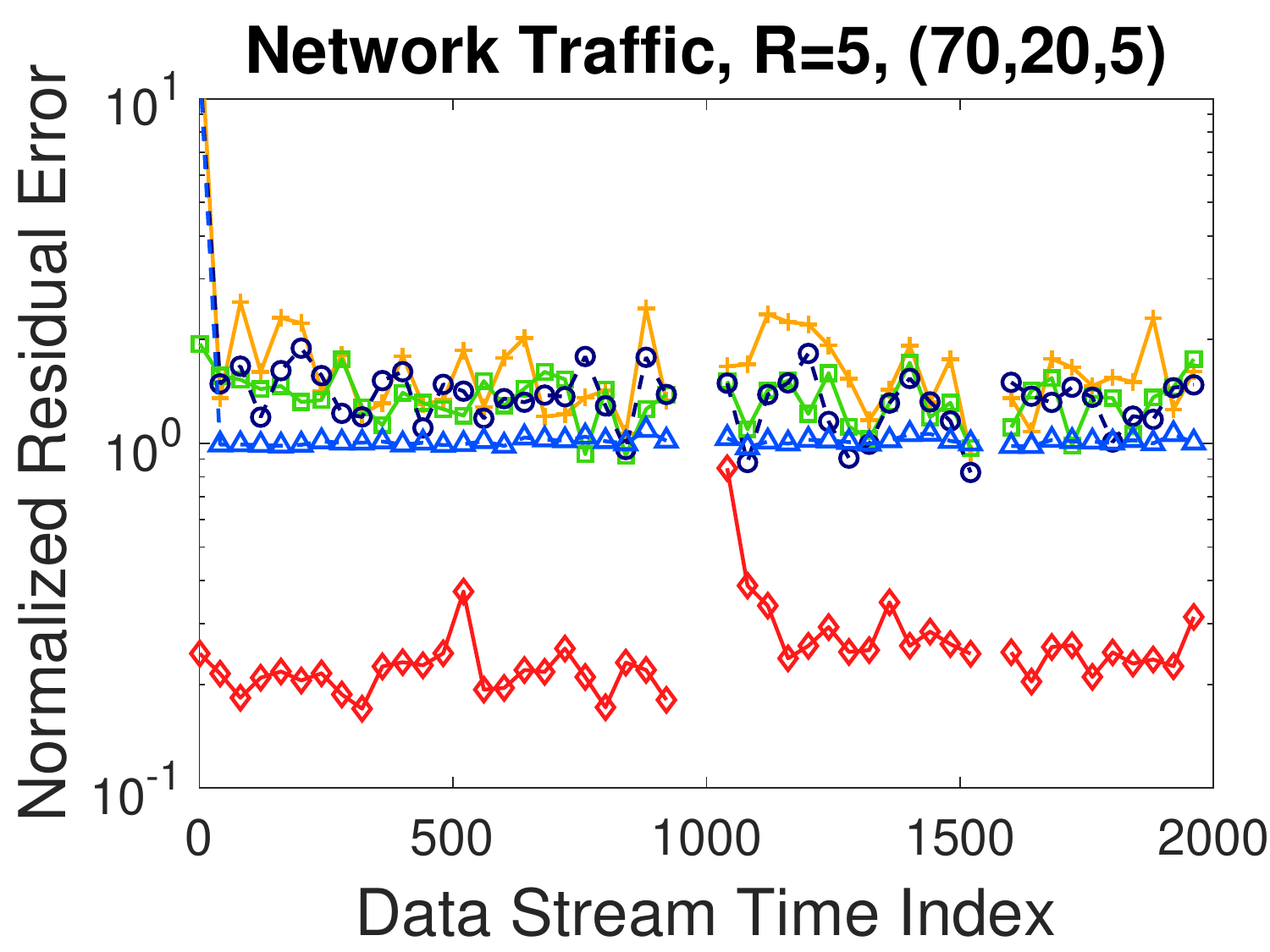} \\ \vspace{1mm}
	\includegraphics[width=0.23\linewidth]{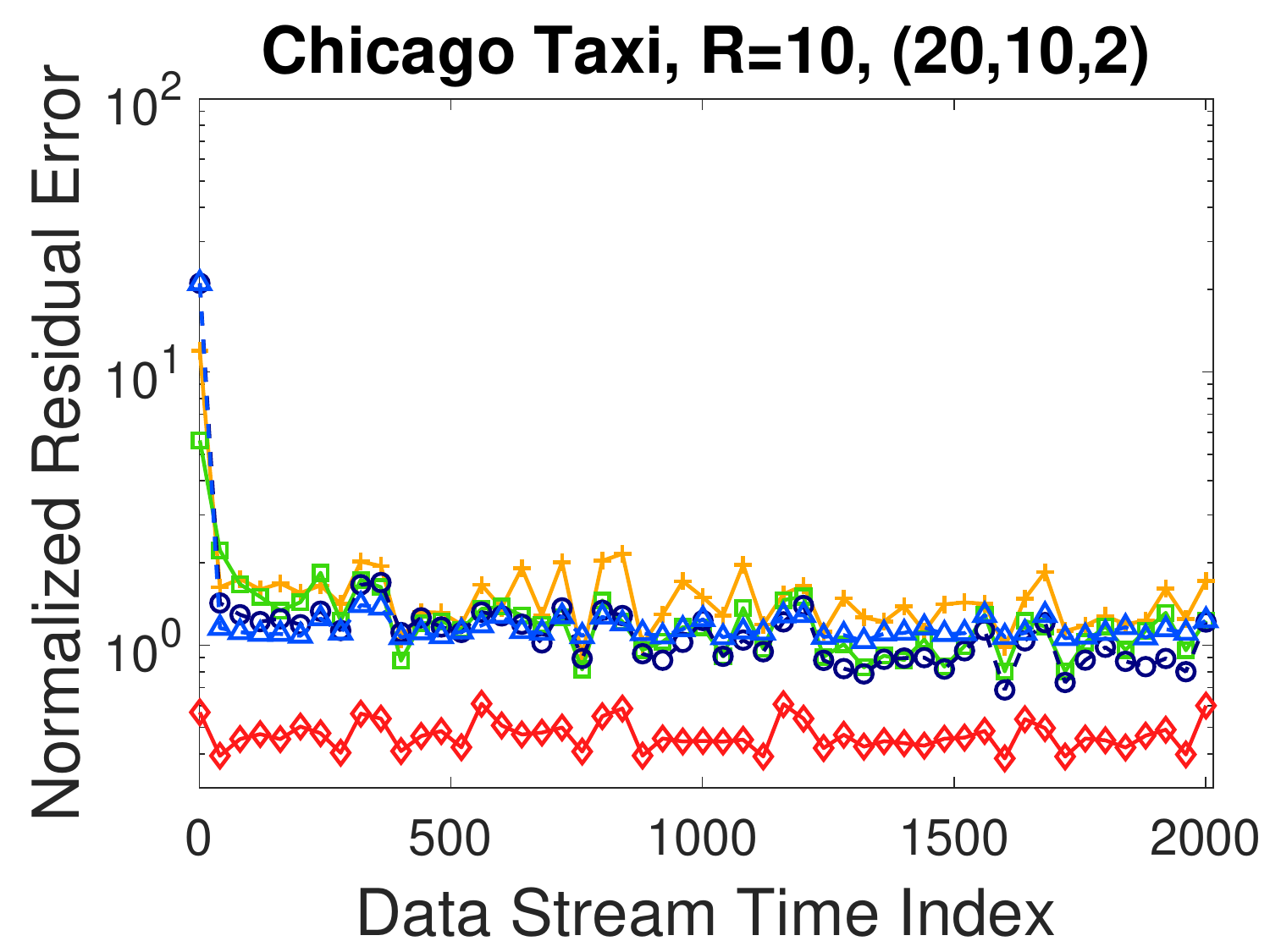}
	\includegraphics[width=0.23\linewidth]{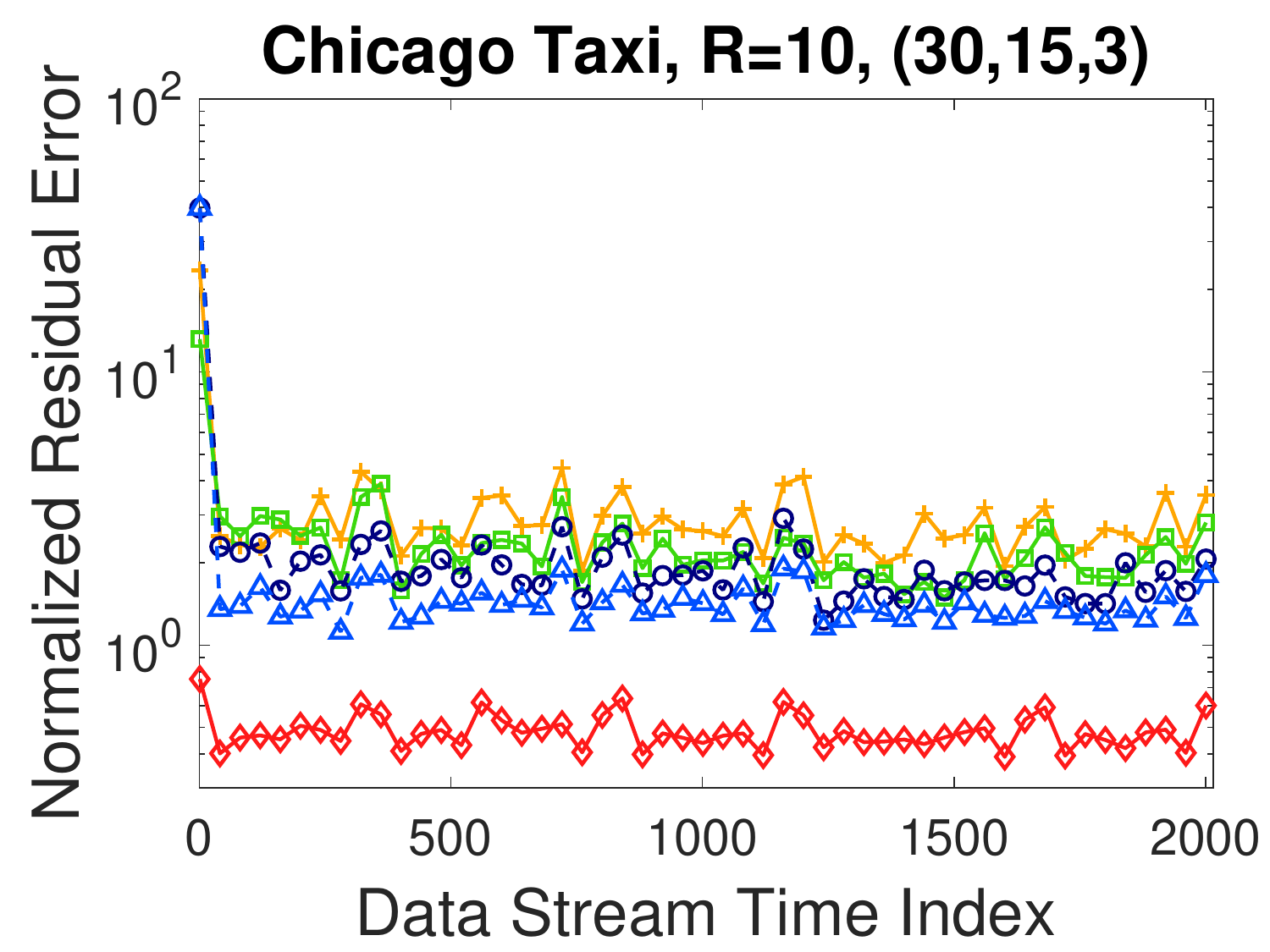}
	\includegraphics[width=0.23\linewidth]{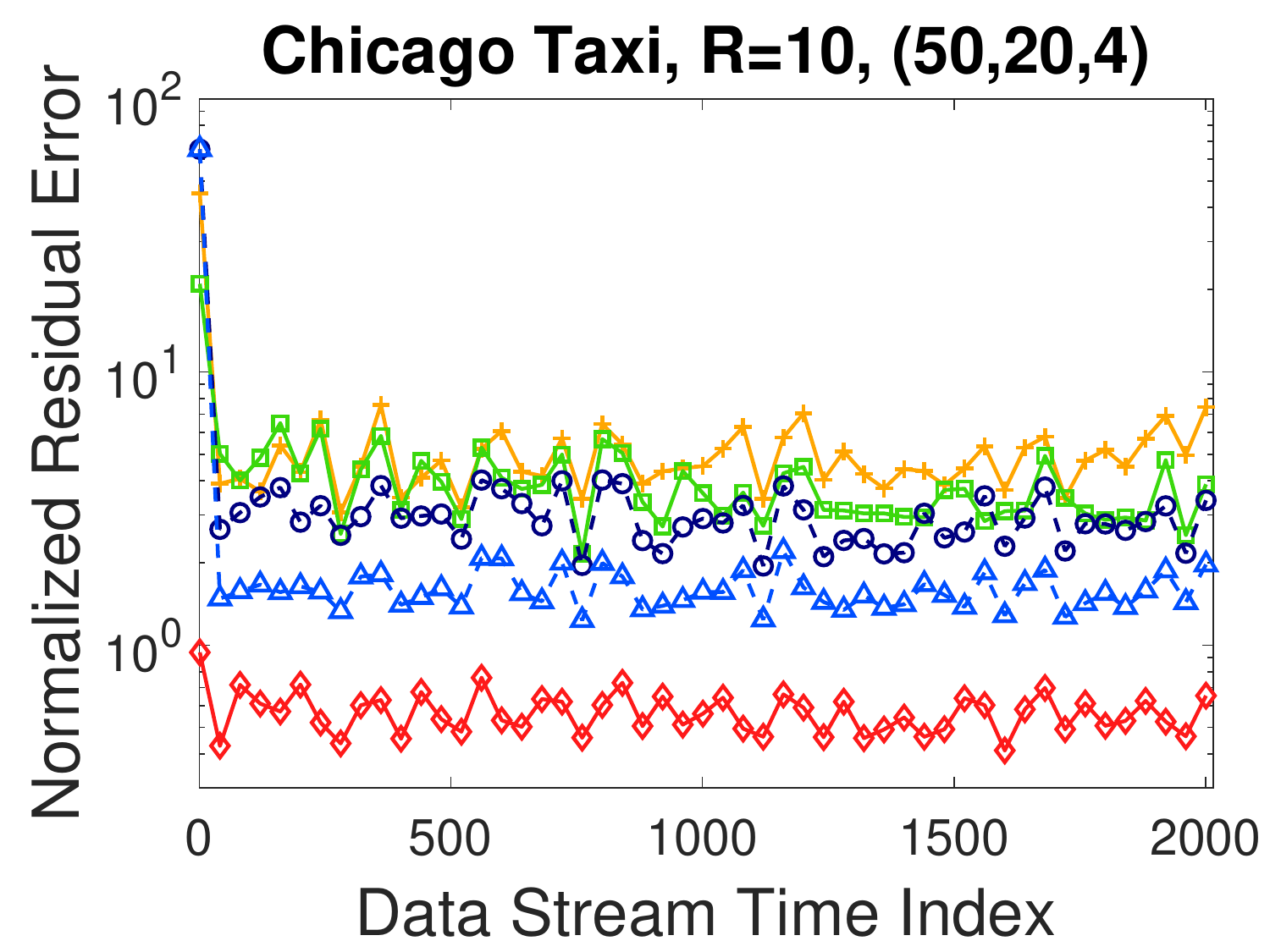}
	\includegraphics[width=0.23\linewidth]{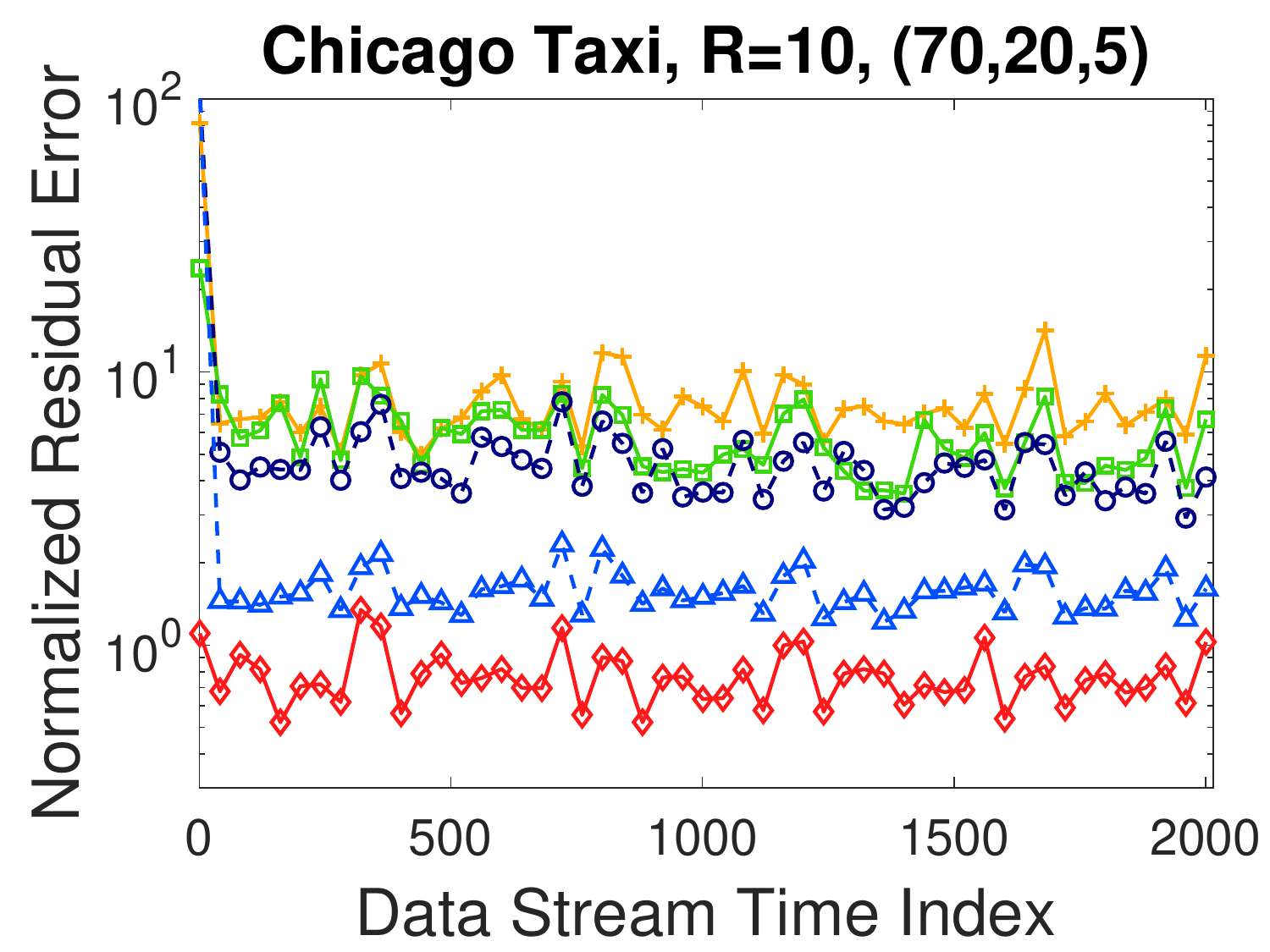} \\ \vspace{1mm}
	\includegraphics[width=0.23\linewidth]{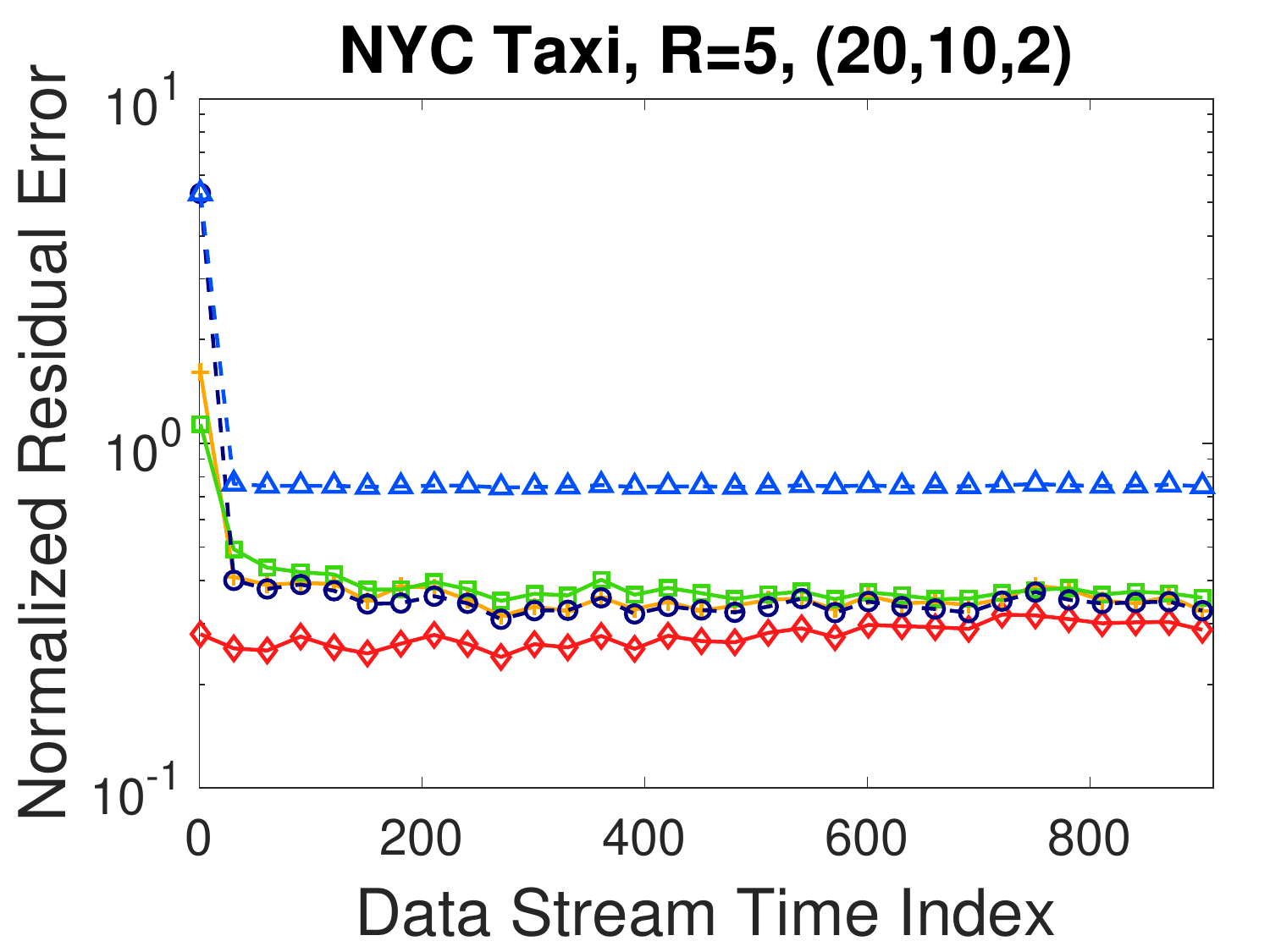}
	\includegraphics[width=0.23\linewidth]{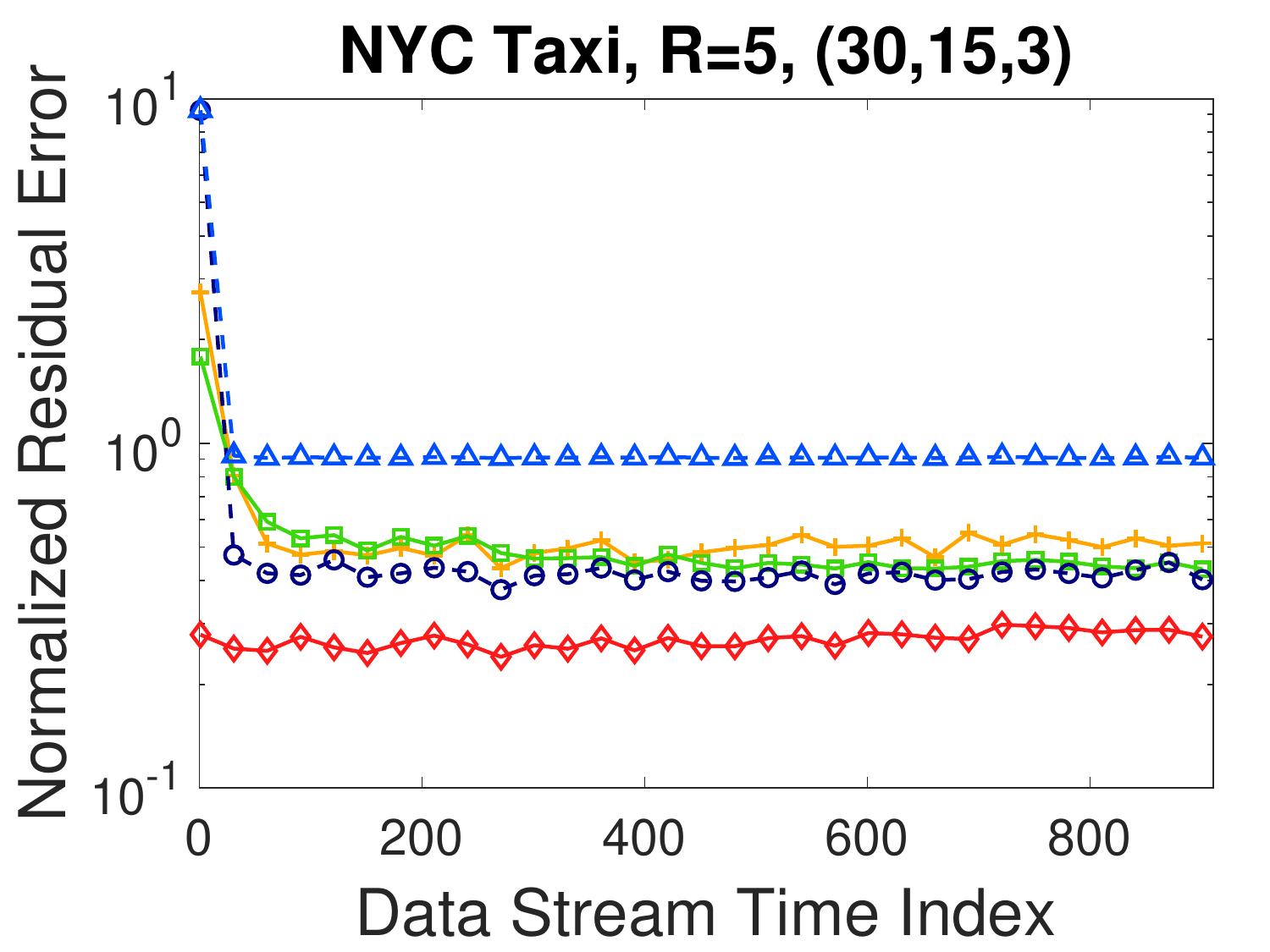}
	\includegraphics[width=0.23\linewidth]{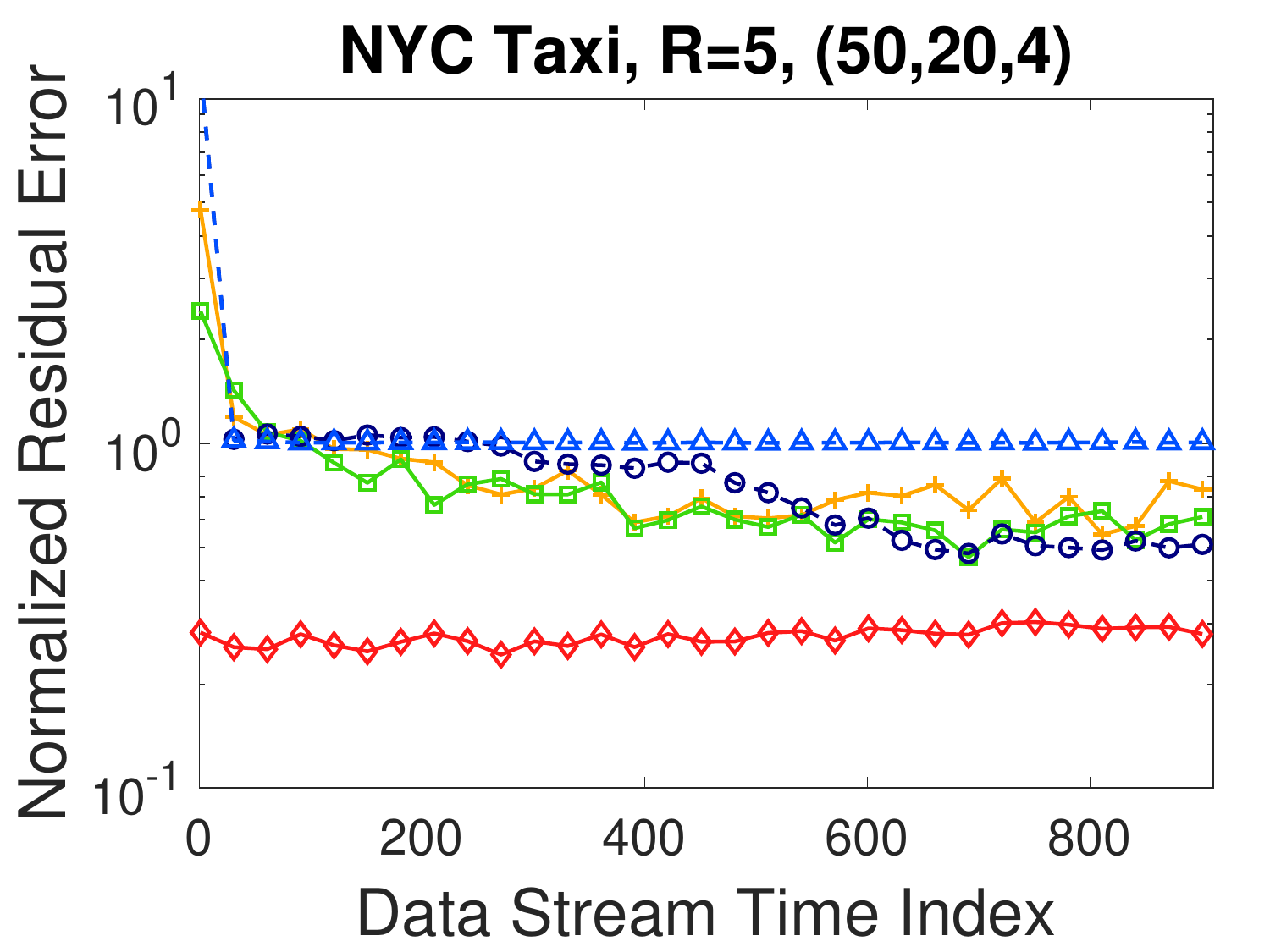}
	\includegraphics[width=0.23\linewidth]{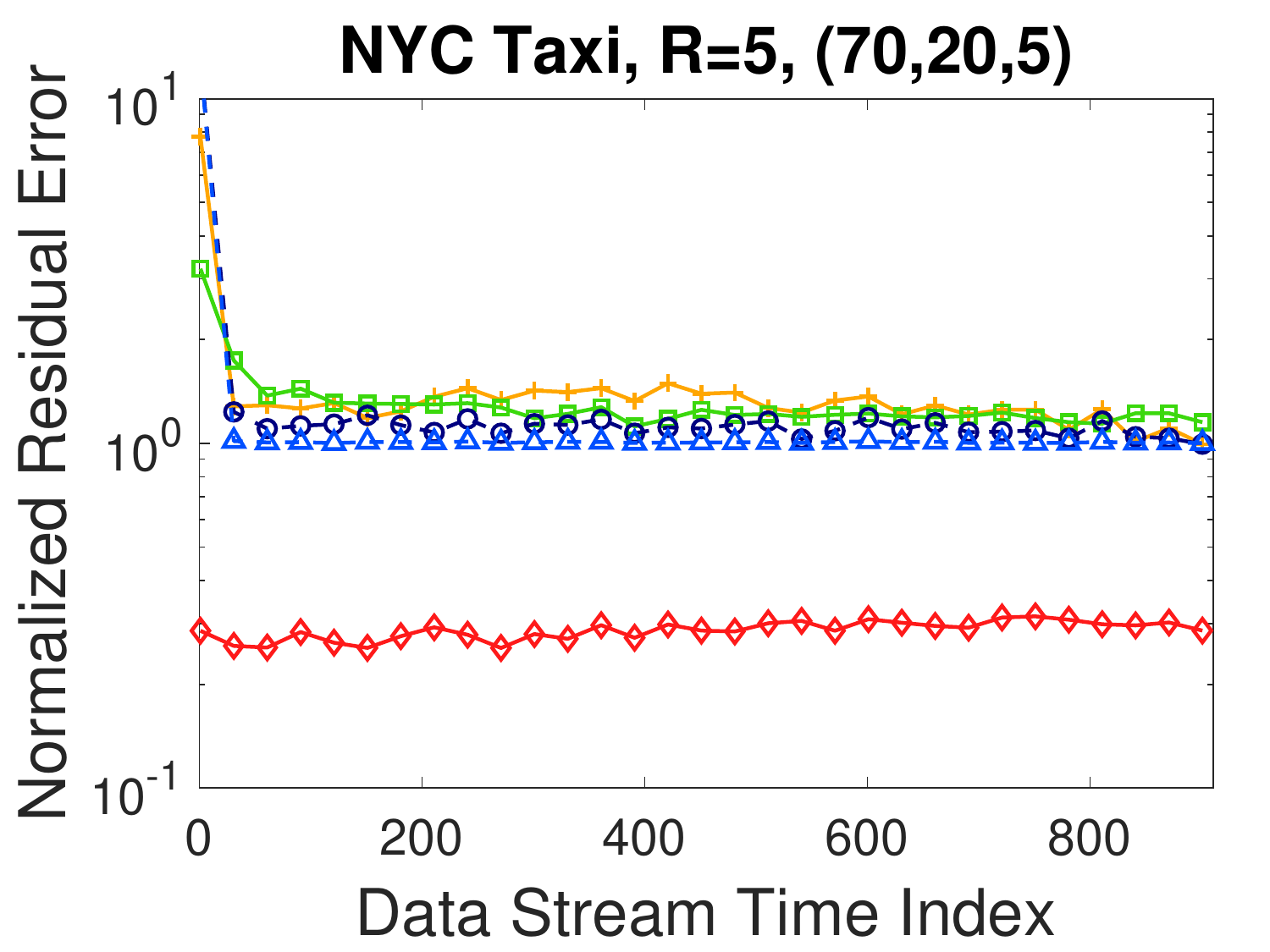}
	\\
	\vspace{-2mm}
	\caption[]{\label{fig:exp:imputation:nre}
		The normalized residual error under 4 experimental settings from the mildest (leftmost) to the harshest (rightmost).
		\textbf{\method was the most accurate} in all the tensor streams under all the experimental settings.
	}
\end{figure*}

\subsection{Experiment Specifications}\label{sec:exp:spec}
\noindent\underline{\textbf{\smash{Machine and Implementation}}}:
We implemented our algorithm and all competitors in Matlab; all the experiments were conducted on a PC with a 3.70GHz Intel i5-9600K CPU and 64GB memory.

\begin{table}[t!]
	\centering
	\caption{Summary of datasets.}
	\label{tab:datasets}
	\scalebox{0.95}{
		\begin{tabular}{llll}
			\toprule
			{\bf Dataset}          & {\bf Dimension}             & {\bf Period} & {\bf Granularity in Time} \\
			\midrule
			{\bf Intel Lab Sensor} & $54\times4\times1152^{*}$   & $144$        & every $10$ minutes        \\
			{\bf Network Traffic}  & $23\times23\times2000^{*}$  & $168$        & hourly                    \\
			{\bf Chicago Taxi}     & $77\times77\times2016^{*}$  & $168$        & hourly                    \\
			{\bf NYC Taxi}         & $265\times265\times904^{*}$ & $7$          & daily                     \\
			\bottomrule
			\multicolumn{4}{l}{The time mode is marked with an asterisk (*).} 
		\end{tabular}
	}
\end{table}

\noindent\underline{\textbf{\smash{Datasets}}}:
We conducted experiments on 4 real-world datasets that are summarized in Table~\ref{tab:datasets}.
\begin{itemize}[leftmargin=*]
	\item \textbf{Intel Lab Sensor}~\cite{madden2003intel}: The 4 indoor environmental sensor data collected from 54 positions in the Intel Berkeley Research Lab. We made a tensor with \textit{(position, sensor, time)} triples with a 10-minute interval and standardized the observations from each sensor.
	\item \textbf{Network Traffic}~\cite{uhlig2006providing}: The network traffic records between 23 routers. We made a tensor with \textit{(source, destination, time)} triples with an 1-hour interval and used $\log_{2}(x+1)$ for each entry $x$ to adjust for scaling bias in the amount of traffic.
	\item \textbf{Chicago Taxi}\footnote{\url{https://data.cityofchicago.org/Transportation/Taxi-Trips/wrvz-psew}}: The taxi trip data in Chicago. We created a tensor with \textit{(source, destination, pick-up time)} triples with an 1-hour interval and used $\log_{2}(x+1)$ for each entry $x$.
	\item \textbf{NYC Taxi}\footnote{\url{https://www1.nyc.gov/site/tlc/about/tlc-trip-record-data.page}}: The yellow taxi trip records in New York City. We created a tensor with \textit{(source, destination, pick-up date)} triples and used $\log_{2}(x+1)$ for each entry $x$.
\end{itemize}


\noindent\underline{\textbf{\smash{Competitors}}}:
To evaluate our method, we compare our method with the following seven competitors:
(1) \textbf{\onlineSGD}~\cite{mardani2015subspace}, a streaming CP factorization method optimized by SGD,
(2) \textbf{\olstec}~\cite{kasai2016online}, a streaming CP factorization method optimized by recursive least square (RLS),
(3) \textbf{\mast}~\cite{song2017multi}, a multi-aspect streaming tensor completion method,
(4) \textbf{\brst}~\cite{zhang2018variational}, an outlier-robust streaming tensor factorization approach based on bayesian inference,
(5) \textbf{\ormstc}~\cite{najafi2019outlier}, a robust multi-aspect streaming tensor completion algorithm,
(6) \textbf{\smf}~\cite{hooi2019smf}, a streaming matrix factorization method that is able to forecast future values using seasonal patterns, and
(7) \textbf{\cphw}~\cite{dunlavy2011temporal}, which can predict future values based on a static tensor factorization and the HW method.
The first 5 approaches are used to compare the imputation performance and the last 2 algorithms are used to compare the forecasting performance.

\noindent\underline{\textbf{Evaluation Metrics}}:
We use the following four metrics to measure the accuracy and efficiency of each algorithm:
\begin{itemize}
	\setlength{\itemindent}{-.15in}
	      \footnotesize
	\item {\small\textbf{Normalized Residual Error (NRE)}}: $\dfrac{\norm{\TXthat-\TXt}_{F}}{\norm{\TXt}_{F}}$
	\item {\small\textbf{Running Average Error (ARE)}}: $\dfrac{1}{T}\sum_{t=1}^{T}\dfrac{\norm{\TXthat-\TXt}_{F}}{\norm{\TXt}_{F}}$
	\item {\small\textbf{Average Forecasting Error (AFE)}}: $\dfrac{1}{t_{f}}\sum_{h=1}^{t_{f}}\dfrac{\norm{\TXhat_{t+h|t}-\TX_{t+h}}_{F}}{\norm{\TX_{t+h}}_{F}}$
	\item {\small\textbf{Average Running Time (ART)}}: $\dfrac{1}{T-\ti-1}\sum_{t=\ti+1}^{T}RT(t)$
	      \normalsize
\end{itemize}
where $T$ is the length of the entire stream, $t_{f}$ is the forecasting time steps, and $RT(t)$ is the running time to process a subtensor at time step $t$.
Since the initialization is executed only once, ART is calculated except for the time spent on initialization. Algorithms without initialization are set to $\ti=0$.
In the following experiments, we computed each evaluation metric $5$ times for each algorithm, and the mean is reported.

\begin{figure*}[!t]
	\centering
	\vspace{-4mm}
	\includegraphics[width=0.60\linewidth]{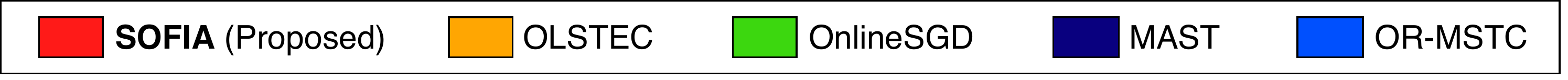}
	\vspace{-2mm}
\end{figure*}

\begin{figure}[!t]
	\centering
	\subfigure[Intel Lab Sensor]{\label{fig:art:traffic}
		\includegraphics[width=0.9\linewidth]{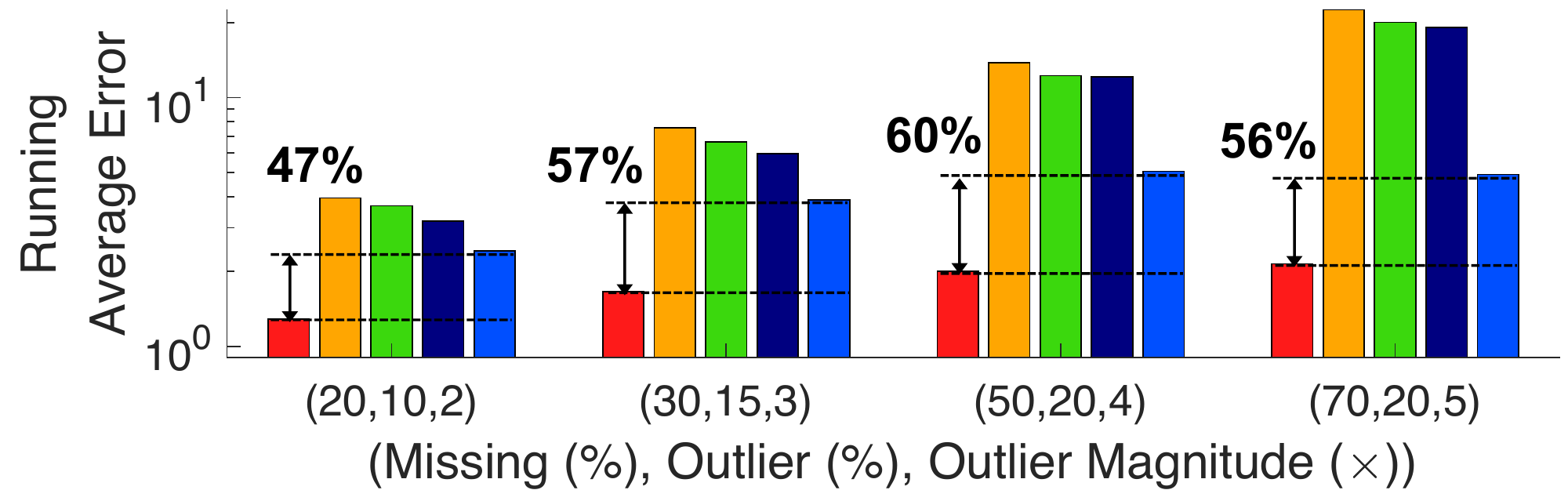}
	}\\
	\vspace{-3mm}
	\subfigure[Network Traffic]{\label{fig:art:traffic}
		\includegraphics[width=0.9\linewidth]{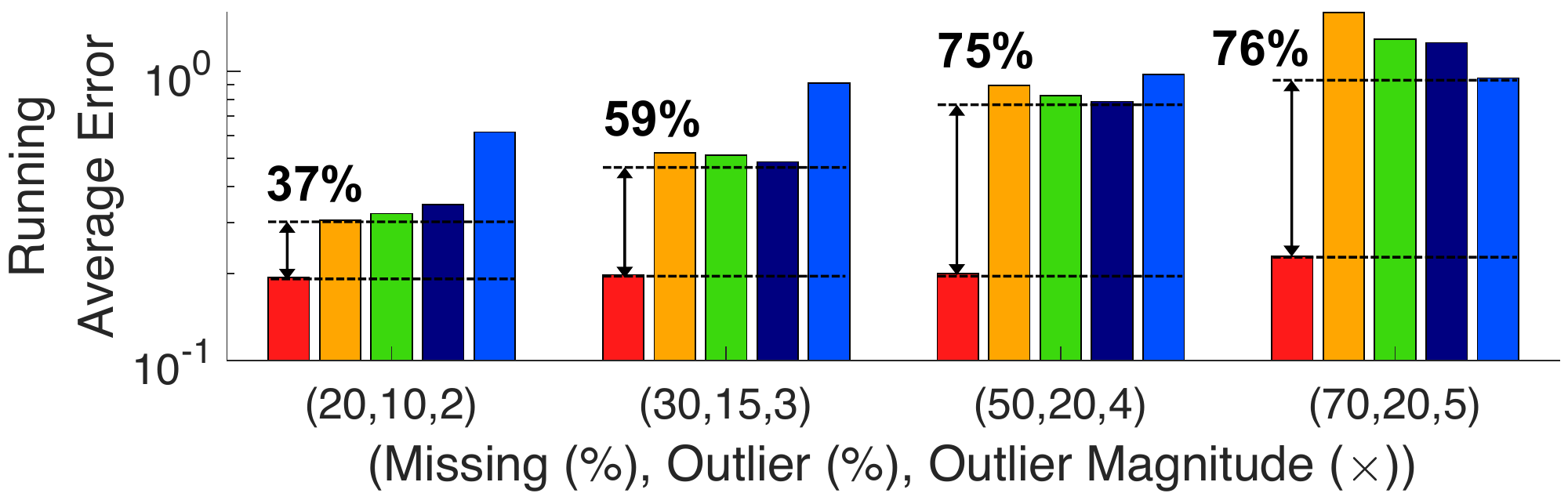}
	}\\
	\vspace{-3mm}
	\subfigure[Chicago Taxi]{\label{fig:art:traffic}
		\includegraphics[width=0.9\linewidth]{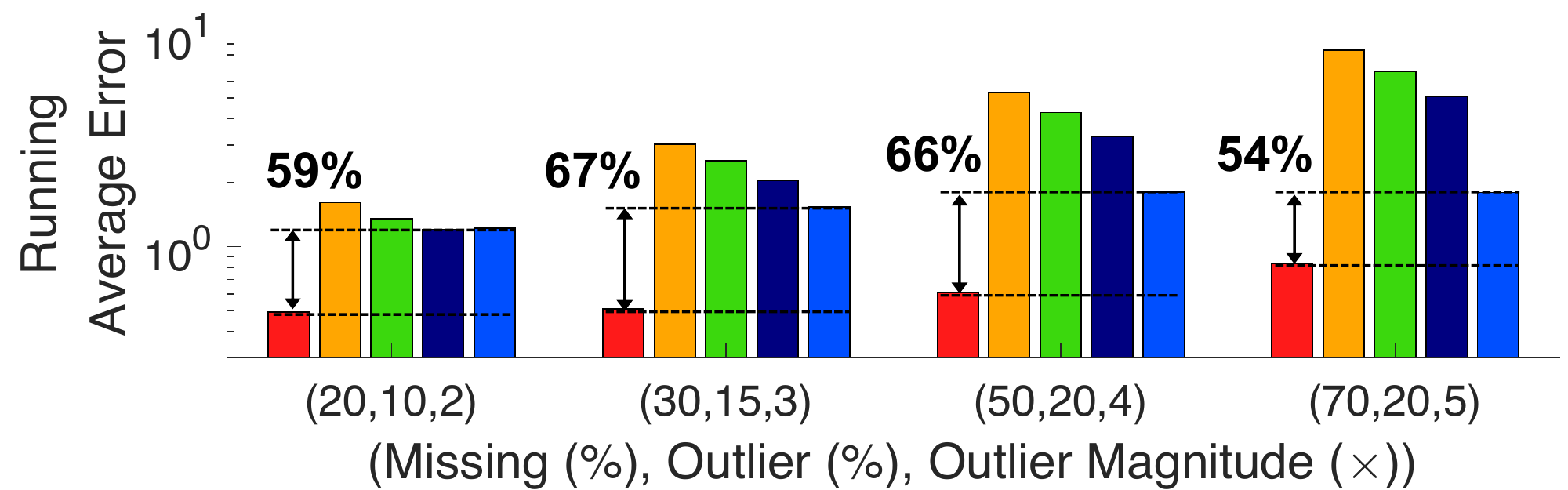}
	}\\
	\vspace{-3mm}
	\subfigure[NYC Taxi]{\label{fig:art:traffic}
		\includegraphics[width=0.9\linewidth]{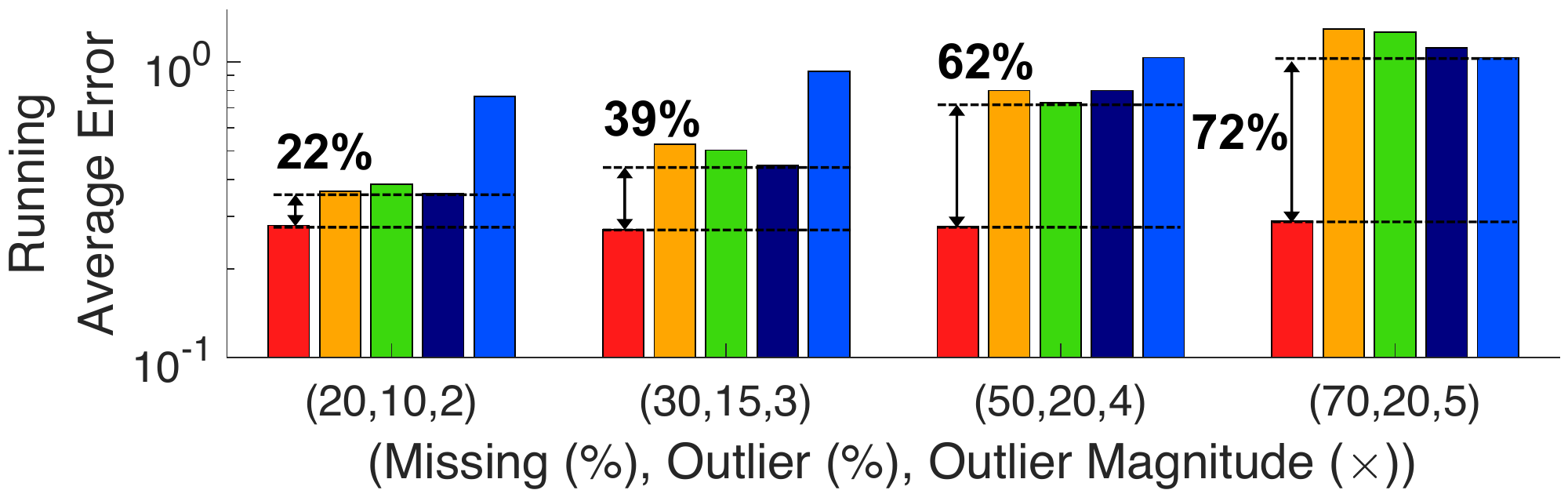}
	}\\
	\vspace{-2mm}
	\caption[]{\label{fig:exp:imputation:rae}
		The running average error under 4 experimental settings from the mildest (leftmost) to the harshest (rightmost).
		\textbf{\method was the most accurate} in all the tensor streams and all the experimental settings.
	}
\end{figure}

\begin{figure}[!t]
	\centering
	\subfigure[Intel Lab Sensor]{\label{fig:art:intel}
		\includegraphics[width=0.9\linewidth]{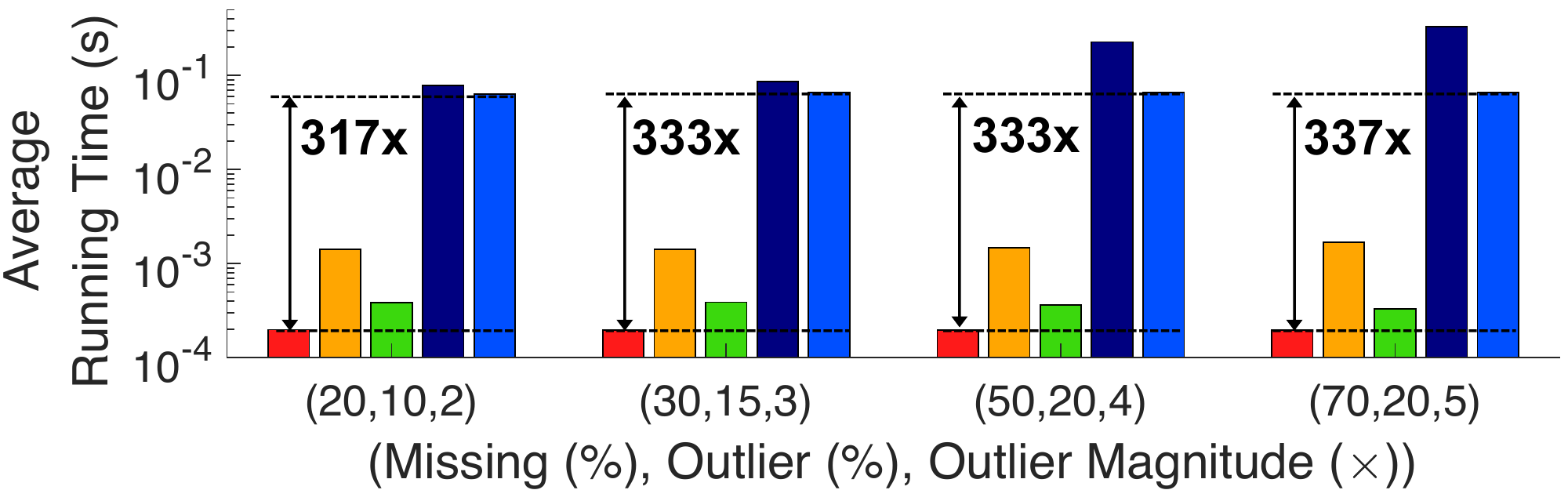}
	}\\
	\vspace{-3mm}
	\subfigure[Network Traffic]{\label{fig:art:network}
		\includegraphics[width=0.9\linewidth]{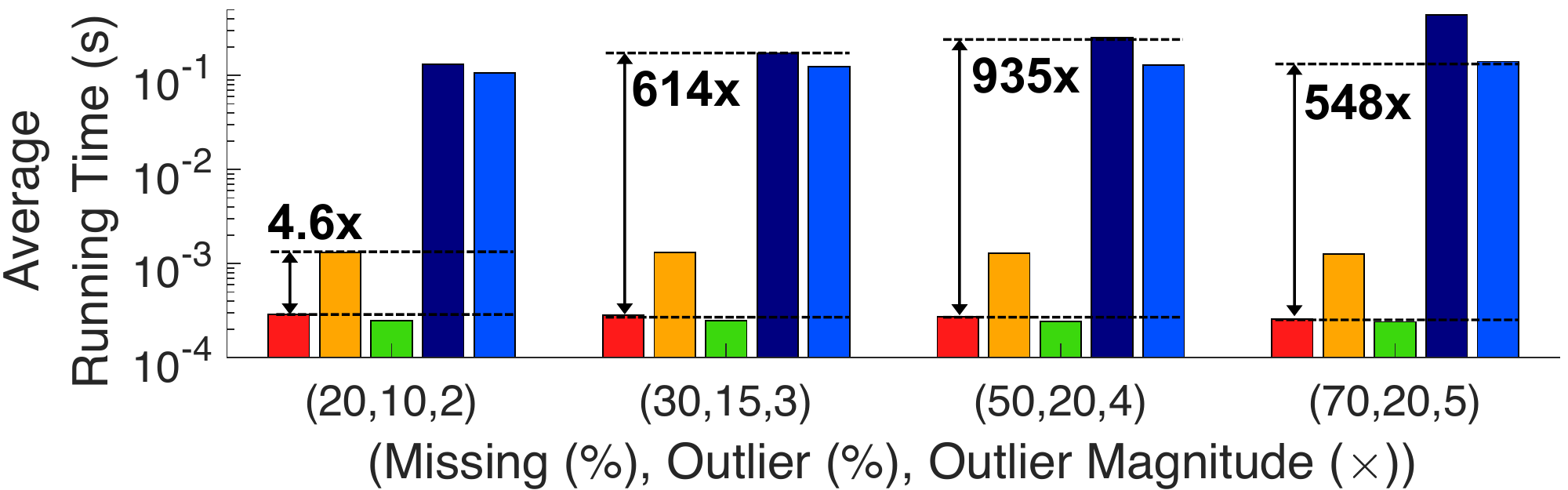}
	}\\
	\vspace{-3mm}
	\subfigure[Chicago Taxi]{\label{fig:art:chicago}
		\includegraphics[width=0.9\linewidth]{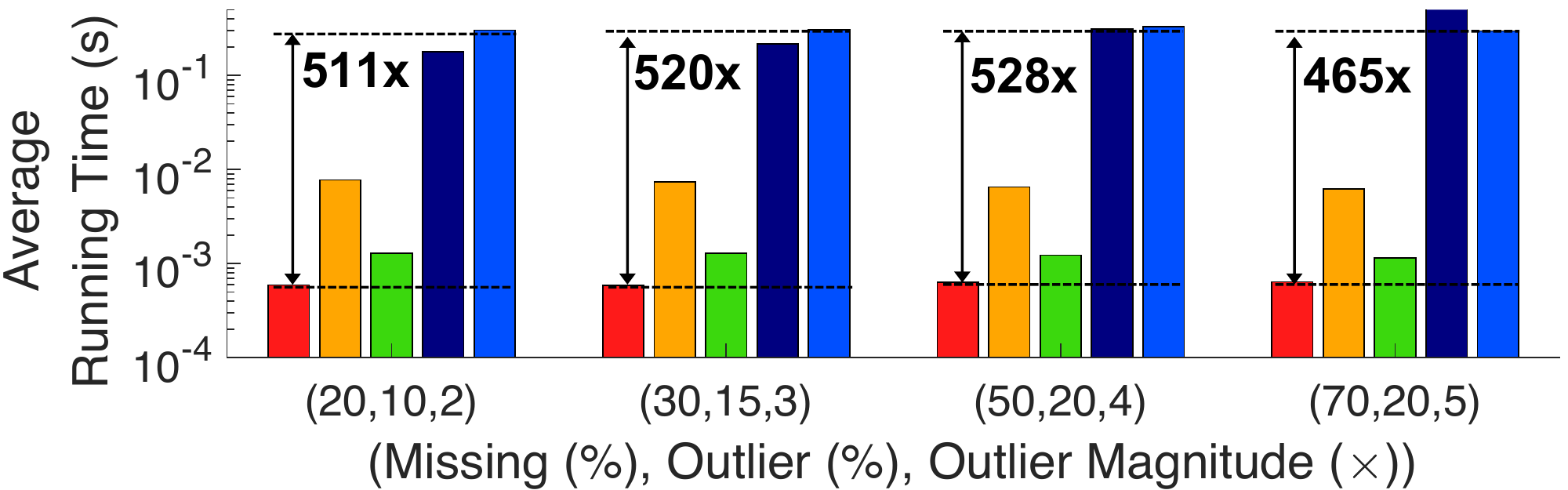}
	}\\
	\vspace{-3mm}
	\subfigure[NYC Taxi]{\label{fig:art:nyc}
		\includegraphics[width=0.9\linewidth]{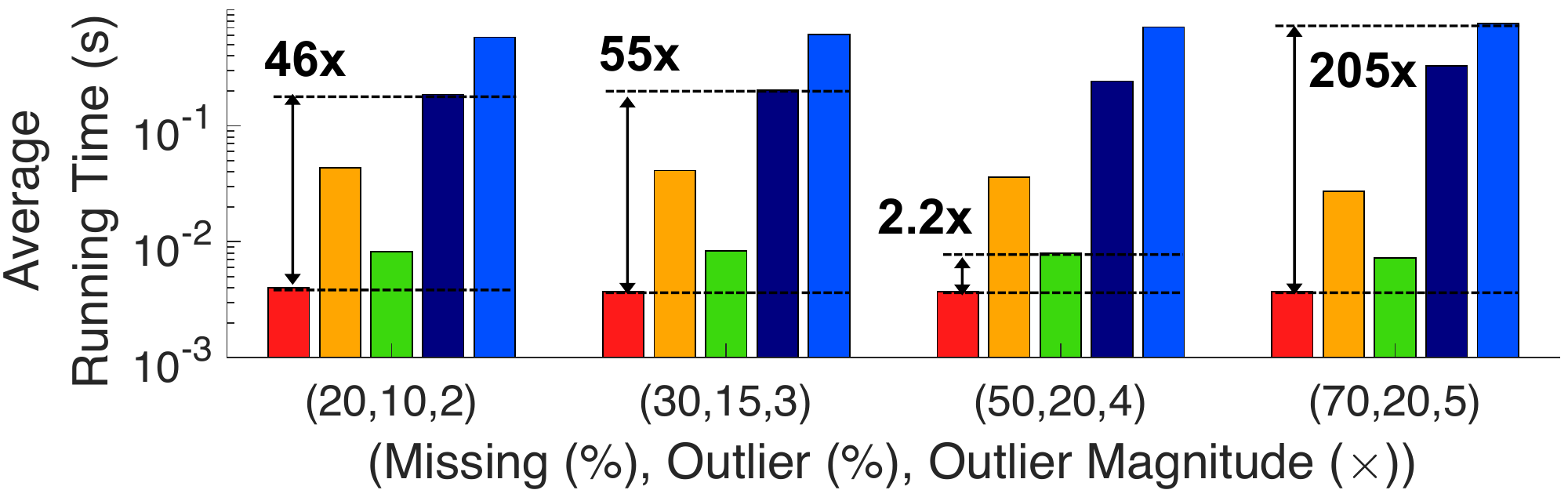}
	}\\
	\vspace{-2mm}
	\caption[]{\label{fig:exp:art}
		The average running time to process one subtensor under 4 experimental settings from the mildest (leftmost) to the harshest (rightmost).
		\textbf{\method was up to $935\times$ faster} than the second-most accurate algorithm.
	}
\end{figure}

\noindent\underline{\smash{\textbf{Parameter Setting}}}:
Unless otherwise stated, we used $\lambda_{1}=\lambda_{2}=10^{-3}$, $\lambda_{3}=10$, $\mu=0.1$, and $\phi=0.01$ as default parameters.
For baseline methods, we tuned their hyperparameters using grid search or following their authors' suggestions.
We set the maximum number of iterations and the tolerance rate to $300$ and $10^{-4}$ for all the methods.
The rank is adjusted using 10 ranks varying from 4 to 20 based on running average error.

\noindent\underline{\smash{\textbf{Missing and Outlier}}}:
A $Y\%$ of randomly selected entries are corrupted by outliers and $X\%$ of randomly selected entries are ignored and treated as missings.
The magnitude of each outlier is $-Z\cdot \max(\TX)$ or $Z\cdot\max(\TX)$ with equal probability, where $\max(\TX)$ is the maximum entry value of the entire ground truth tensor.
We use a tuple of $(X,Y,Z)$ to denote the experimental setting.
For example, $(70,20,5)$ represents that $70\%$ of entries are missing and $20\%$ of entries are contaminated by outliers whose magnitude is $-5\cdot\max(\TX)$ or $5\cdot\max(\TX)$.


\subsection{Q1. Initialization Accuracy}\label{sec:exp:sofia_als}
We evaluated how precisely the initialization step using \methodALS discovers temporal patterns in an incomplete and noisy tensor as the number of outer iterations (i.e., lines~\ref{alg:init:start}-\ref{alg:init:end} in Algorithm~\ref{alg:init}) increases.
We used a low-rank synthetic tensor of size $30\times 30\times 90$ generated by rank-3 factor matrices, i.e., $\Umat^{(1)},\Umat^{(2)}\in\mathbb{R}^{30\times 3}$, and $\Umat^{(3)}\in\mathbb{R}^{90\times 3}$.
To model a tensor that has temporal patterns, the $r$-th column of the temporal factor matrix was formulated as $\tilde{\uvec_{r}}^{(3)}=[a_{r}\sin((2\pi/m)i+b_{r})+c_{r}]$, where $i=1,\dots,90$ and $m=30$, for $r=1,2,3$.
The coeffieicnts $a_{r}$, $b_{r}$, and $c_{r}$ were set to values selected uniformly at random from $[-2,2]$, $[0,2\pi]$, and $[-2,2]$, respectively.
Figure~\ref{fig:als:ground} shows the ground-truth temporal factor matrix.
After that, we set the experimental environment to $(90,20,7)$, which is extremely harsh.

We extracted the temporal patterns from the contaminated tensor using two methods: 1) initialization with the vanilla ALS~\cite{zhou2008large} and 2) initialization with \methodALS. 
As shown in Figures~\ref{fig:als:vanilla}-\ref{fig:als:nre}, the method using \methodALS was able to restore the temporal factor matrix accurately as outer iteration proceeded, while the method using the vanilla ALS did not.
This is because temporal and seasonal smoothness considered in \methodALS are greatly helpful to find the underlying patterns even in a situation where $90\%$ of data were lost and many extreme outliers existed.

\begin{figure*}[!t]
	\centering
	\vspace{-4mm}
	\includegraphics[width=0.75\linewidth]{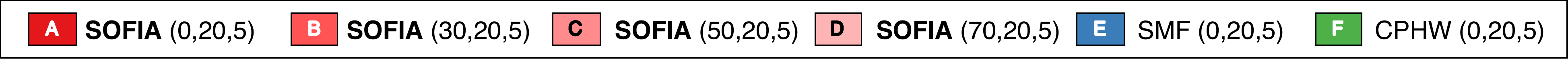}\\
	\vspace{-2mm}
	\subfigure[Intel Lab Sensor]{\label{fig:forecast:intel}
		\includegraphics[width=0.20\linewidth]{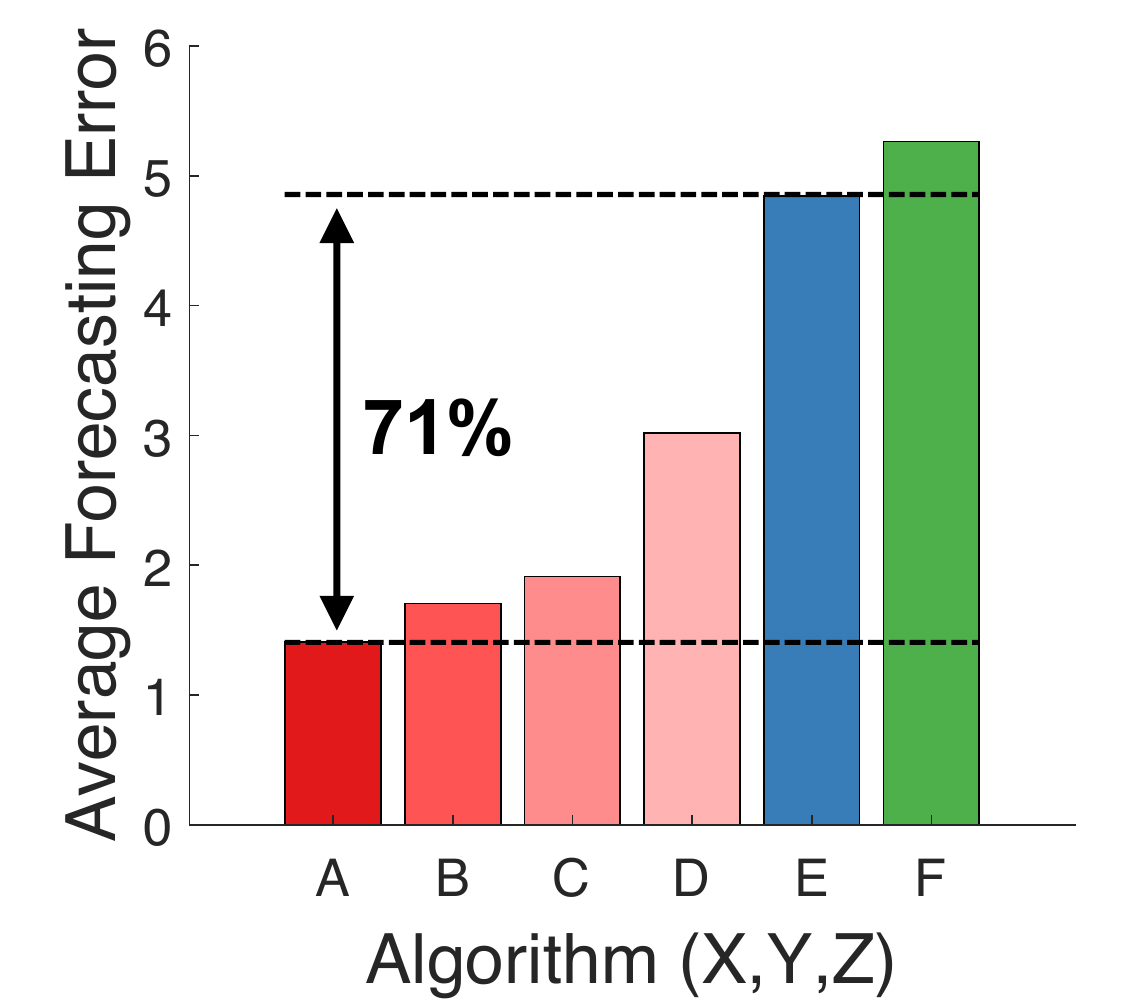}
	}
	\hspace{-3mm}
	\subfigure[Network Traffic]{\label{fig:forecast:traffic}
		\includegraphics[width=0.20\linewidth]{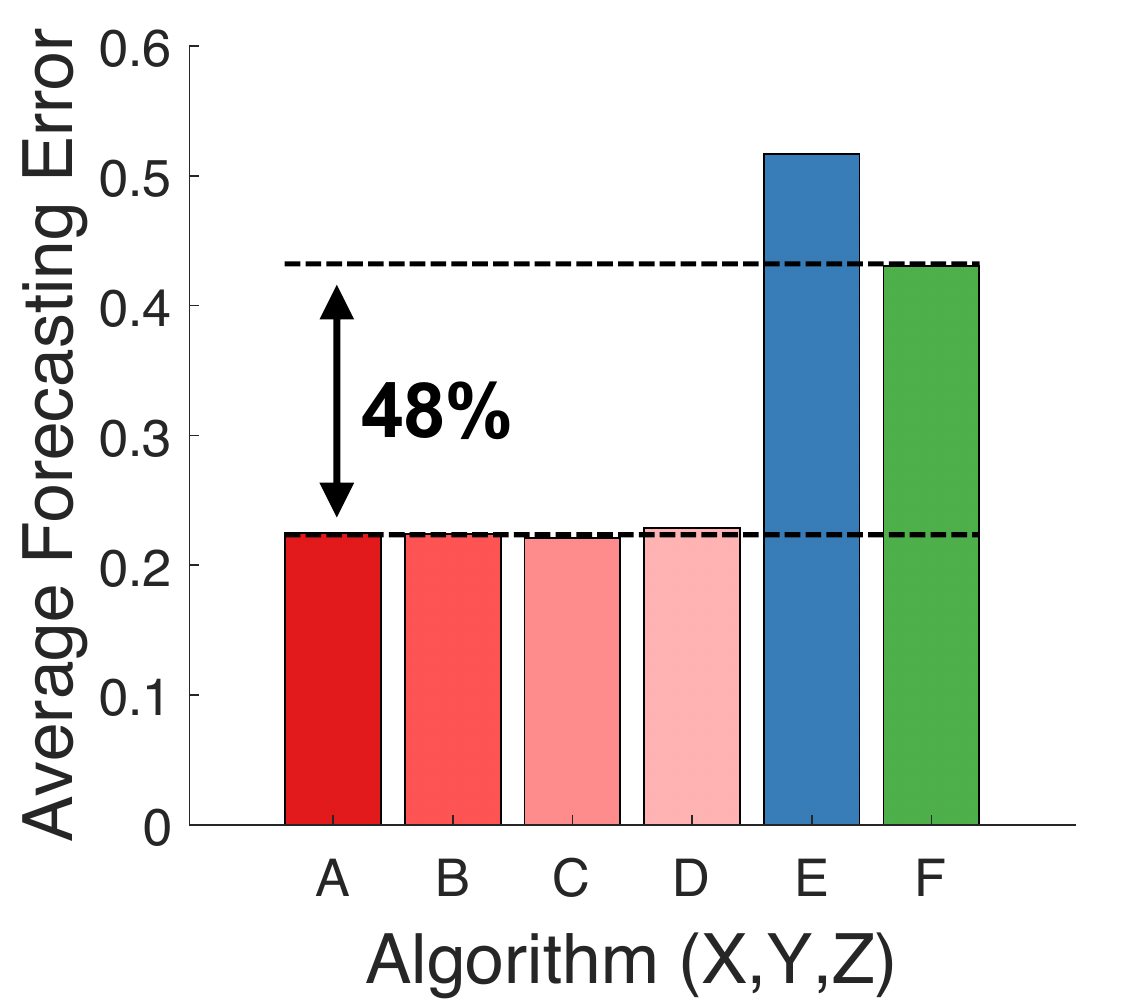}
	}
	\hspace{-3mm}
	\subfigure[Chicago Taxi]{\label{fig:forecast:chicago}
		\includegraphics[width=0.20\linewidth]{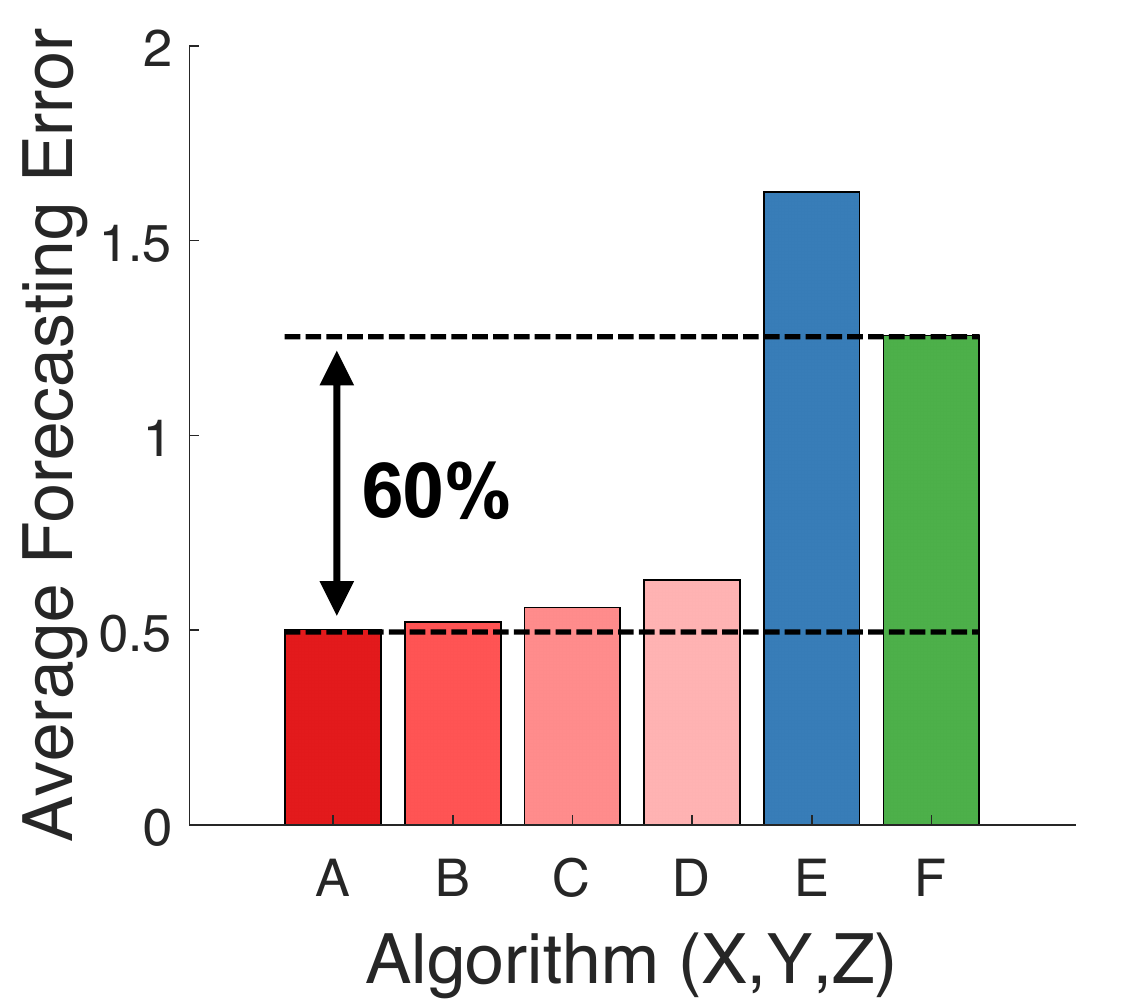}
	}
	\hspace{-3mm}
	\subfigure[NYC Taxi]{\label{fig:forecast:nyc}
		\includegraphics[width=0.20\linewidth]{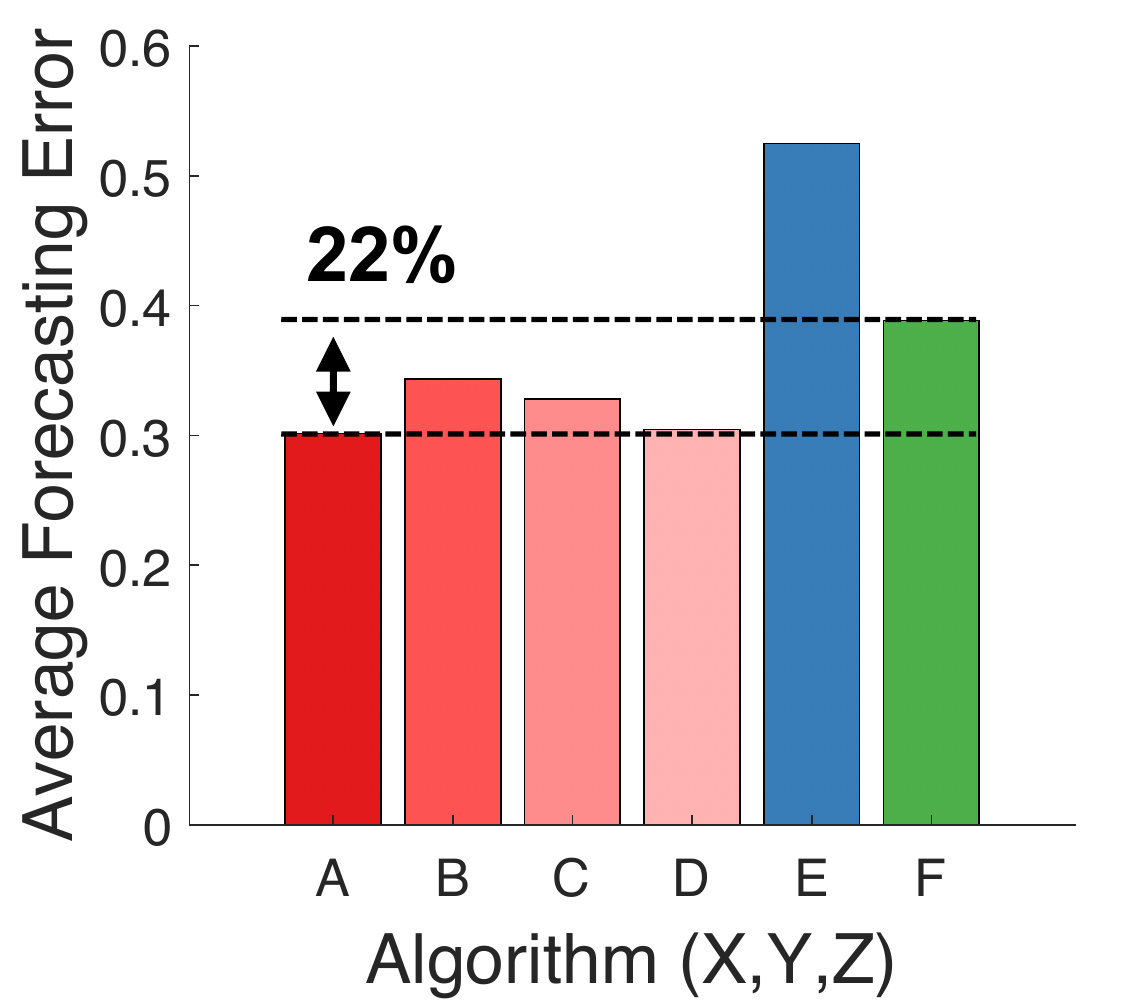}
	}\\
	\vspace{-1mm}
	\caption[]{\label{fig:forecast}
		The average forecasting error under 4 experimental settings. We evaluated \method on various fractions of missing entries, while the competitors were evaluated assuming all entries are observed.
		\textbf{\method was the most accurate}, despite the presence of missing values.
	}
\end{figure*}

\subsection{Q2. Imputation Accuracy}\label{sec:exp:imputation}
We measured how accurately \method estimates missing entries.
Figure~\ref{fig:exp:imputation:nre} shows the accuracy of the considered approaches at each time $t$ in four different levels of missing ratio, outlier ratio, and outlier magnitude.
The mildest setting was $(20,10,2)$, and the most extreme one was $(70,20,5)$.

In all the tensor streams, \method was the most accurate in terms of normalized residual error (NRE) regardless of the degree of missing data and outliers.
This is because \method discovers seasonal patterns behind the noisy and incomplete data and accurately predicts the entries in the next time step using the HW method.
Based on the predictions, \method filters out extremely high or low values regarded as outliers, and thus \method is robust to outliers.
Since \onlineSGD, \olstec, and \mast do not distinguish outliers from normal values, their models are susceptible to outliers.
Since \ormstc is designed to deal with outliers that are distributed over a specific mode of the tensor (e.g., 2-nd mode outliers), it is not effective to handle element-wise outliers used in this experiment.
We did not report the results of \brst, which wrongly estimated that the rank is $0$ in all the tensor streams.

Figure~\ref{fig:exp:imputation:rae} shows the overall accuarcy over the entire stream.
\method gave up to $76\%$ smaller running average error (RAE) than the second-most accurate approach.


\subsection{Q3. Speed}\label{sec:exp:speed}

We measured the average running time of the dynamic update steps of different approaches.
Note that since the initialization and HW fitting steps in \method are executed only once at the beginning, the time spent for them becomes negligible as the stream evolves continuously.
Similarly, time spent for initialization in \mast and \ormstc were excluded.

Figure~\ref{fig:exp:art} shows the average running time to process a single subtensor.
\method was the fastest in the NYC Taxi, Chicago Taxi, and Intel Lab Sensor datasets, while \method was comparable to the fastest competitor in the Network Traffic dataset.
Notably, \method was $2.2-935\times$ faster than the second-most accurate algorithm, demonstrating that \method is suitable for real-time data processing.


\subsection{Q4. Forecasting Accuracy}\label{sec:exp:forecasting}

We evaluated the forecasting accuracy of \method compared to two competitors.
Each algorithm consumes $(T-t_{f})$ subtensors and forecasts the following $t_{f}$ subtensors.
We set $t_{f}$ to $200$ for the Chicago Taxi, Network Traffic, and Intel Lab Sensor datasets and set it to $100$ for the NYC Taxi dataset.
We injected $20\%$ outliers whose magnitudes are $\pm 5 \cdot \max(\TX)$.
Since the competitors do not take missing values into account, we evaluated them assuming all entries are observed, while \method is evaluated on various fractions of missing entries.

We used the average forecasting error (AFE) for the future subtensors to measure the accuracy.
As seen in Figure~\ref{fig:forecast}, in all the tensor streams, \method was the most accurate despite the presence of missing entries.
Especially, in the Intel Lab Sensor dataset, \method gave up to $71\%$ smaller average forecasting error than the second-best method.
Note that, since \smf and \cphw do not filter out outliers, these models are heavily affected by outliers, while \method is robust to outliers.
Notably, the average forecasting error was almost the same in the Network Traffic dataset regardless of the missing percentage.
That is, the discovered seasonal pattern was nearly identical under all settings.
On the other hand, in the Intel Lab Sensor dataset, as the fraction of missing entries increased, it became hard to find the correct seasonal pattern, and thus the forecasting error increased.

We also compared the speed of \method and \smf.
\smf was faster than \method (i.e., $1.32$, $2.1$, $2.98$, and $5.06$ times faster in the Intel Lab Sensor, Network Traffic, Chicago Taxi, and NYC Taxi datasets, respectively). However, \method was significantly more accurate than \smf, as seen in Figure~\ref{fig:forecast}. Moreover, \method is applicable to incomplete tensors with missing entries, while \smf is not.
Since \cphw is a batch algorithm, it needs to be rerun from scratch at each time step.

\begin{figure}
	\vspace{-2mm}
	\centering
	\subfigure[w.r.t. the Number of Entries]{\label{fig:scale:element}
		\includegraphics[width=0.45\linewidth]{FIG/scalability_entry.pdf}
	}
	\subfigure[w.r.t. the Number of Time Steps]{\label{fig:scale:time}
		\includegraphics[width=0.45\linewidth]{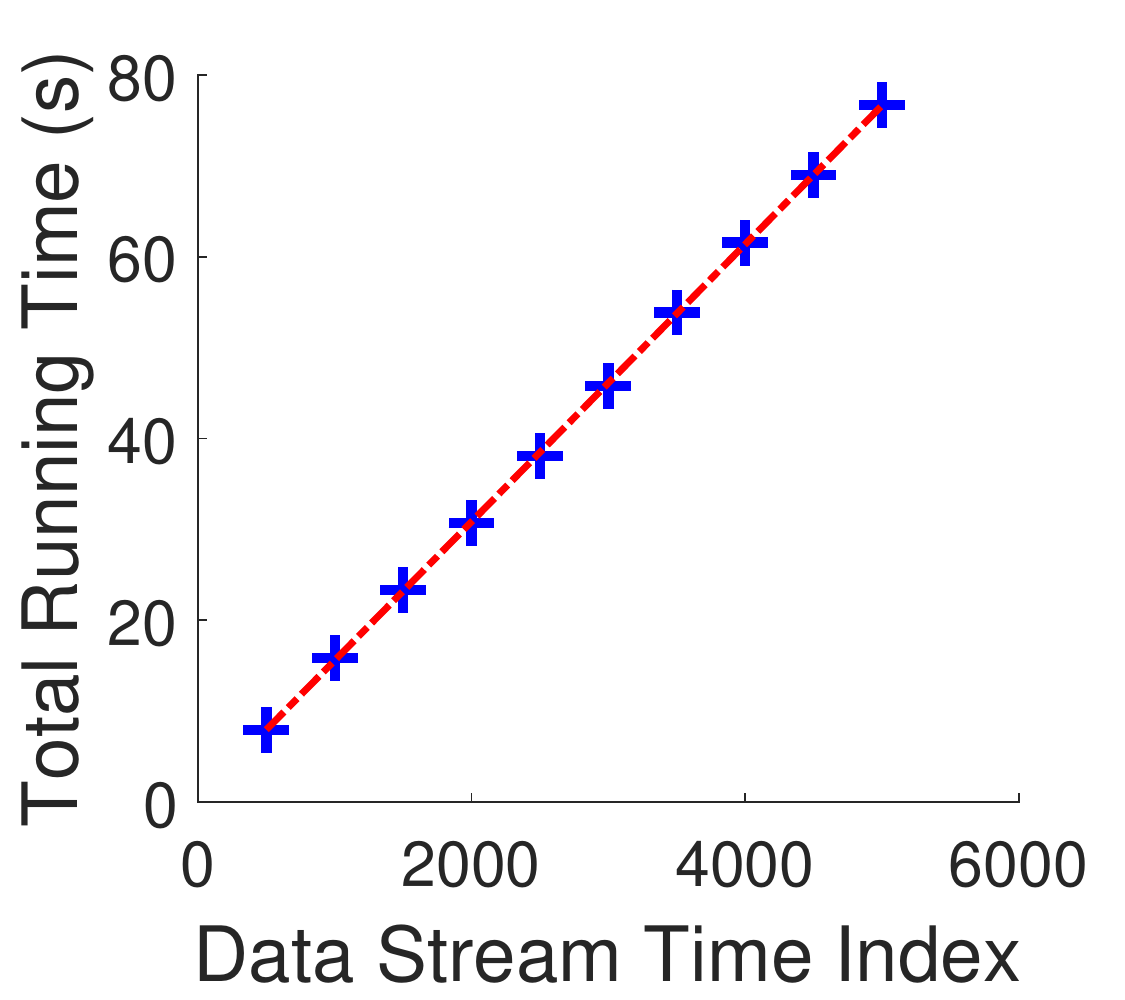}
	}
	\vspace{-1mm}
	\caption[]{\label{fig:exp:scalability}
		\textbf{\method scaled linearly} with the number of entries in each subtensor and the number of time steps.
		That is, time taken by \method per time step was almost constant regardless of the number of subtensors processed so far.
	}
\end{figure}

\subsection{Q5. Scalability}\label{sec:exp:scalability}

We measured how the running time of \method scales with (1) the number of entries in each subtensor and (2) the number of time steps.
We created a synthetic tensor stream consisting of $500\times 500$ subtensors (i.e., matrix) for $5000$ time steps and set the seasonal period to $10$.
For simplicity, we set all entries are observed and there are no outliers.
We sampled $\{50, 100, \dots, 500\}$ indices of the first mode and made tensor streams with different numbers of entries per subtensor.
Then, we measured the time taken to process the entire tensor stream excluding initialization and HW fitting, as discussed in Section~\ref{sec:exp:speed}.
Figure~\ref{fig:exp:scalability} shows that \method scaled linearly with the number of entries and the number of time steps.
That is, time taken by \method per time step was almost constant regardless of the number of subtensors processed so far.

\section{Conclusion}
\label{sec:conclusion}

In this work, we propose \method, an online algorithm for factorizing real-world tensors that evolve over time with missing entries and outliers.
By smoothly and tightly combining tensor factorization, outlier detection, and temporal-pattern detection,
\method achieves the following strengths over state-of-the-art competitors:

\begin{itemize}[leftmargin=3mm]
	\item \textbf{Robust and accurate}: \method yields up to $76\%$ and $71\%$ lower imputation and forecasting error than its best competitors (Figures~\ref{fig:exp:imputation:nre}, \ref{fig:exp:imputation:rae}, and \ref{fig:forecast}).
	\item \textbf{Fast}: Compared to the second-most accurate method, using \method makes imputation up to $935\times$ faster (Figure~\ref{fig:exp:art}).
	\item \textbf{Scalable}: \method incrementally processes new entries in a time-evolving tensor, and it scales linearly with the number of new entries per time step (Figure~\ref{fig:exp:scalability} and Lemma~\ref{lemma:complexity:dynamic}).
\end{itemize}

\noindent{\textbf{Reproducibility}}:
The code and datasets used in the paper are available at \address.

\section*{Acknowledgement}
\textls[-15]{\small This work was supported by Samsung Electronics Co., Ltd. and Institute of Information \& Communications Technology Planning \& Evaluation (IITP) grant funded by the Korea government (MSIT) (No. 2019-0-00075, Artificial Intelligence Graduate School Program (KAIST)).}

\bibliographystyle{IEEEtran}
\bibliography{./BIB/dongjin.bib}

\end{document}